\documentclass[twoside,11pt]{article}

\usepackage{blindtext}

\usepackage{jmlr2e}

\usepackage{tikz}
\usepackage{algorithm}
\usepackage{algpseudocode}
\usepackage{nicematrix}
\usepackage{setspace} 
\usepackage{amsfonts} 
\usepackage{amsmath} 
\usepackage{placeins}
\usepackage{amssymb} 
\usepackage{bm}
\usepackage{graphicx}
\usepackage{lineno} 
\usepackage{multicol}
\usepackage{multirow}
\usepackage{caption}
\usepackage{subcaption}
\usepackage{pdflscape}
\usepackage{hyperref}
\hypersetup{ hidelinks }
\usepackage{breqn}
\usepackage{diagbox}
\usepackage{cleveref}
\usepackage{ifthen}

\newcommand{\fakesection}[1]{%
  \par\refstepcounter{section}
  \sectionmark{#1}
  \addcontentsline{toc}{section}{\protect\numberline{\thesection}#1}
}

\newcommand{\norml}{{\|\varphi^\ell_\alpha\|^2 \|\varphi^\ell_\beta\|^2}}
\newcommand{\normll}{{\|\varphi^{\ell+1}_\alpha\|^2 \|\varphi^{\ell+1}_\beta\|^2}}
\newcommand{\normratio}{{\frac{\normll}{\norml}}}
\newcommand{\normratioo}{R_{\ell+1}}
\newcommand{\ph}{{\varphi}}
\newcommand{\pha}{{\ph_\alpha^\ell}}
\newcommand{\phb}{{\ph_\beta^\ell}}
\newcommand{\phaa}{{\ph_\alpha^{\ell+1}}}
\newcommand{\phbb}{{\ph_\beta^{\ell+1}}}
\newcommand{\cov}{{\mathbf{Cov}}}
\DeclareMathOperator{\var}{\mathbf{Var}}

\newcommand{\review}[1]{{#1}}

\newtheorem{fact}{Fact}
\newtheorem{approximation}{Approximation}
\newtheorem{prop}{Proposition}

\usepackage{lastpage}

\jmlrheading{25}{2024}{1-\pageref{LastPage}}{3/23; Revised
7/24}{8/24}{23-0350}{Cameron Jakub and Mihai Nica}

\ShortHeadings{Depth Degeneracy in Neural Networks}{Jakub and Nica}
\firstpageno{1}

\begin{document}
\title{Depth Degeneracy in Neural Networks: Vanishing Angles in Fully Connected ReLU Networks on Initialization}

\author{\name Cameron Jakub \hfill \texttt{cjakub@uoguelph.ca} \\
       \addr Department of Mathematics \& Statistics\\
       University of Guelph
       \AND
       \name Mihai Nica \hfill \texttt{nicam@uoguelph.ca} \\
       \addr Department of Mathematics \& Statistics\\
       University of Guelph} 

\editor{Miguel Carreira-Perpinan}

\maketitle

\begin{abstract}
Despite remarkable performance on a variety of tasks, many properties of deep neural networks are not yet theoretically understood. One such mystery is the depth degeneracy phenomenon: the deeper you make your network, the closer your network is to a constant function on initialization. In this paper, we examine the evolution of the angle between two inputs to a ReLU neural network as a function of the number of layers. By using combinatorial expansions, we find precise formulas for how fast this angle goes to zero as depth increases. These formulas capture microscopic fluctuations that are not visible in the popular framework of infinite width limits, and leads to qualitatively different predictions. We validate our theoretical results with Monte Carlo experiments and show that our results accurately approximate finite network behaviour. \review{We also empirically investigate how the depth degeneracy phenomenon can negatively impact training of real networks.} The formulas are given in terms of the mixed moments of correlated Gaussians passed through the ReLU function. We also find a surprising combinatorial connection between these mixed moments and the Bessel numbers that allows us to explicitly evaluate these moments.  \end{abstract}


\section{Introduction}
\label{sec:intro}
The idea of stacking many layers to make truly \emph{deep} neural networks (DNNs) is what arguably led to the neural net revolution in the 2010s. Indeed, from a function-space point of view, it is known that depth exponentially improves expressibility \citep{poole_expressivity, eldan_expressivity}. However, an important but less well known fact is that under standard initialization schemes, deep neural networks become more and more \emph{degenerate} as depth gets larger and larger. One sense in which this happens is the phenomenon of vanishing and exploding gradients \citep{boris_gradients}. Another sense in which networks become degenerate is that a neural network gets closer and closer to a (random) constant function, i.e. the network sends all inputs to the same output and cannot distinguish input points. This phenomenon seems to have been  discovered and analyzed from different points of view by several authors  \citep{JMLR:cut_off, dherin_simple_solutions, li_nica_roy, hayou, schoenholz, cnn_nachum, buchanan_deep}.  \citet{cnn_nachum} found for convolutional neural networks, the level of degeneracy was dependent on the type of input fed into the network.

One method already proposed to deal with the degeneracy phenomenon is the idea of activation function \emph{shaping} \citep{martens_shaping}. In particular, \citet{li_nica_roy}, showed that rescaling the non-linear activation function (i.e. using Leaky ReLUs with leakiness depending on network depth) can lead to a non-trivial angle between inputs. However, a detailed analysis of the evolution of the angle $\theta$ \emph{without} any scaling (e.g. using an ordinary ReLU in all layers) remained an outstanding problem. This is the gap we fill in this article.

\subsection{Main Results for the Angle Process $\theta_\ell$}
\label{sec:main_results}
In this paper, we examine the evolution of the \emph{angle} $\theta_\ell$  between two arbitrary inputs $x_\alpha, x_\beta \in \mathbb{R}^{n_{in}}$ after passing through $\ell$ layers of a fully connected ReLU network (a.k.a. a multi-layer perceptron) on initialization. The angle is defined in the usual way by the inner product between two vectors in $\mathbb{R}^{n_\ell}$
\begin{equation*}
\cos\left( \theta_\ell \right) := \frac{\langle F^\ell(x_\alpha), F^\ell(x_\beta) \rangle}{\Vert F^\ell(x_\alpha) \Vert \Vert F^\ell(x_\beta) \Vert } , 
\end{equation*} %
where $n_\ell$ is the width (i.e. number of neurons) of the $\ell$-th layer and $F^\ell: \mathbb{R}^{n_{in}} \to \mathbb{R}^{n_\ell}$ is the (random) neural network function mapping input to the post-activation logits in layer $\ell$ on initialization. We assume here that the initialization is done with appropriately scaled independent Gaussian weights so that the network is on the ``edge of chaos'' \citep{hayou, schoenholz}, where the variance of each layer is order one as layer width increases. See Table \ref{tbl:notation} for our precise definition of the fully connected ReLU neural network. 

\begin{figure}[h!]
    \centering
    \includegraphics[scale=0.58]{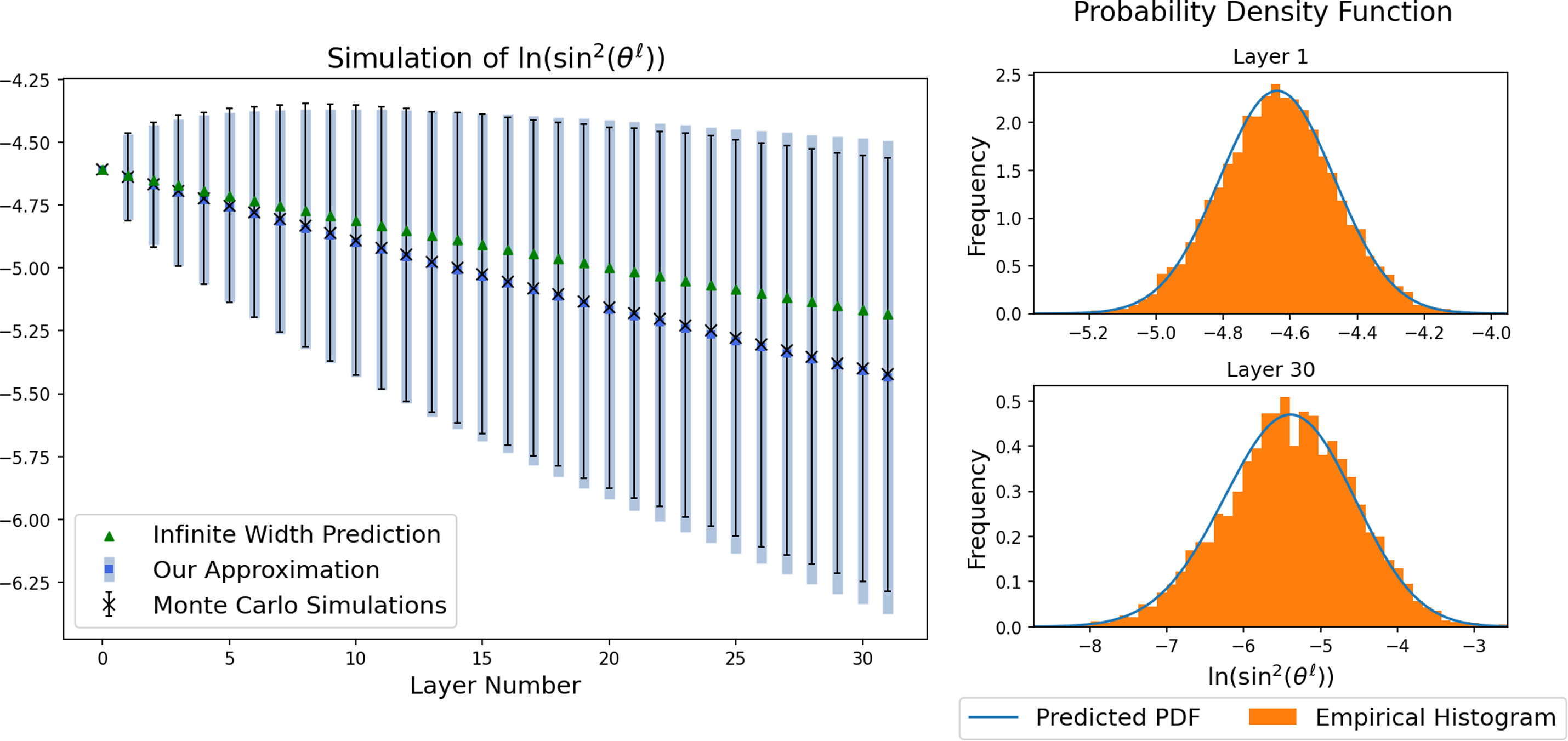}
    \caption{We feed 2 inputs with initial angle $\theta_0 = 0.1$ into 5000 Monte Carlo samples of independently initialized networks with network width $n_\ell = 256$ for all layers. \emph{Left:} Using the Monte Carlo samples, we plot the empirical mean and standard deviation of $\ln(\sin^2(\theta_\ell))$ at each layer. We compare this to both the infinite width update rule and our prediction using Approximation \ref{approx:simple} for the mean of $\ln(\sin^2(\theta_\ell))$ \review{(Shown as the blue square)}. Our prediction for the standard deviation in each layer using Approximation \ref{approx:Gaussian} is also plotted as the shaded area. \review{To compute this, we iterate Approximation 2 to estimate the PDF in each layer, and then compute the variance using the PDF.} In contrast to our prediction, the infinite width rule predicts 0 variance in all layers. \emph{Right:}  We plot histograms of our simulations as well as our predicted probability density function using Approximation \ref{approx:Gaussian} from \eqref{eq:update_full} at Layer 1 (top) and Layer 30 (bottom). \review{The predicted PDF is computed numerically by iterating Approximation \ref{approx:Gaussian} over the 30 layers, using the PDF in layer $\ell$ to get the PDF in layer $\ell+1$. }The predicted and empirical distribution are statistically indistinguishable according to a Kolmogorov-Smirnov test, with $p$ values $0.987 > 0.05$ (top) and $ 0.186  > 0.05$ (bottom). \review{The code which produced this figure can be found at the following \href{https://github.com/camjakub/Depth-Degeneracy-in-Neural-Networks}{\texttt{GitHub link}}.}}
    \label{updaterule}
\end{figure}

With this setup, since the effect of each layer is independent of everything previous, $\theta_\ell$ can be thought of as a Markov chain evolving as layer number $\ell$ increases. As expected by the aforementioned ``large depth degeneracy'' phenomenon,  we observe that the angle concentrates $\theta_\ell \to 0$ as $\ell\to \infty$ (see Figure \ref{updaterule} for an illustration). This indicates that the hidden layer representation of \emph{any} two inputs  becomes closer to co-linear as depth increases. 

We obtain a simple, yet remarkably accurate, approximation for the evolution of $\theta_\ell$ as a function of $\ell$ that captures precisely how quickly this degeneracy happens for small angles $\theta_\ell$ and large layer widths $n_\ell$. \review{ In Section \ref{sec:training}, we also empirically investigate how these predictions relate to network performance \emph{after} training and show they may have applications in neural architecture search.}

\begin{approximation}
\label{approx:simple}
For small angles $\theta_\ell \ll 1$ and large layer width $n_\ell \gg 1$, the angle $\theta_{\ell+1}$ at layer $\ell+1$ is well approximated by
\begin{equation} \label{eq:update_simple}
\ln \sin^2(\theta_{\ell+1}) \approx
\ln \sin^2(\theta_\ell) - \frac{2}{3\pi}\theta_\ell - \rho(n_\ell), 
\end{equation} 
where $\rho(n_\ell)$ is a constant which depends on the width $n_\ell$ of layer $\ell$, namely:
\begin{equation}
    \rho(n) := \ln\left( \frac{n+5}{n-1}\right) - \frac{10n}{\left(n+5\right)^2} + \frac{6n}{\left(n-1\right)^2 }= \frac{2}{n} + \mathcal{O}\left(n^{-2}\right). \label{eq:rho}
\end{equation}

\end{approximation}


Figure \ref{updaterule} illustrates how well this prediction matches Monte Carlo simulations of $\theta_\ell$ sampled from real networks. Also illustrated is the \emph{infinite width} prediction for $\theta_\ell$ (discussed in \Cref{app:infinite_width}) which is less accurate at predicting finite width network behaviour than our formula, due to the $n^{-1}_\ell$ effects that our formula captures in the term $\rho(n_\ell)$ but are not present in the infinite width formula. Approximation 1 is a simple corollary to the mathematically rigorous statement, Theorem \ref{thm:mean_var}, for the mean and variance of the random variable $\ln\sin^2(\theta_\ell)$. 
\review{Approximation \ref{approx:simple} is obtained by doing a Taylor series expansion around $\theta = 0$ and ignoring terms of size $\mathcal{O}(n^{-2})$ from the more precise Theorem \ref{thm:mean_var} concerning the random variable $\ln \sin^2(\theta_\ell)$, which is why Approximation \ref{approx:simple} only holds for $\theta \ll 1$ and $n \gg 1$. See also Corollary \ref{thm:mean_var_exp}
 for related expansions.}
\subsubsection{Theoretical Consequences and Comparison to Previous Work}
\label{sec:consequences}

Approximation \ref{approx:simple} predicts that $\theta_\ell \to 0$ \emph{exponentially fast} in $\ell$ due to $\rho(n)$; it predicts: 
$$\theta_\ell \leq \exp\left(-\frac{1}{2}\sum_{i=1}^\ell \rho(n_i)\right) = \exp\left(-\sum_{i=1}^\ell \frac{1}{n_i} + \mathcal{O}(n_i^{-2})\right).$$
\noindent (Note that the exponential behaviour vanishes when $n_\ell \to \infty$ with $\ell$ fixed). In contrast to this prediction, an analysis using only expected values or equivalently working in the infinite-width $n\to \infty$ limit predicts that $\theta_\ell \to 0$ like $\ell^{-1}$, which is qualitatively very different! \review{This difference in the rate of degeneracy demonstrates significant difference between ``real world'' and infinite width networks. See also Figure \ref{fig:comparison} for comparisons of the infinite vs finite width predictions for some real architectures.}

The prediction of the infinite width degeneracy was first demonstrated under the name ``edge of chaos'' \citep{hayou, schoenholz}  and again in \citet{principles_deep_learning, boris_correlation_functions}. These earlier works studied the correlation $\cos(\theta_\ell)$ as a function of layer number, and showed showed that $1-\cos(\theta_\ell)\to 0$ like $\ell^{-2}$, which is equivalent to $\theta_\ell\to 0$ like $\ell^{-1}$ by Taylor series expansion $\cos(x)\approx\frac{1}{2}x^2$ as $x\to 0$. To unify the notation, we also present a derivation of the update rule for $\cos(\theta_\ell)$ in the infinite width limit in our notation in Appendix \ref{app:infinite_width}. 

We can also recover the infinite width prediction from our result by replacing $\rho(n)$ with $0$ in Approximation \ref{approx:simple} in the update rule \eqref{eq:update_simple}. Exponentiating both sides and using $\sin(\theta) \approx \theta, e^\theta \approx 1+\theta$ for $\theta \ll 1$, Approximation \ref{approx:simple} becomes $(\theta_{\ell+1})^2 \approx \left(\theta^{\ell})^2(1-\frac{2}{3\pi}\theta_\ell\right)$, which is equivalent to the result of Proposition C.1 of \citet{boris_correlation_functions} and is also a corollary of Lemma 1 of \citet{hayou}. In those papers, the rule was derived directly from the infinite width update rule for $\cos(\theta)$, is equivalent since $\theta_\ell \approx \frac{1}{\ell}$ as $\ell \to \infty$. 

One of the main limitations of the infinite width predictions is that they predict zero variance in the random variable $\theta_\ell$. In contrast to this, our methods allow us to also understand the variance of this random variable, as discussed below.

\subsubsection{More detailed results for the mean and variance}

Approximation \ref{approx:simple}  comes from a simplification of more precise formulas for the mean and variance of the random variable $\ln(\sin^2(\theta_\ell))$, which are stated in Theorem \ref{thm:mean_var} below.

\begin{theorem}[Formula for mean and variance in terms of J functions]
\label{thm:mean_var}
Conditionally on the angle $\theta_\ell$ in layer $\ell$ \review{(see Table \ref{tbl:notation} for a precise definition of all the notations)}, the mean and variance of $\ln \sin^2 (\theta_{\ell+1})$ obey the following limit as the layer width $n_\ell \to \infty$  
\begin{align}
\label{mu_formula}
\mathbf{E}[\ln \sin^2(\theta_{\ell+1}) | \theta_\ell] =& \mu(\theta_\ell, n_\ell) + \mathcal{O}(n^{-2}_\ell), \quad \var[\ln \sin^2(\theta_{\ell+1})| \theta_\ell] = \sigma^2(\theta_\ell,n_\ell) + \mathcal{O}({n^{-2}_\ell}),\\ 
\mu(\theta,n) :=& \ln \left( \frac{(n-1)(1-4J_{1,1}^2)}{4J_{2,2}-1+n}\right)+ \frac{ 4(J_{2,2}+1)}{n\left(\frac{4J_{2,2}-1}{n}+1\right)^2} \label{eq:mu} \\
& -  \frac{4\left(8J_{1,1}^2J_{2,2} - 8J_{1,1}^4 + 4J_{1,1}^2 - 8J_{1,1}J_{3,1} + J_{2,2} + 1 \right) }{n\left( 1-\frac{1}{n} \right)^2 \left( 1 - 4 J_{1,1}^2 \right)^2 },\nonumber\\
\sigma^2(\theta,n) :=& \frac{8n(J_{2,2}+1)}{(4J_{2,2}-1+n)^2} + \frac{8n (8J_{1,1}^2J_{2,2} - 8J_{1,1}^4 + 4J_{1,1}^2 - 8J_{1,1}J_{3,1} + J_{2,2} +1)}{(n-1)^2(1-4J_{1,1}^2)^2} \label{eq:sigma_sq} \\
    &- \frac{16n(2J_{1,1}^2 - 4J_{1,1}J_{3,1} + J_{2,2} +1)}{(4J_{2,2}-1+n)(n-1)(1-4J_{1,1}^2)}, \nonumber
\end{align}
where 
$J_{a,b} := J_{a,b}(\theta_\ell)$ are the joint moments of correlated Gaussians passed through the ReLU function $\varphi(x)=\max\{x,0\}$, namely
\begin{equation}
J_{a,b}(\theta) := \mathbf{E}_{G,\hat{G}}[\varphi^a(G) \varphi^b(\hat{G})],
\end{equation}
where $G,\hat{G}$ are marginally $\mathcal{N}(0,1)$ random variables with correlation $\mathbf{E}[G\hat{G}]=\cos(\theta)$.
 \end{theorem}
The joint moments $J_{a,b}(\theta)$ are discussed in detail in \Cref{sec:j_functions}. A new combinatorial method of computing these moments is presented, which is used to give an explicit formula is given for these joint-moments, which is presented in Theorem \ref{thm:j_ab}. Using the explicit formula for $J_{a,b}$, the result of Theorem \ref{thm:mean_var} can be used to obtain useful asymptotic formulas for $\mu$ and $\sigma$, as in the following corollary.

\begin{corollary}[Small $\theta$ asymptotics for mean and variance]
\label{thm:mean_var_exp}
Conditionally on the angle $\theta_\ell$ in layer $\ell$, the mean and variance of $\ln \sin^2 (\theta_{\ell+1})$ obey the following limit as the layer width $n_\ell \to \infty$  
\begin{align}
\label{mu_formula2}
\mathbf{E}[\ln \sin^2(\theta_{\ell+1})] =& \mu(\theta_\ell, n_\ell) + \mathcal{O}(n^{-2}_\ell), \quad \var[\ln \sin^2(\theta_{\ell+1})] = \sigma^2(\theta_\ell,n_\ell) + \mathcal{O}({n^{-2}_\ell}),\\ 
\mu(\theta,n) =& \ln\sin^2\theta - \frac{2}{3\pi}\theta - \rho(n) - \frac{8\theta}{15\pi n} -\left(\frac{2}{9\pi^2}-\frac{68}{45\pi^2 n}\right)\theta^2 + \mathcal{O}(\theta^3), \label{eq:mu_asymp_formula} \\
\sigma^2(\theta,n) =&  \frac{8}{n} -  \frac{64}{15\pi}\frac{\theta}{n} - \left(8+\frac{296}{45\pi}\right)\frac{\theta^2}{n} +\mathcal{O}\left( \theta^3\right), \label{eq:si_asymp_formula}
\end{align}
where $\rho(n)$ is as defined in \eqref{eq:rho}.
\end{corollary}

 To derive Approximation  \ref{approx:simple} from Theorem \ref{thm:mean_var}, we simply keep only the first few terms of the series expansion \eqref{eq:mu_asymp_formula}, and then also completely drop the variability, essentially approximating $\sigma^2(\theta_\ell,n) \approx 0$ (Note that in reality $\sigma^2(\theta,n) \approx 8/n$ from \eqref{eq:si_asymp_formula}). Therefore Approximation 1 is a greatly simplified consequence of Theorem \ref{thm:mean_var}.

Moreover, our derivation shows that $\ln\sin^2(\theta_\ell)$ can be expressed in terms of \review{sums} over $n$ pairs of independent Gaussian variables (see (\ref{sums_1}-\ref{sums_4})). Thus, by central-limit-theorem type arguments, one would expect the following approximation by Gaussian laws which also accounts for the variability of $\ln \sin^2(\theta)$ using our calculated value for the variance.

\begin{figure}[h!]
\centering
\begin{subfigure}{.5\textwidth}
  \centering
  \includegraphics[scale=0.38]{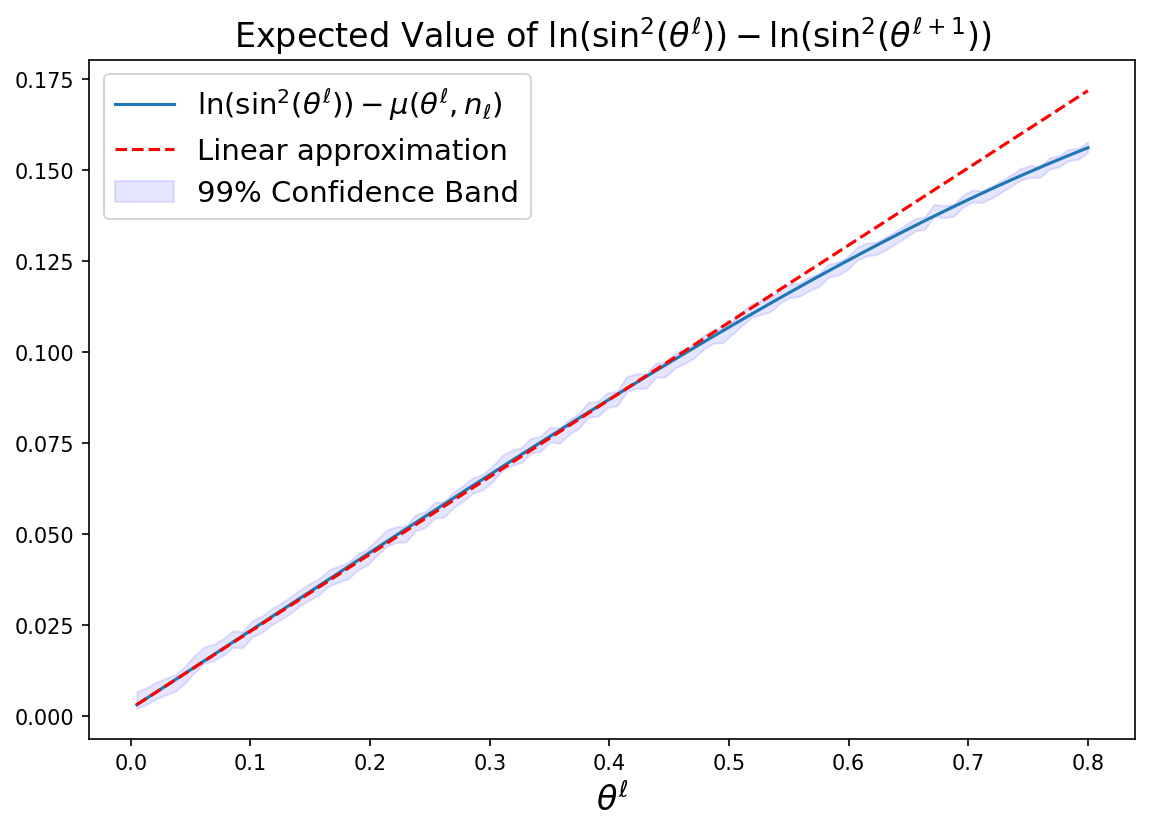}
  \caption{Mean as a function of $\theta$}
  \label{fig:theta_exp_value}
\end{subfigure}%
\begin{subfigure}{.5\textwidth}
  \centering
  \includegraphics[scale=0.38]{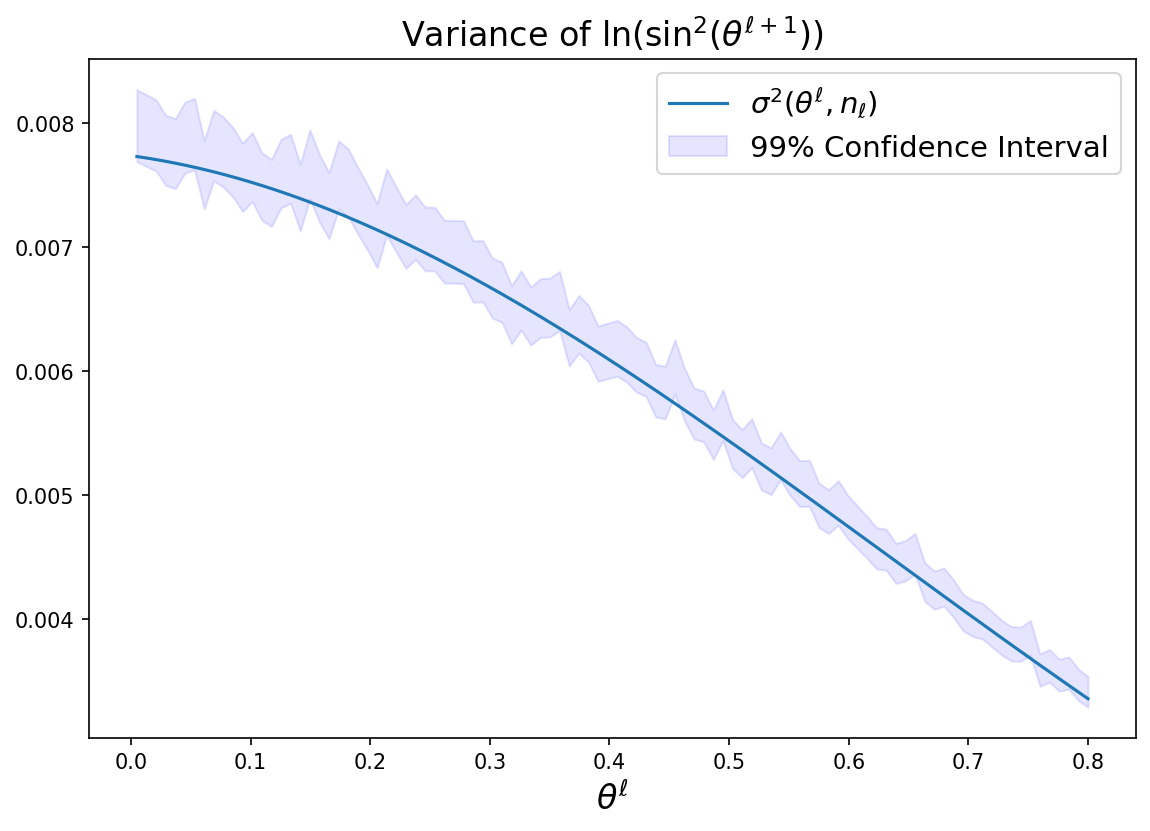}
  \caption{Variance as a function of $\theta$}
  \label{fig:theta_variance}
\end{subfigure}
\caption{Plots comparing the functions $\mu(\theta, n)$ and $\sigma^2(\theta,n)$ to simulated neural networks. The linear approximation of $\mu$, used to create Approximation \ref{approx:simple} is also displayed. Confidence bands are constructed by randomly initializing 10,000 neural networks with layer width $n_\ell=1024$, and a range of 100 initial angles $0.005 \leq \theta_\ell \leq 0.8$. We study $\theta_{\ell+1}$ and use the simulations to construct 99\% confidence intervals for a) $\mathbf{E}\left[ \ln(\sin^2(\theta_\ell)) - \ln(\sin^2(\theta_{\ell+1}))\right]$ and b) $\var\left[ \ln(\sin^2(\theta_{\ell+1}))\right]$. }
\label{fig:theta_plots}
\end{figure}

\begin{approximation}
\label{approx:Gaussian}
Conditional on the value of $\theta_\ell$, the angle at layer $\ell+1$ is well approximated by a Gaussian random variable
\begin{equation} \label{eq:update_full}
 \ln \sin^2(\theta_{\ell+1}) \stackrel{d}{\approx} \mathcal{N}(\mu(\theta_\ell,n_\ell), \sigma^2(\theta_\ell,n_\ell)),  
\end{equation}
where $\mu,\sigma^2$ are as in \review{Corollary} \ref{thm:mean_var_exp}. \review{This approximation is understood in the sense that in the limit $n_\ell \to \infty$, we have }

$$ \frac{\ln \sin^2(\theta_{\ell+1})-\mu(\theta_\ell,n_\ell)}{\sqrt{\sigma^2(\theta_\ell,n_\ell)}} \stackrel{d}{\Rightarrow} \mathcal{N}(0,1),  $$

\end{approximation}

We find that the normal approximation \eqref{eq:update_full} matches simulated finite neural networks remarkably well; see Monte Carlo simulations from real networks in Figure \ref{updaterule}. The big advantage of this approximation is that it very accurately captures the variance of $\ln(\sin^2(\theta_\ell))$, not just its mean. This variance grows as $\ell$ increases, so it is crucial for understanding behaviour of very deep networks.


The methods we use to obtain these approximations are quite flexible. For example, more accurate approximations can be obtained by incorporating higher moments $J_{a,b}(\theta)$ (see \Cref{sec:relu_nets} for a discussion). We also believe that it should be possible to extend these methods to other non-linearities beyond ReLU and more complicated neural network architectures through the same basic principles we introduce here. 
\subsection{\review{Practical Consequences: Depth Degeneracy Negatively Impacts Training}}
\label{sec:training}

\review{ In this section, we empirically investigate how the theoretical prediction of large depth degeneracy phenomenon can lead to poor results \emph{after training}. In other words, we show evidence that the depth degeneracy phenomenon (identified and studied only at \emph{initialization}) can be used as a screening tool for neural architecture search to identify problematic neural architectures before they are trained. This could potentially add to the arsenal of existing tools used for neural architecture search (see e.g. \cite{NAS})
 
 We use the formula $\mu(\theta,n)$ developed in Theorem \ref{thm:mean_var} to create a simple algorithm which accurately predicts the angle between inputs after travelling through the layers of an initialized network up to an error of size $\mathcal{O}(n_\ell^{-2})$ in layer $\ell$.

\begin{algorithm}
    \caption{Angle prediction between inputs for a feed-forward ReLU network with depth $L$ and layer widths $n_\ell,\; 1 \leq \ell \leq L$. The function $\mu(\theta,n)$ is given in Theorem \ref{thm:mean_var_exp}. }
    \label{algo:update_rule}
    \begin{algorithmic}[1]
        \item $\theta^0 = $ angle between inputs
        \For{$\ell = 0, \ldots, L-1$}
            \State $x = \mu(\theta_\ell, n_\ell)$ \Comment{$x$ represents $\mathbf{E}[\ln(\sin^2(\theta_{\ell+1}))]$}
            \State $\theta_{\ell+1} = \arcsin(e^{\frac{x}{2}})$ 
        \EndFor
        \item Final angle $= \theta_L$
    \end{algorithmic}
\end{algorithm}

\Cref{algo:update_rule} predicts the angle at the final layer on initialization based solely on the network architecture $n_1,n_2,\ldots n_L$. \Cref{fig:simulations} demonstrates how networks which exhibit this type of degeneracy empirically tend to perform worse after training.
\begin{figure}[h]
    \centering
    \includegraphics[scale=0.55]{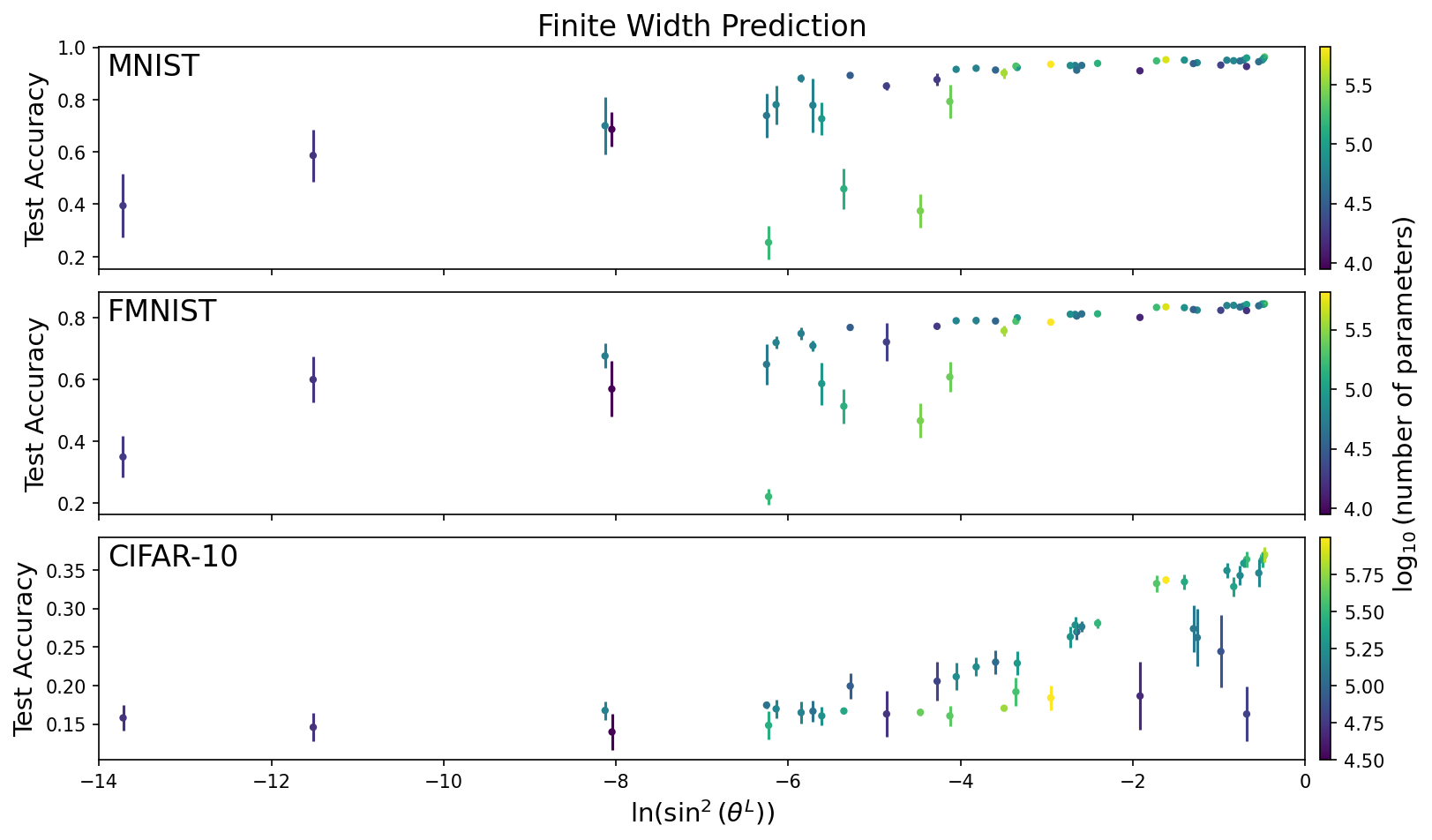}
    \caption[Comparison of Network Degeneracy to Training Performance]{\review{We compare 45 different network architectures trained on the MNIST \cite{mnist}, Fashion-MNIST \cite{fmnist}, and CIFAR-10 \cite{cifar} datasets 10 times each. Using the architecture of the network and \Cref{algo:update_rule}, we predict the angle between 2 orthogonal inputs at the final output layer of the network on initialization. We express the angle as $\ln(\sin^2(\theta_L))$, to follow the form used when  developing the finite width approximations. The angle is plotted against the accuracy of each network on the test data after training, with error bars representing a 95\% confidence interval across the 10 runs. We observe that small angle $\theta_L$ is related to lower test accuracy. All networks are trained using 1 epoch, batch size $=100$, categorical cross-entropy loss, the ADAM optimizer, and default learning rate in the Keras module of TensorFlow \cite{tensorflow}. See Appendix \ref{app:network_architectures} for details on all of the network architectures used. The code which produced this figure can be found at \href{https://github.com/camjakub/Depth-Degeneracy-in-Neural-Networks}{\texttt{GitHub link}}.}}
    \label{fig:simulations}
\end{figure} When \Cref{algo:update_rule} predicts that $\theta$ is small at initialization, this serves a warning that the network may train poorly i.e. the test accuracy seems to be lower. Before going through the computationally expensive process of training many networks to assess their performance, this prediction could be used to quickly filter out network architectures that are unlikely to perform well due to excessive degeneracy.

\subsubsection{Comparison to Infinite Width Update Rule}
\label{sec:infinite_width_comparison}

We compare the performance of Algorithm \ref{algo:update_rule} (which takes into account the layer size $n$) to the infinite width update rule (given in Approximation \ref{approx:infinite}). Since the infinite width prediction cannot account for the differences in layer widths, all networks with the same depth have the same prediction. In contrast, our method considers both the depth and width of each layer to predict how the angle propagates layer-by-layer through the network. \Cref{fig:comparison}-Left illustrates how our method yields different angle predictions for different architectures with the same depth, while the infinite width method does not. \Cref{fig:comparison}-Right shows the how the infinite width predictions differ from our ``finite width'' method which takes into account fluctuations of size $\mathcal{O}(n^{-1})$ in each layer.

\begin{figure}[h]
    \centering
    \includegraphics[scale=0.48]{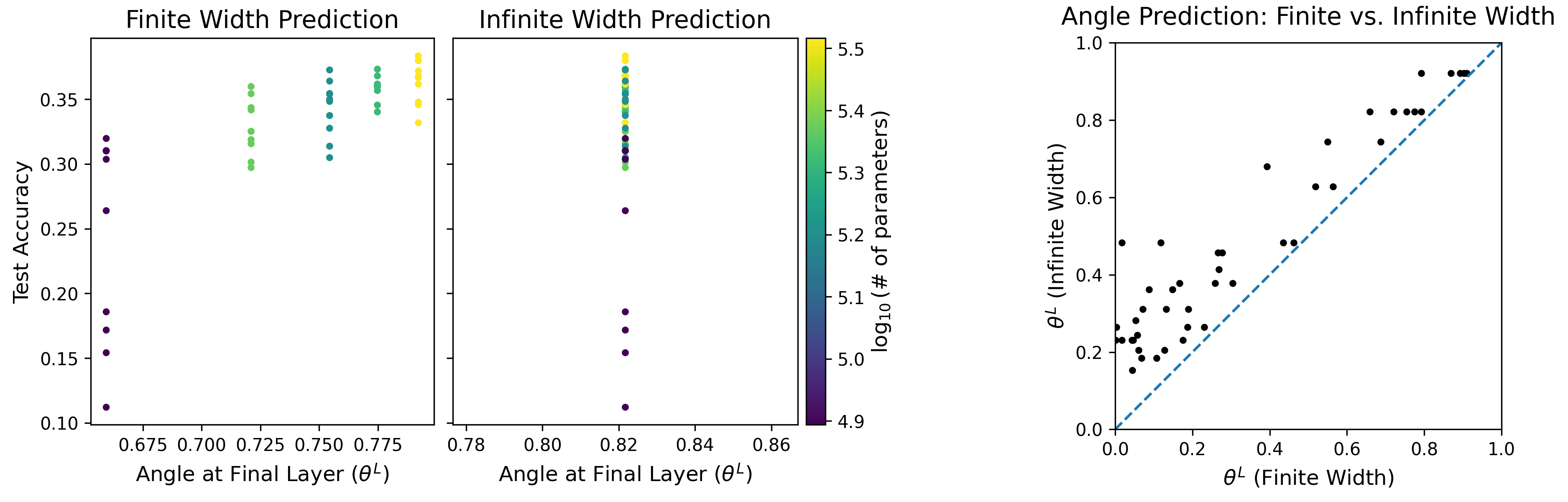}
    \caption[Finite Versus Infinite Width Angle Prediction Comparison]{\review{Left: Comparison of the finite and infinite width predictions for 5 network architectures with a depth of $L=3$ trained 10 times each on the CIFAR-10  dataset \cite{cifar}. The infinite width predicts the same final angle for all networks, since it only depends on network depth. Right: Using the same 45 network architectures as in \Cref{fig:simulations}, we plot a comparison of the predicted angle $\theta_L$ using \Cref{algo:update_rule} (finite width) versus the infinite width prediction. We see that the infinite width prediction tends to underestimate the rate at which $\theta_\ell$ tends towards 0.}}
    \label{fig:comparison}
\end{figure} }

\subsection{J Functions and Infinite Width Limits}
\label{sec:intro_jfunc}
In 2009, \citet{cho-saul} introduced the $p$-th moment for correlated ReLU-Gaussians, which they denoted with the letter $J$,
\begin{equation}\label{eq:J_ChoSaul}
J_p(\theta) := 2\pi \mathbf{E}\left[ \varphi^p(G) \varphi^p(\hat{G})  \right],
\end{equation}
where $p \in \mathbb{N}$, $\varphi(x) = \max\{x,0\}$ is the ReLU function, and $G, \hat{G} \in \mathbb{R}$ are marginally two standard $\mathcal{N}(0,1)$ Gaussian random variables with correlation $\cov(G,\hat{G})=\cos(\theta)$. This quantity has found numerous applications for infinite width networks. One simple application of $J_1$ appears in the infinite width approximation for $\cos(\theta_\ell)$, where $\ell$ is fixed and we take the limit $n_1, n_2, \ldots n_\ell \to \infty$ (see \Cref{app:infinite_width} for a detailed derivation):
\begin{approximation} \label{approx:infinite}
The infinite-width approximation for the angle $\theta_{\ell+1}$ given $\theta_\ell$ is    
\begin{equation} \label{eq:infinte_width_update}
\cos\left( \theta_{\ell+1} \right) = \frac{J_1( \theta_\ell )}{\pi}=\frac{\sin(\theta_\ell)+(\pi-\theta_\ell)\cos(\theta_\ell)}{\pi}. 
\end{equation}
\end{approximation}
The formula for $J_1$ is the $p=1$ case of a remarkable explicit formula for $J_p$ derived by \citet{cho-saul} namely, 
$$J_p(\theta) = (-1)^p(\sin \theta)^{2p+1} \left(\frac{1}{\sin\theta} \frac{\partial}{\partial \theta} \right)^p \left(\frac{\pi-\theta}{\sin\theta} \right).$$ 
This allows one to derive asymptotics of $\theta_\ell$ in the infinite width limit, as in Section \ref{sec:consequences}. However, there are several limitations to this approach. Most important is that the infinite width limit is not a good approximation when the network depth $\ell$ is comparable to the network width $n$ (\cite{li_nica_roy}). The infinite width limit uses the law of large numbers to obtain \eqref{eq:infinte_width_update}, thereby discarding random fluctuations. For very deep networks, microscopic fluctuations (on the order of $\mathcal{O}(1/n_\ell)$) from layer to layer can accumulate over $\ell$ layers to give macroscopic effects. This is why the infinite width predictions for $\theta_\ell$ are not a good match to the simulations in Figure \ref{updaterule}; very deep networks are far from the infinite width limit in this case.  See Figure \ref{updaterule} where the infinite width predictions are compared to finite networks.

Instead, to analyze the evolution of the angle $\theta_\ell$ more accurately, we need to do something more precise than the law of large numbers to capture the effect of these microscopic fluctuations. This is the approach we carry out in this paper. While the mean only depends on the $p$-th moment functions $J_p$ from \eqref{eq:J_ChoSaul}, these fluctuations depend on the \emph{mixed} moments, which we denote by $J_{a,b}$ for $a,b \in \mathbb{N}$ as follows\footnote{Note that compared to Cho and Saul's definition for $J_p$, we omit the factor of $2\pi$ in our definition of $J_{a,b}$. The factor of $2\pi$ seems natural when $a+b$ is even (like the case $a=b=p$ that Cho-Saul considered), but when $a+b$ is odd a different factor of $2\sqrt{2\pi}$ appears! Therefore the factor of $2\pi$ would confuse things in the general case (see Table \ref{jtable}). The correct translation between Cho-Saul $J_p$ and our $J_{a,b}$ is $J_p = 2\pi J_{p,p}$.}
\begin{equation}\label{eq:J_ab_def}
J_{a,b}(\theta) := \mathbf{E}\left[ \varphi^a(G) \varphi^b(\hat{G})  \right],
\end{equation}
with $G,\hat{G}$ again as in \eqref{eq:J_ChoSaul} are marginally $\mathcal{N}(0,1)$ with correlation $\cos(\theta)$. In  \Cref{sec:expected_value} we carry out a detailed asymptotic analysis to write the evolution of $\theta_\ell$ in terms of the mixed moments $J_{a,b}$.
In order to make useful predictions, one must also calculate a formula for $J_{a,b}(\theta)$. Unfortunately, the method that Cho-Saul originally proposed for this does \emph{not} seem to work when $a\neq b$. This is because that method used contour integrals, and relied on using certain trig identities  which do not hold when $a\neq b$. Instead, in \Cref{sec:j_functions}, we introduce a new method, based on Gaussian integration by parts, to compute $J_{a,b}$ for general $a,b$ via a recurrence relation. By serendipity\footnote{This connection was first noticed by calculating the first few $J$ functions, and then using the On-Line Encyclopedia of Integer Sequence to discover the connection to Bessel number (\url{https://oeis.org/A001498}).}, we find a remarkable combinatorial connection between $J_{a,b}$ and the Bessel numbers (\cite{bessel}), which allows one to find an explicit (albeit complicated) formula for $J_{a,b}$ in terms of binomial coefficients. The formula for the first few functions are shown in Table \ref{jtable}, and the general explicit formula is presented in Theorem \ref{thm:j_ab}.

\renewcommand{\arraystretch}{1.5}
\begin{table}[h!]
\begin{center}
\begin{tabular}{|c|c|c|c|c|}
      \hline \diagbox[width=1.25cm,height=1cm]{$a$}{$b$} & 0 & 1 & 2 & 3 \\ \hline
     0 & $\frac{\scriptstyle\pi-\theta}{2\pi}$ & $\frac{\scriptstyle \cos\theta+1}{2\sqrt{2\pi}}$ & $\frac{\scriptstyle(\pi-\theta)+\sin\theta\cos\theta}{2\pi}$ & $\frac{\scriptstyle 2(\cos\theta+1) + \sin^2\theta\cos\theta}{2\sqrt{2\pi}}$\\ \hline 
     1 &  & $\frac{\scriptstyle \sin\theta + (\pi-\theta)\cos\theta}{2\pi}$ & $\frac{\scriptstyle (\cos\theta+1)^2}{2\sqrt{2\pi}}$& $\frac{\scriptstyle3(\pi-\theta)\cos\theta + \sin\theta\cos^2\theta + 2\sin\theta}{2\pi}$\\ \hline
     2 &  &  & $\frac{\scriptstyle (\pi-\theta)(2\cos^2\theta+1)+3\sin\theta\cos\theta}{2\pi}$ & $\frac{\scriptstyle 3\cos\theta(\cos\theta+1)^2+2(\cos\theta+1)+\sin^2\theta\cos\theta}{2\sqrt{2\pi}}$ \\ \hline
     3 &  &  &  & \begin{tabular}{@{}c@{}} $\frac{\scriptstyle(\pi-\theta)(6\cos^2 \theta+9)\cos \theta}{2\pi}$ \\ $\frac{\scriptstyle+5\sin\theta\cos^2\theta + (6\cos^2\theta + 4)\sin\theta}{2\pi}$ \end{tabular} \\ \hline
\end{tabular}
\end{center}
\caption{Table of formulas for the first few $J$ functions. Note that all entries have a denominator either $c_0 = 2\pi$ or $c_1 = 2\sqrt{2\pi}$ depending on the parity of $a+b$. These generalize $J_p(\theta)$ of \eqref{eq:J_ChoSaul} which appear on the diagonal of this table. Note that $J_{a,b} = J_{b,a}$ so only upper triangular entries are shown. An explicit formula for all $J_{a,b}$ is derived in \Cref{sec:explicit_formula}.}
\label{jtable}\
\end{table}

\subsection{Outline}

The two main contributions of this paper are to prove Theorem \ref{thm:mean_var} for the evolution of $\theta_\ell$ in terms of the mixed moments $J_{a,b}$, and then separately derive an explicit formula, Theorem \ref{thm:j_ab}, for any of the mixed moments $J_{a,b}$. Combined, these allow for the explicit formula Corollary \ref{thm:mean_var_exp} and the useful simpler Approximations \ref{approx:simple} and \ref{approx:Gaussian}. See Figure \ref{updaterule} and Figure \ref{fig:theta_plots} for comparisons of these predictions to Monte-Carlo simulations. We also believe the methods proposed here are flexible enough to be modified to apply to non-linearities other than ReLU and to different neural network architectures beyond fully-connected networks in future work.

\Cref{sec:relu_nets} contains the analysis of the angle process and predicted distribution of $\ln(\sin^2(\theta_\ell))$ in deep ReLU networks. \Cref{sec:expected_value} covers our approximation of $\mathbf{E}[\ln(\sin^2(\theta_{\ell+1}))]$ which leads to the rule for $\theta_\ell$ as in equation \eqref{eq:update_simple}, while \Cref{sec:variance} outlines our approximation for $\var[\ln(\sin^2(\theta_{\ell+1}))]$. 

In \Cref{sec:j_functions}, we cover the derivation of the explicit formula for the $J$ functions. We state the main results of this section in \Cref{sec:j_main_results}, and cover the mathematical tools needed to solve the expectations using Gaussian integration by parts in \Cref{sec:G_I_P}. We develop the formula for $J_{a,b}$ by first finding a recursive formula in \Cref{sec:recursive_formula}, which reveals a connection between the $J$ functions and the Bessel numbers. This recursion is studied to develop an explicit formula for $J_{a,b}$ in \Cref{sec:explicit_formula}. 


\section{ReLU Neural Networks on Initialization}
\label{sec:relu_nets}
\begin{table}[ht] 
\begin{center}
\begin{tabular}{ |c|c| } 
\hline
 \textbf{Symbol} & \textbf{Definition} \\ \hline\hline
 $x \in \mathbb{R}^{n_{in}}$ & Input (e.g. training example) in the input dimension $n_{in}\in\mathbb{N}$ \\ \hline
 $\ell \in \mathbb{N}$ & Layer number. $\ell=0$ is the input \\ \hline
 $n_\ell \in \mathbb{N}$ & Width of hidden layer $\ell$ (i.e. number of neurons in layer $\ell$)\\ \hline
 $W^\ell \in \mathbb{R}^{n_{\ell+1} \times n_\ell}$ & Weight matrix for layer $\ell$. Initialized with iid standard Gaussian entries \\
 & $W^{\ell}_{a,b} \sim \mathcal{N}(0,1)$ \\ \hline
 $\varphi: \mathbb{R}^n \to \mathbb{R}^n$ & Entrywise ReLU activation function $\varphi(x)_i = \varphi(x_i)=\max\{x_i,0\} $ \\ \hline
 $z^\ell(x) \in \mathbb{R}^{n_\ell}$ & Pre-activation vector in the $\ell^{\text{th}}$ layer for input $x$ (a.k.a logits of layer $\ell$) \\ 
 & $z^1(x) := W^1x, \quad\quad  z^{\ell+1}(x) := \sqrt{\frac{2}{n_\ell} }W^{\ell+1} \varphi(z^\ell(x)).$\\ \hline
 $\varphi_\alpha^\ell, \varphi_\beta^\ell \in \mathbb{R}^{n_\ell}$ & Post-activation vector on inputs $x_\alpha, x_\beta$ respectively \\
 & $\varphi_\alpha^\ell := \varphi(z^\ell(x_\alpha)), \quad\quad \varphi_\beta^\ell := \varphi(z^\ell(x_\beta))$ \\ \hline
 $\theta_\ell \in [0,\pi]$ & Angle between $\varphi_\alpha^\ell$ and  $\varphi_\beta^\ell$ defined by $\cos(\theta_\ell) := \frac{\langle \varphi_\alpha^\ell, \varphi_\beta^\ell \rangle}{\Vert \varphi_\alpha^\ell \Vert \Vert \varphi_\beta^\ell \Vert }$ \\ \hline
 $\normratioo \in \mathbb{R}$ & Shorthand for the ratio $\normratioo := \normratio$ \\
\hline
\end{tabular}
\end{center}
\caption{Definition and notation used for fully connected ReLU neural networks.}
\label{tbl:notation}
\end{table}
In this section, we analyze ReLU neural networks and show how the the mixed moments $J_{a,b}$ appear in evolution of the angle $\theta_\ell$ on initialization. We define the notation we use for a fully connected ReLU neural network, along with other notations we will use in Table \ref{tbl:notation}. 
%
Note that the factor of $\sqrt{2/n_\ell}$ in our definition is implementing the so called He initialization \citep{he_initialization}, which ensures that $\mathbf{E}[\Vert z^\ell \Vert^2]=\Vert x \Vert^2$ for all layers $\ell$. This initialization is known to be the ``critical'' initialization for taking large limits of the network \citep{principles_deep_learning, hayou}. Given this neural network, we wish to study the evolution of 2 inputs $x_\alpha$ and $x_\beta$ as they traverse through the layers of the network. Specifically, we wish to study how the angle $\theta$ between the inputs changes as the inputs move from layer to layer.

The starting point for our calculation is to notice that because the weights are Gaussian, the values of $\phaa,\phbb$ are jointly Gaussian given the vectors of $\pha,\phb$. In fact, it turns out that by properties of Gaussian random variables, one only needs to know the values of the scalars $\Vert \pha \Vert,\Vert \phb \Vert$ and $\theta_\ell$ to understand the full distribution of $\phaa,\phbb$. (see \Cref{app:identities} for details) By using the positive homogeneity of the ReLU function $\varphi( \lambda x) = \lambda \varphi(x)$ for $\lambda >0$, we can factor out the effect of the norm of each vector in layer $\ell$. After some manipulations, these ideas lead us to the following identities that are the heart of our calculations; a full derivation of these quantities are provided in \Cref{app:identities} and \ref{app:cauchy_binet}.
\begin{align}
    \|\phaa\|^2 = \frac{\|\pha \|^2}{n_\ell} \sum_{i=1}^{n_\ell}2\ph^2(G_i), \quad&\quad \| \phbb\|^2 = \frac{\|\phb \|^2}{n_\ell} \sum_{i=1}^{n_\ell}2\ph^2(\hat{G}_i), \label{sums_1} \\ 
    \langle \phaa, \phbb\rangle &= \frac{\|\pha\| \|\phb\|}{n_\ell}\sum_{i=1}^{n_\ell}2 \ph(G_i) \ph(\hat{G}_i), \label{sums_2} \\
    \frac{\|\phaa \|^2 \|\phbb \|^2}{\| \pha \|^2 \|\phb \|^2} \sin^2(\theta_{\ell+1}) &=  \frac{2}{n_\ell^2} \sum_{i,j=1}^{n_\ell} \left(\ph(G_i)\ph(\hat{G}_j) - \ph(G_j)\ph(\hat{G}_i) \right)^2,
    \label{sums_4}
\end{align}
where $G_i$, $\hat{G}_i$ are all marginally $\mathcal{N}(0,1)$,
%
with correlation $\cov(G_i,\hat{G}_i)=\cos(\theta_\ell)$ and independent for different indices $i$. The identity in (\ref{sums_4}) is derived using the \emph{determinant of the Gram matrix} for vectors $\phaa,\; \phbb$ (full derivation given in \Cref{app:cauchy_binet}).  Combining the equations in (\ref{sums_1}) gives us a useful identity for the ratio $\normratioo$, namely:
\begin{gather}
    \normratioo = \frac{4}{n^2_\ell} \sum_{i,j=1}^{n_\ell} \ph^2(G_i) \ph^2(\hat{G}_j) \label{eq:norm_ratio_identity}.
\end{gather}
Given some $\theta_\ell$, we wish to predict the behaviour of $\theta_{\ell+1}$. Rather than studying $\theta_{\ell+1}$ directly, we instead study the quantity $\ln(\sin^2(\theta_{\ell+1}))$. This allows us to use convenient approximations and identities for quantities we are interested in. And indeed, a post-hoc analysis shows that as $\theta \to 0$, the random variable $\ln \sin^2(\theta_{\ell+1})$ has a \emph{non-zero constant} variance which depends only on $n_\ell$. This is in contrast to $\theta_\ell$ itself which has variance tending to \emph{zero}. This is one reason why the Gaussian approximation for $\ln \sin^2(\theta_\ell) \in (-\infty,0] \subset (-\infty,\infty)$ works well, whereas Gaussian approximations for $\theta_\ell$ or $\cos(\theta_\ell) \in [-1,1]$ are less accurate. \review{This observation can equivalently be understood as the observation that the random fluctuations seem to be multiplicative, rather than additive, which is why taking the log makes them more amenable to calculation.} We first derive a formula for $\mathbf{E}\left[ \ln(\sin^2(\theta_{\ell+1}))\right]$. 

\subsection{Expected Value}
\label{sec:expected_value}
In this section, we show how to compute the expected value of $\ln(\sin^2(\theta^{\ell}))$ in terms of the $J$ functions as in Theorem \ref{thm:mean_var}.  Firstly, we rewrite this expectation as the difference
\begin{gather}
\mathbf{E}\left[\ln(\sin^2(\theta_{\ell+1}))\right] = \mathbf{E}\left[\ln\left(\normratioo\sin^2(\theta_{\ell+1})\right)\right] - \mathbf{E}\left[\ln\left(\normratioo\right)\right]. 
    \label{E_ln_sin}
\end{gather}
The two random variables $\normratioo$ and $\normratioo \sin^2(\theta_{\ell+1})$ in \eqref{E_ln_sin} both have interpretations in terms of sums of Gaussians as in \eqref{sums_4} and \eqref{eq:norm_ratio_identity} which makes it possible to calculate their moments in terms of the $J$ functions. To enable our use of the moments here, we use the following approximation of $\ln(X)$ for a random variable $X$, which is based on the Taylor expansion for $\ln(1+x)=x-\frac{1}{2}x^2+\ldots$ (a full derivation is given in \Cref{app:expected_value}):
\begin{align}
    \ln(X) = \ln(\mathbf{E}[X]) + \frac{X-\mathbf{E}[X]}{ \mathbf{E}[X]} - \frac{(X-\mathbf{E}[X])^2}{2 \mathbf{E}[X]^2} + \epsilon_2\left(\frac{X - \mathbf{E}[X]}{\mathbf{E}[X]}\right), \label{eq:lnX_approx}
\end{align}
where $\epsilon_2(x)$ is the Taylor remainder term in $\ln(1+x) = x - \frac{x^2}{2} +\epsilon_2(x)$ and satisfies $\epsilon_2(x)=\mathcal{O}(x^3)$. Applying this approximation to the terms appearing on the right hand side of \eqref{E_ln_sin}, and taking expected value of both sides, we obtain the estimates 
\begin{align*}
    \mathbf{E}\left[\ln\left(\normratioo\sin^2(\theta_{\ell+1})\right)\right] &= \ln\left(\mathbf{E}\left[\normratioo\sin^2(\theta_{\ell+1})\right]\right) - \frac{\var\left[\normratioo\sin^2(\theta_{\ell+1})\right]}{2\mathbf{E}\left[\normratioo\sin^2(\theta_{\ell+1})\right]^2} + \mathcal{O}(n^{-2}_\ell), \\
    \mathbf{E}\left[\ln\left(\normratioo\right)\right] &= \ln\left(\mathbf{E}\left[\normratioo \right]\right) - \frac{\var\left[\normratioo \right]}{2\mathbf{E}\left[\normratioo \right]^2}+ \mathcal{O}(n^{-2}_\ell).
\end{align*}
To control the error here, we have used here the fact that $\normratioo$ and $\normratioo\sin^2(\theta_{\ell+1})$ can be written as averages over random variables as in  (\ref{sums_1} - \ref{sums_4}). This allows us to show the 3rd central moments for $\normratioo$ and $\normratioo\sin^2(\theta_{\ell+1})$ are $\mathcal{O}(n^{-2}_{\ell})$; see \Cref{app:expected_value} 
for details. This approximation is convenient because we are able to calculate the values on the right hand side of the equations in terms of the moments $J_{a,b}$ by expanding/taking expectations of the representations (\ref{sums_1} - \ref{sums_4}).  The key quantities we calculate are
\begin{align}
    \mathbf{E}\left[\normratioo \right] &= \frac{4J_{2,2}-1}{n_\ell}+1, \label{eq:exp_norm}\\
    \var\left[\normratioo \right] &= \frac{4}{n_\ell}(J_{2,2} +1) + \frac{16}{n_\ell^2}\left(2J_{4,2} - \frac{5}{2}J_{2,2} + J_{2,2}^2 + \frac{5}{8} \right) + \mathcal{O}\left(n_\ell^{-3} \right), \label{eq:var_norm}\\
     \mathbf{E}\left[\normratioo \sin^2(\theta_{\ell+1})\right]
     &= \frac{(n_\ell-1)(1-4J_{1,1}^2)}{n_\ell}, \label{eq:exp_norm_sin}\\
     \var\left[\normratioo \sin^2(\theta_{\ell+1})\right] &= \frac{8\left( -8J_{1,1}^4 + 8J_{1,1}^2J_{2,2} + 4J_{1,1}^2 - 8J_{1,1}J_{3,1} + J_{2,2} + 1\right)}{n_\ell} 
    + \mathcal{O}\left(n_\ell^{-2} \right), \label{eq:var_norm_sin}
\end{align}
 where $J_{a,b}=J_{a,b}(\theta_\ell)$. These formulas are calculated in \Cref{app:exp_value_calculation} and \Cref{app:formula_calculation} by a combinatorial expansion using the representations from (\ref{sums_1}-\ref{sums_4}). Combining these gives the result for $\mu(\theta,n)$ in Theorem \ref{thm:mean_var}. Note that to obtain a more accurate approximation, we would simply include more terms in the variance expressions in (\ref{eq:var_norm}, \ref{eq:var_norm_sin}).

\subsection{Variance of $\ln(\sin^2(\theta_{\ell+1}))$}
\label{sec:variance}
In this section, we show how to compute the variance of $\ln(\sin^2(\theta^{\ell}))$ in terms of the $J$ functions as in Theorem \ref{thm:mean_var}. We can rewrite $\var[\ln(\sin^2(\theta_{\ell+1}))]$ in the following way:
\begin{gather}
    \var[\ln(\sin^2(\theta_{\ell+1}))] = \var\left[ \ln\left( \normratioo \sin^2(\theta_{\ell+1})\right) - \ln\left( \normratioo\right)\right] \label{eq:ln_variance_trick}\\
    = \var\left[ \ln\left( \normratioo \sin^2(\theta_{\ell+1})\right)\right] + \var\left[  \ln\left( \normratioo\right)\right] - 2\cov\left(  \ln\left( \normratioo \sin^2(\theta_{\ell+1})\right),  \ln\left( \normratioo\right)\right). \nonumber
\end{gather}

\noindent We have now expressed this in terms of $R^{\ell+1}$ and $\normratioo \sin^2(\theta_{\ell+1})$ which will allow us to use identities as in (\ref{sums_1} - \ref{sums_4}) in our calculations. \Cref{app:variance} and \Cref{app:covariance_approx} cover the method used to approximate the unknown variance and covariance terms above.
%
%
Once again, we control the error term arising from moments in the error term of the Taylor series by using representation as sums (\ref{sums_1} - \ref{sums_4}). 
\noindent We have already calculated most of the quantities on the right hand side already in our calculation for $\mu(\theta,n)$. The only new term is
\begin{gather*}
    \cov\left( \normratioo \sin^2(\theta_{\ell+1}),\normratioo \right) = \frac{1}{n_\ell}\left(16J_{1,1}^2 - 32J_{1,1}J_{3,1} + 8J_{2,2} +8 \right) + \mathcal{O}\left(n_\ell^{-2} \right).
\end{gather*}

This is again computed by a combinatorial expansion of the sums (\ref{sums_1}-\ref{sums_4}). \noindent (Full calculation given in \Cref{app:formula_calculation}). We now have solved for all of the functions needed to perform our approximation of $\var[\ln(\sin^2(\theta_{\ell+1}))]$. Putting it together, we end up with the expression for $\sigma^2(\theta,n)$ as in \eqref{eq:sigma_sq}. We compare the predicted probability distribution of $\ln(\sin^2(\theta))$ using our formulas $\mu(\theta,n)$ and $\sigma^2(\theta,n)$ to empirical probability distributions in Figure \ref{updaterule}.




\section{Explicit Formula for the Mixed-Moment J Functions}
\label{sec:j_functions}
In this section we develop a combinatorial method that allows us \review{to} compute exact formulas for the $J$ functions. The method is to use Gaussian integration by parts to find a recurrence relationship between the moments $J_{a,b}$, and then solve it explicitly. We begin by generalizing the definition of $J_{a,b}$ from \eqref{eq:J_ab_def} to include $a=0$ and/or $b=0$ as follows. Let $G$, $W$ be \emph{independent} $\mathcal{N}(0,1)$ variables. Then, we define the functions $J_{a,b}(\theta)$ as 
\begin{align}
    J_{a,b}(\theta) = \mathbf{E}[G^a (G\cos\theta  + W\sin\theta )^b \;1\{G>0\}\; 1\{G\cos\theta  + W\sin\theta  >0 \}]\label{eq:j_generalization},
\end{align}
where $a,b \in \mathbb{N} \cup \{0\}$.  Note that $G \cos\theta+ W\sin\theta = \hat{G}$ is marginally $\mathcal{N}(0,1)$ and has correlation $\cos(\theta)$ with $G$, matching the original definition. The ReLU function satisfies the identity $\varphi(x)^a = x^a 1\{x>0\}$ for $a\geq 1$, so \eqref{eq:j_generalization} generalizes \eqref{eq:J_ab_def} to the case $a=0$. We also note that $J_{a,b}(\theta) = J_{b,a}(\theta)$ for all $a,b \in \mathbb{N}\cup\{0\}$.

\remark{ \review{Note that \eqref{eq:j_generalization} can equivalently be written in terms of the correlation $\rho = \cos \theta$ and $\sqrt{1-\rho^2} = \sin \theta$ as follows. 

$$G^a (\rho G + \sqrt{1-\rho^2} W)^b 1\{ G >0 \} 1\{\rho G + \sqrt{1-\rho^2} W\}$$

Some authors prefer to work with the correlation $\rho$ rather than the angle $\theta$. In this work we choose to work with the angle $\theta$ since our technique works well to directly see how quickly $\theta \to 0$.}}


\subsection{Statement of Main Results and Outline of Method}
\label{sec:j_main_results}

By using the method of Gaussian integration by parts, we are able to derive recurrence relations for the $J_{a,b}$ functions. Since the definition of $J_{a,b}$ involves the indicator function $1\{G > 0\}$, we must make sense of what the derivative of this function means for the purposes of integration by parts; see Section \ref{sec:G_I_P} where this is carried out.  Then, by use of the generalized Gaussian integration by parts formula, we obtain the following recurrence relations for $J_{a,b}$.

\begin{prop}[Recurrence relations for $J_{a,b}$]
\label{prop:rec_relations}
For $a \geq 2$, the sequence $J_{a,0}$ satisfies the recurrence relation:
\begin{equation}
    J_{a,0}(\theta) = (a-1) J_{a-2,0}(\theta) + \frac{\sin^{a-1}\theta \cos \theta}{c_{a \bmod 2}}(a-2)!!,
\end{equation}
where $c_0 = 2\pi$, $c_1 = 2\sqrt{2\pi}$. For $a\geq 2$, and $b\geq 1$, the collection $J_{a,b}$ satisfies the following two-index recurrence relation:
\begin{equation}
J_{a,b}(\theta) = (a-1)J_{a-2, b}(\theta) + b \cos \theta J_{a-1,b-1}(\theta).
\end{equation}
\end{prop}

The same integration by parts technique that yields the recurrence relation also makes it easy to evaluate the first few $J$ functions. They are as follows:

\begin{prop}[Explicit Formula for $J_{0,0},J_{1,0},J_{1,1}$]
\label{prop:j_small}

$J_{0,0}$, $J_{1,0}$, and $J_{1,1}$ are given by
\begin{equation} \label{eq:J_small}
    J_{0,0}(\theta) = \frac{\pi - \theta}{2\pi}, \quad\quad
    J_{1,0}(\theta) = \frac{1+\cos \theta}{2 \sqrt{2 \pi}}, \quad\quad
    J_{1,1}(\theta) = \frac{\sin \theta + (\pi - \theta)\cos \theta }{2 \pi}.
\end{equation}
\end{prop}
See \Cref{app:low_order} for a derivation of these quantities. Note that \citet{cho-saul} have previously discovered the formulas for $J_{0,0}$ and $J_{1,1}$ by use of a completely different contour-integral based method.

The combination of Propositions \ref{prop:rec_relations} and \ref{prop:j_small} make it possible to practically calculate any value of $J_{a,b}$ when $a,b$ are not too large. However, by serendipity, we are able to find remarkable explicit formulas for $J_{a,b}$, which we report below.





\begin{prop}[Explicit Formulas for $J_{a,0}(\theta)$, $J_{a,1}(\theta)$]
\label{prop:J_a0}
 Let $a \geq 2$. Then, $J_{a,0}$ and $J_{a,1}$ are explicitly given by the following:
\begin{align*}
    J_{a,0}(\theta) = (a-1)!!\left( J_{a \bmod 2,0} + \frac{\cos \theta }{c_{a \bmod 2}} \sum_{\substack{i \not\equiv a(\bmod 2) \\ 0<i<a}} \frac{(i-1)!!}{i!!} \sin^i \theta \right),
\end{align*}
where $c_0 = 2\pi$, $c_1 = 2\sqrt{2\pi}$. We can then use the explicit formula for $J_{a,0}$ in the formula for $J_{a,1}$:
\begin{align*}
    J_{a,1}(\theta) = (a-1)!! \left(  J_{a \bmod 2,1} + \cos \theta \sum_{\substack{i \not\equiv a(\bmod 2)\\0<i<a}} \frac{J_{i,0}(\theta)}{i!!} \right),
\end{align*}
where an explicit formula for the first term (either $J_{1,0}$ or $J_{1,1}$ depending on the parity of $a$) is given in Proposition \ref{prop:j_small}.
\end{prop}
We can finally express $J_{a,b}$ as a linear combination of $J_{0,n}$ and $J_{1,n}$, as follows. (In light of the previous explicit formulas, this is an explicit formula for $J_{a,b}$.) It turns out that the coefficients are given in terms of two special numbers $P(a,b)$ and $Q(a,b)$ which are known as the Bessel numbers.


\begin{definition}[Bessel numbers] \label{defn:PQ}
The numbers $P(a,b)$ and $Q(a,b)$ are defined as follows,
\begin{align}
    P(a,b) &= \begin{cases}
    \frac{a!}{b! \left( \frac{a-b}{2}\right)! 2^{\frac{a-b}{2}} },  &  a \geq b,\; a \equiv b \;(\bmod \; 2) \\
    0, & \text{otherwise}
    \end{cases}, \label{p_explicit} \\
    Q(a,b) &= \begin{cases}
    \frac{\left( \frac{a+b}{2}\right)!}{b!}2^{\frac{b-a}{2}} \sum_{i=0}^{\frac{a-b}{2}} {a+1 \choose i},&  a \geq b,\; a \equiv b \;(\bmod \; 2) \\
    0, & \text{otherwise}
    \end{cases}. \label{q_explicit}
\end{align}
\end{definition}

$P(a,b)$ represents a family of numbers known as the Bessel numbers of the second kind (\cite{bessel}), and $Q(a,b)$ comes from a closely related family of numbers (\cite{pq}). Using these, we can express $J_{a,b}$ as follows.

\begin{theorem}[Explicit Formula for $J_{a,b}(\theta)$]
\label{thm:j_ab}
Let $b \geq 2, a\geq 1, \; b \geq a$. Then, we have the following formula for $J_{a,b}(\theta)$ in terms of $J_{0,n}$ and $J_{1,n}$
\begin{align*}
    J_{a,b} &= \sum_{\substack{i \equiv 0(\bmod 2)\\0 < i\leq a}} (b)_{a-i}  (\cos\theta)^{a-i} \left(P(a,a-i) - Q(a-1, a-1-i)\right)  J_{0,b-a+i} \\
    &+ \sum_{\substack{i \equiv 1(\bmod 2) \\ 0 < i \leq a}} (b)_{a-i} (\cos\theta)^{a-i} Q(a-1, a-i) J_{1,b-a+i}.
\end{align*}
\end{theorem}

\begin{remark}
Since $J_{1,n}$ is also given in terms of $J_{0,n}$, one may further simplify the formula for $J_{a,b}$ to be in terms of only $J_{0,n}$, $J_{0,0}$ and $J_{0,1}$. This substitution yields the following formula. For notational convenience, we will let $\delta := b-a$,
\begin{align*}
    J_{a,b} &= \sum_{\substack{i \equiv 0(\bmod 2)\\0 < i\leq a}}  (b)_{a-i} (\cos\theta)^{a-i} (P(a,a-i) - Q(a-1, a-1-i))  J_{0,\delta+i} \\
    &+ \sum_{\substack{i \equiv 1(\bmod 2) \\ 0 < i \leq a}} (b)_{a-i} (\cos\theta)^{a-i} Q(a-1, a-i) (\delta+i-1)!! J_{(\delta+1)\bmod 2,1} \\ 
    &+  \cos \theta \sum_{\substack{i \equiv 1(\bmod 2) \\ 0 < i \leq a}} \sum_{\substack{j \equiv \delta(\bmod 2) \\ 0<j<\delta+i}} (b)_{a-i} (\cos\theta)^{a-i} Q(a-1, a-i) \frac{(\delta+i-1)!!}{j!!} J_{0,j}.
\end{align*}
\end{remark}



\subsection{Gaussian Integration-by-Parts Formulas}
\label{sec:G_I_P}

In this section, we state two important formulas that together give us the tools for computing the expectations that appear in $J_{a,b}$, \review{based on the well known Gaussian integration-by-parts trick; see for example Chapter 7.2 of the textbook \citep{vershynin}}.

\begin{fact}[Gaussian Integration by Parts]
\label{fact:G_I_P}
Let $G \sim \mathcal{N}(0,1)$ be a Gaussian variable and $f: \mathbb{R} \to \mathbb{R}$ be a differentiable function. Then,
\begin{align}
\label{eq:Gaussian_IBP}
    \mathbf{E}[G f(G)] = \mathbf{E}[f'(G)].
\end{align}
\end{fact}

Using this type of Gaussian integration by parts formula, we can generalize the expected value of Gaussians to derivatives of functions which are not necessarily differentiable. For example the indicator function $1\{x>a\}$ is not differentiable, but for the purposes of computing Gaussian expectation, we can use the following integration formula.

\begin{remark}
\review{An alternative way to view this kind of Gaussian expected value calculation is Wick's formula / Isserlis theorem. The Gaussian integration by parts trick can be thought of as the extension of those formulas from polynomials to arbitrary functions. } 
\end{remark}

\begin{fact}[Gaussian expectations involving $1^\prime\{x>a\}$]
\label{fact:dirac_delta} Let $G$ be a Gaussian variable and $a \in \mathbb{R}$. Let $f: \mathbb{R} \to \mathbb{R}$ such that $\lim_{g\to \infty} f(g)e^{\frac{-g^2}{2}} =0$. Then, using the Gaussian integration by parts formula to assign a meaning to expectations involving the ``derivative of the indicator function'', $1^\prime \{x>a\}$, we have
\begin{align}
\label{eq:E_deriv_one}
    \mathbf{E}[1'\{G>a\} f(G)] = f(a) \frac{e^{\frac{-a^2}{2}}}{\sqrt{2\pi}}.
\end{align}
\label{dirac2}
\end{fact}

\begin{remark}
The purpose of assigning a value to the expectation 
\eqref{eq:E_deriv_one} is to allow one to compute ``honest'' expectations of the form \eqref{eq:Gaussian_IBP} when $f(x)$ involves $1\{x>0\}$; see Lemma \ref{lem:powers_of_phi} for an illustrative example. The final result does \emph{not} require interpreting ``$1^\prime\{x>a\}$''; this is only a useful intermediate step in the sequence of calculations leading to the final result.

The formula can also be understood or proven in a number of different alternative ways. One is simply to say that $1^\prime\{x>a\} = \delta\{x=a\}$ is a ``Dirac delta function'' at $x=a$. A more rigorous way would be to take any differentiable family of functions $1_\epsilon\{x>a\}$ which suitably converge to $1\{x>a\}$ as $\epsilon \to 0$ and then interpret the result as the limit of the expectation $\lim_{\epsilon \to 0} \mathbf{E}[1_\epsilon'\{G>a\} f(G)]$. 
Here is the argument that is used to obtain Fact \ref{fact:dirac_delta}: Applying integration by parts, we formally have
\begin{align*}
    \mathbf{E}[1'\{G>a\} f(G)] &= \intop_{-\infty}^{\infty} 1'\{g>a\}f(g)\frac{e^{\frac{-g^2}{2}}}{\sqrt{2\pi}} dg \\
    &=  \left[1\{ g>a \} f(g)\frac{e^{\frac{-g^2}{2}}}{\sqrt{2\pi}}  \right]_{-\infty}^{\infty} - \intop_{-\infty}^{\infty} 1\{g>a\} \frac{d}{dg} \left( f(g)\frac{e^{\frac{-g^2}{2}}}{\sqrt{2\pi}} \right)dg.
\end{align*}
Note that the first term is 0 by the hypothesis $\displaystyle\lim_{g\to \infty} f(g) e^{\frac{-g^2}{2}} = 0$, and we have then
\begin{align*}
    \mathbf{E}[1'\{G>a\} f(G)]&= - \intop_{a}^{\infty}  \frac{d}{dg} \left( f(g)\frac{e^{\frac{-g^2}{2}}}{\sqrt{2\pi}} \right)dg = - \left[f(g) \frac{e^{\frac{-g^2}{2}}}{\sqrt{2\pi}} \right]_a^\infty = 0+ f(a) \frac{e^{\frac{-a^2}{2}}}{\sqrt{2\pi}}, 
\end{align*}
where we have used the hypothesis on $f$ once again.
\end{remark}

The two facts about Gaussian integration by parts can be combined to create recurrence relations for expectations involving $1\{G>a\}$. A simple example is the following lemma, which we will also use later in our derivation. The proof strategy of this lemma is a microcosm of the proof strategy we use to compute $J_{a,b}$ in general, namely to use Gaussian integration by parts to derive a recurrence relation and initial condition, and then solve.

\begin{lemma}[Moments of $\varphi(G)$] For $k\geq 0$, we have
\label{lem:powers_of_phi}
$$\mathbf{E}[\varphi(G)^k] = \mathbf{E}[G^k 1\{G > 0 \}]  = \begin{cases}
\frac{(k-1)!!}{2} & k\text{ is even} \\
\frac{(k-1)!!}{\sqrt{2\pi}} & k\text{ is odd} \\
\end{cases} = 
\sqrt{2\pi} \frac{(k-1)!!}{c_{ {k-1} \bmod 2}},$$
where $c_0 = 2\pi$ and $c_1 = 2\sqrt{2\pi}$.
\end{lemma}
\begin{proof}
We prove this for even and odd $k$ separately by induction on $k$. The base case for $k=0$ is trivial since $(0-1)!!=1$ is the empty product. The base case $k=1$ follows by first applying \eqref{eq:Gaussian_IBP} with $f(x)=1\{x>0\}$ and then applying \eqref{eq:E_deriv_one} with $f(x)\equiv 1$,
$$\mathbf{E}[\varphi(G)]=\mathbf{E}[G 1\{G>0\}]=\mathbf{E}[1^\prime\{G>0\}] = {\sqrt{2\pi}}^{-1}.$$
Now, to see the induction, we apply \eqref{eq:Gaussian_IBP} with $f(x)=x^{k-1}1\{x>0\}$, $k \geq 2$. Due to the product rule, there are two terms in the derivative,
\begin{align}
\mathbf{E}[\varphi(G)^k]&=\mathbf{E}[G \cdot G^{k-1} 1\{G>0\}]\\
&=(k-1)\mathbf{E}[G^{k-2} 1\{G>0\}]+\mathbf{E}[G^{k-1}1^\prime\{G>0\}] \notag \\
&= (k-1) 
\mathbf{E}[\varphi(G)^{k-2}] + 0, \notag
\end{align}
where we have recognized that the second term is $0$ by application of \eqref{eq:E_deriv_one} with $f(x)=x^{k-1}$ which has $f(0)=0$. The recurrence $\mathbf{E}[\varphi(G)^k] = (k-2)\mathbf{E}[\varphi(G)^{k-2}]$ along with initial condition leads to the stated result by induction.
\end{proof}

\subsection{Recursive Formulas for $J_{a,b}(\theta)$ - Proof of Proposition \ref{prop:rec_relations}}
\label{sec:recursive_formula}


\begin{proof}[Of Proposition \ref{prop:rec_relations}] To find a recursive formula for $J_{a,0}, a\geq 2$, we apply the Gaussian integration by parts formula \eqref{eq:Gaussian_IBP} to $f(x)=x^{a-1}1\{x > 0\} 1\{\cos \theta x + W \sin \theta >0\}$ to evaluate the expected value over $G$ first. When applying product rule there are three terms
\begin{align}
\label{eq:J_a0}
J_{a,0} =& \mathbf{E}[G \cdot G^{a-1} 1\{G>0\} 1\{G \cos \theta  + W \sin \theta > 0 \}] \\
    =& (a-1)\mathbf{E}[G^{a-2} 1\{G>0\} 1\{G \cos \theta  + W \sin \theta  > 0 \}] \notag \\
    &+ \mathbf{E}[G^{a-1} 1\{G>0\} 1'\{G \cos \theta  + W \sin \theta >0\}]\cos \theta \notag \\
    &+ \mathbf{E}[G^{a-1} 1^\prime\{G>0\} 1\{G \cos \theta  + W \sin \theta  > 0\} ]. \notag
\end{align}
The first term is simply $(a-1)J_{a-2,0}$. The last two terms can now be evaluated with the help of \eqref{eq:E_deriv_one}. The last term of \eqref{eq:J_a0} is \eqref{eq:E_deriv_one} with the  function $f(x)=x^{a-1}1\{x \cos \theta  + W \sin \theta > 0\}$ which has $f(0) = 0$ for $a\geq 2$. Therefore, this term simply vanishes.

To evaluate the middle term of \eqref{eq:J_a0}, we introduce a change of variables to express $G \cos \theta + W \sin \theta$ in terms of two other independent Gaussian  variables $Z,W \sim \mathcal{N}(0,1)$
\begin{align}
    Z = G \cos \theta  + W \sin \theta,  \quad \quad & G = Z \cos \theta  + Y \sin \theta,  \label{cov}\\
    Y = G \sin \theta  - W \cos \theta,  \quad \quad & W = Z \sin \theta  - Y \cos \theta , \notag
\end{align}
where $Y,Z$ iid $\mathcal{N}(0,1)$. Under this change of variables, $J_{a,0},\; a\geq 2$ is setup to apply \eqref{eq:E_deriv_one} with $f(x)=1\{x \cos\theta  + Y \sin\theta \}^{a-1} 1\{x \cos\theta  + Y \sin \theta  > 0\}$
\begin{align*}
    J_{a,0} &= (a-1)J_{a-2,0} + \mathbf{E}[G^{a-1} 1\{G>0\} 1'\{G \cos \theta  + W \sin \theta >0\}]\cos \theta \\
     &= (a-1)J_{a-2,0} + \mathbf{E}[(Z\cos \theta  + Y \sin \theta )^{a-1} 1\{Z \cos \theta  + Y \sin \theta  > 0 \} 1'\{Z>0\}]\cos \theta \\
     &= (a-1)J_{a-2,0} + \mathbf{E}[(0 + Y \sin \theta )^{a-1} 1\{0 + Y \sin \theta  > 0 \}]\frac{1}{\sqrt{2\pi}} \cos \theta \\
     &= (a-1) J_{a-2,0} + \frac{\sin^{a-1}\theta \cos \theta}{c_{a \bmod 2}}(a-2)!!,
\end{align*}
 where we have applied Lemma \ref{lem:powers_of_phi} to evaluate the last expectation.

 A similar argument is used to find the recursive formula for $J_{a,b},\; a\geq 2, b\geq 1$, by using \eqref{eq:Gaussian_IBP} with the function $f(x)=x^{a-1} (x\cos\theta + W \sin\theta)^b 1\{x > 0\} 1\{x\cos \theta + W \sin \theta >0\}$. There are 4 terms in the product rule derivative. Fortunately in this case, the last two terms are simply zero by application of \eqref{eq:E_deriv_one} since the expressions vanish when $G=0$, so we get
\begin{align*}
    J_{a,b} &= \mathbf{E}[G \cdot G^{a-1} (G \cos \theta + W \sin \theta)^b 1\{G>0\} 1\{G \cos \theta + W \sin \theta > 0 \}] \\ 
    &= \mathbf{E}[(a-1)G^{a-2} (G \cos \theta + W \sin \theta)^b 1\{G>0\} 1\{G \cos \theta + W \sin \theta > 0 \}] \\
    &\;\;\;\;\;\;+ \mathbf{E}[G^{a-1} b\cos \theta (G \cos \theta + W \sin \theta)^{b-1} 1\{G>0\} 1\{G \cos \theta + W \sin \theta > 0 \}] \\
    &\;\;\;\;\;\;+ \mathbf{E}[G^{a-1} (G \cos \theta + W \sin \theta)^b 1'\{G>0\} 1\{G \cos \theta + W \sin \theta > 0 \}] \\
    &\;\;\;\;\;\;+ \mathbf{E}[G^{a-1} (G \cos \theta + W \sin \theta)^b 1\{G>0\} 1'\{G \cos \theta + W \sin \theta > 0 \} \cos \theta] \\ 
    &=(a-1)J_{a-2, b} + b \cos \theta J_{a-1,b-1} +0 +0, 
\end{align*}
as desired.
\end{proof}


\subsection{Solving the Recurrence to get an Explicit Formula for $J_{a,b}(\theta)$ - Proof of Theorem \ref{thm:j_ab}}
\label{sec:explicit_formula}
 Solving the recurrence for the sequences $J_{a,0}$ and $J_{a,1}$ to get the claimed explicit formula for $J_{a,0}$ is a simple induction proof. We defer these to \Cref{app:j_proof}. More difficult and interesting is the 2D array $J_{a,b}$. To solve the recurrence 
\begin{equation}
\label{eq:J_rec}
    J_{a,b} = \textcolor{red}{(a-1)} J_{a-2,b} + \textcolor{blue}{b\cos\theta} J_{a-1,b-1}, \quad a \geq 2, b \geq 1,
\end{equation} we will apply the recursion repeatedly until $J_{a,b}$ can be expressed as a linear combination of $J_{0,n}$ and $J_{1,n}$ terms for which we already have an explicit formula developed. To determine the coefficients in front of $J_{0,n}$ and $J_{1,n}$, we take a combinatorial approach by thinking of the recurrence relation as a weighted directed graph as defined below.

\begin{definition}[Viewing a recursion as a directed weighted graph]
\label{def:graph}
We can view the recurrence relation for $J_{a,b}$ as a weighted directed graph on the vertex set $(a,b)\in \mathbb{Z}^2$ where vertices represent the values of $J_{a,b}$ and directed edges capture how values of $J_{a,b}$ are connected through the recurrence relation. To be precise, the graph edges and edge weights $w_e$ are defined so that the recursion \eqref{eq:J_rec} for $J_{a,b}$ can be expressed in the graph as a sum over incoming edges,
\begin{equation}
\label{eq:graph}
J_{a,b} = \sum_{e: (a^\prime,b^\prime) \to (a,b)} w^J_e J_{a^\prime, b^\prime},
\end{equation}
where the sum is over the edges $e$ with weight $w^J_e$ incoming to the vertex $(a,b)$. An example of the graph to calculate $J_{6,8}$ is illustrated in Figure \ref{j_jstar}.

By repeatedly applying the recursion, $J_{a,b}$ can be expressed as a linear combination of the values at the \emph{source vertices} of the graph (i.e. those with no incoming edges). For the recurrence $J_{a,b}$, the source vertices are $J_{0,n}$ and $J_{1,n}$. The coefficients in front of each source is simply the weighed sum over all paths from the source to the node, namely
\begin{align}
\label{eq:graph_sources}
J_{a,b} = \sum_{\text{source vertices }v}W^J_{v \to (a,b)}J_v  &= \sum_{n\geq 0} W^J_{(0,n)\to(a,b)} J_{0,n} + \sum_{n\geq 0} W^J_{(1,n)\to(a,b)} J_{1,n}, \\
W^J_{(a^\prime,b^\prime)\to(a,b)} &:= \sum_{\pi:(a^\prime,b^\prime)\to(a,b)} \prod_{e \in \pi} w^J_e,\label{eq:W_def}
\end{align}
where the sum is over all paths $\pi$ from the vertex $(a^\prime,b^\prime)$ to the vertex $(a,b)$ in the $J$ graph. 

\end{definition}

In light of \eqref{eq:graph_sources}, to prove Theorem \ref{thm:j_ab}, we have only to calculate the weighted sum of path $W^J_{(0,n)\to(a,b)}$ and $W^J_{(1,n)\to(a,b)}$. These weighted sums turn out to be given in terms of the $P$ and $Q$ numbers which were defined in Definition \ref{defn:PQ}.

\begin{figure}[h]
\centering
\scalebox{0.7}{
\hfill
  \begin{subfigure}[b]{0.55\textwidth}
    \begin{tikzpicture}
    
    \foreach \i in {6,...,0}{
        \pgfmathsetmacro{\ii}{int(\i)}
        \draw (-0.5,\i+0.5) node{\ii};
    }
    
    \foreach \i in {0,...,6}{
        \pgfmathsetmacro{\ii}{int(8-\i)}
        \draw (\i + 0.5, 7.5) node{\ii};
    }
    
    \foreach \j in {0,...,3}{
        \pgfmathsetmacro{\z}{3-\j}
        
        \foreach \i in {0,...,\z}{
        
            \pgfmathsetmacro{\k}{2*\i}
            \pgfmathsetmacro{\l}{2*\j}
            \pgfmathsetmacro{\ll}{int(\l)}
            \pgfmathsetmacro{\kk}{int(8-\k)}
            \fill[yellow!5] (\k,\l) rectangle (\k+1,\l+1);
            \draw (\k+ 0.5, \l + 0.5) node{$J_{\ll, \kk}$};
            
        }
    }
    
    \foreach \j in {0,...,2}{
        \pgfmathsetmacro{\z}{2-\j}
        
        \foreach \i in {0,...,\z}{
            \pgfmathsetmacro{\k}{2*\i+1}
            \pgfmathsetmacro{\l}{2*\j+1}
            \pgfmathsetmacro{\ll}{int(\l)}
            \pgfmathsetmacro{\kk}{int(8-\k)}
            \fill[yellow!5] (\k,\l) rectangle (\k+1,\l+1);
            \draw (\k+ 0.5, \l + 0.5) node{$J_{\ll, \kk}$};
        }
    }
    
    \draw[step = 1cm, black, thin] (0,0) grid (7,7);

    \foreach \j in {0,...,2}{
        \pgfmathsetmacro{\z}{2-\j}
        
        \foreach \i in {0,...,\z}{
            \pgfmathsetmacro{\k}{2*\i}
            \pgfmathsetmacro{\l}{2*\j}
            
            \pgfmathsetmacro{\n}{int(\l)}
            \pgfmathsetmacro{\m}{int(8-\k)}
            
            \draw[red,thick,->] (\k + 0.5, \l + 1)--(\k + 0.5, \l + 2);
            
            \pgfmathsetmacro{\r}{int(\l+1)}
            \draw (\k + 0.65, \l + 1.5) node[red, scale=0.75]{\r};
        }
    }
    
    \foreach \j in {0,...,1}{
        \pgfmathsetmacro{\z}{1-\j}
        
        \foreach \i in {0,...,\z}{
            \pgfmathsetmacro{\k}{2*\i+1}
            \pgfmathsetmacro{\l}{2*\j+1}
            
            \pgfmathsetmacro{\n}{int(\l)}
            \pgfmathsetmacro{\m}{int(8-\k)}
            
            \draw[red,thick,->] (\k + 0.5, \l + 1)--(\k + 0.5, \l + 2);
            
            \pgfmathsetmacro{\r}{int(\l+1)}
            \draw (\k + 0.65, \l + 1.5) node[red, scale=0.75]{\r};
        }
    }
    
    \foreach \j in {0,...,1}{
        \pgfmathsetmacro{\z}{1-\j}
        
        \foreach \i in {0,...,\z}{
            \pgfmathsetmacro{\k}{2*\i+2}
            \pgfmathsetmacro{\l}{2*\j+2}
            
            \pgfmathsetmacro{\n}{int(\l)}
            \pgfmathsetmacro{\m}{int(8-\k)}
            
            \draw[blue,thick,->] (\k + 0.25, \l + 0.75)--(\k - 0.25, \l + 1.25);
            
            \pgfmathsetmacro{\kk}{int(9-\k)}
            \draw (\k + 0.5, \l + 1.15) node[blue, scale=0.75]{$\kk \cos \theta$};
        }
    }
    
    \foreach \j in {0,...,2}{
        \pgfmathsetmacro{\z}{2-\j}
        
        \foreach \i in {0,...,\z}{
            \pgfmathsetmacro{\k}{2*\i+1}
            \pgfmathsetmacro{\l}{2*\j+1}
            
            \pgfmathsetmacro{\n}{int(\l)}
            \pgfmathsetmacro{\m}{int(8-\k)}
            
            \draw[blue,thick,->] (\k + 0.25, \l + 0.75)--(\k - 0.25, \l + 1.25);
            
            \pgfmathsetmacro{\kk}{int(9-\k)}
            \draw (\k + 0.5, \l + 1.15) node[blue, scale=0.75]{$\kk \cos \theta$};
        }
    }
    \end{tikzpicture}
    \caption{Directed graph associated with $J$}
    \label{jtrue}
  \end{subfigure}
  \hfill
  \begin{subfigure}[b]{0.55\textwidth}
    \begin{tikzpicture}

    \foreach \i in {0,1,...,6}{
        \draw (-0.5,\i+0.5) node{\i};
    }
    
    \foreach \i\j in {0/8,1/7,2/6,3/5,4/4,5/3,6/2}{
        \draw (\i + 0.5, 7.5) node{\j};
    }
    
    \foreach \j in {0,...,3}{
        \pgfmathsetmacro{\z}{3-\j}
        
        \foreach \i in {0,...,\z}{
        
            \pgfmathsetmacro{\k}{2*\i}
            \pgfmathsetmacro{\l}{2*\j}
            \fill[yellow!5] (\k,\l) rectangle (\k+1,\l+1);
            
            \pgfmathsetmacro{\n}{int(\l)}
            \pgfmathsetmacro{\m}{int(8-\k)}
            \draw (\k + 0.5, \l + 0.5) node{$J_{\n,\m}^*$};
            
        }
    }
    
    \foreach \j in {0,...,2}{
        \pgfmathsetmacro{\z}{2-\j}
        
        \foreach \i in {0,...,\z}{
            \pgfmathsetmacro{\k}{2*\i+1}
            \pgfmathsetmacro{\l}{2*\j+1}
            \fill[yellow!5] (\k,\l) rectangle (\k+1,\l+1);
            
            \pgfmathsetmacro{\n}{int(\l)}
            \pgfmathsetmacro{\m}{int(8-\k)}
            \draw (\k + 0.5, \l + 0.5) node{$J_{\n,\m}^*$};
        }
    }

    \draw[step = 1cm, black, thin] (0,0) grid (7,7);

    \foreach \j in {0,...,2}{
        \pgfmathsetmacro{\z}{2-\j}
        
        \foreach \i in {0,...,\z}{
            \pgfmathsetmacro{\k}{2*\i}
            \pgfmathsetmacro{\l}{2*\j}
            
            \pgfmathsetmacro{\n}{int(\l)}
            \pgfmathsetmacro{\m}{int(8-\k)}
            
            \draw[red,thick,->] (\k + 0.5, \l + 1)--(\k + 0.5, \l + 2);
            
            \pgfmathsetmacro{\r}{int(\l+1)}
            \draw (\k + 0.65, \l + 1.5) node[red, scale=0.75]{\r};
        }
    }
    
    \foreach \j in {0,...,1}{
        \pgfmathsetmacro{\z}{1-\j}
        
        \foreach \i in {0,...,\z}{
            \pgfmathsetmacro{\k}{2*\i+1}
            \pgfmathsetmacro{\l}{2*\j+1}
            
            \pgfmathsetmacro{\n}{int(\l)}
            \pgfmathsetmacro{\m}{int(8-\k)}
            
            \draw[red,thick,->] (\k + 0.5, \l + 1)--(\k + 0.5, \l + 2);
            
            \pgfmathsetmacro{\r}{int(\l+1)}
            \draw (\k + 0.65, \l + 1.5) node[red, scale=0.75]{\r};
        }
    }
    
    \foreach \j in {1,...,3}{
        \pgfmathsetmacro{\z}{3-\j}
        
        \foreach \i in {0,...,\z}{
            \pgfmathsetmacro{\k}{2*\i+2}
            \pgfmathsetmacro{\l}{2*\j-2}
            
            \pgfmathsetmacro{\n}{int(\l)}
            \pgfmathsetmacro{\m}{int(8-\k)}
            
            \draw[blue,thick,->] (\k + 0.25, \l + 0.75)--(\k - 0.25, \l + 1.25);
            
            \draw (\k + 0.15, \l + 1.15) node[blue, scale=0.75]{1};
        }
    }
    
    \foreach \j in {0,...,2}{
        \pgfmathsetmacro{\z}{2-\j}
        
        \foreach \i in {0,...,\z}{
            \pgfmathsetmacro{\k}{2*\i+1}
            \pgfmathsetmacro{\l}{2*\j+1}
            
            \pgfmathsetmacro{\n}{int(\l)}
            \pgfmathsetmacro{\m}{int(8-\k)}
            
            \draw[blue,thick,->] (\k + 0.25, \l + 0.75)--(\k - 0.25, \l + 1.25);
            
            \draw (\k + 0.15, \l + 1.15) node[blue, scale=0.75]{1};
        }
    }
    \end{tikzpicture}
    \caption{Directed graph associated with $J^*$}
    \label{jstar}
  \end{subfigure}
  }
  \caption{The graph associated with the recursions for $J$ in \eqref{eq:J_rec} (left) and $J^\ast$ in \eqref{newrecursion} (right). The graph is defined so that the recursion is given by a sum of incoming edges as in \eqref{eq:graph}. The edges are color coded red and blue to match the coefficients in the recursion.}
  \label{j_jstar}
\end{figure}
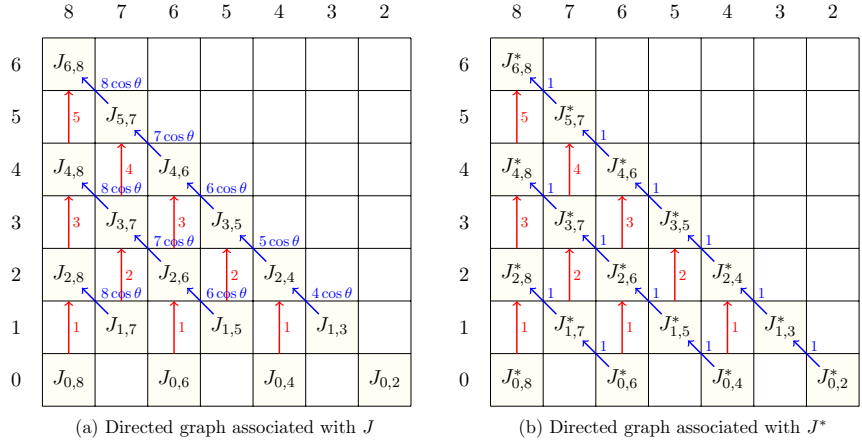

\begin{prop}[Weighted sums of paths for $J$]
\label{prop:J_weights}
In the graph for $J$, we have the following formulas for the sum over weighted paths $W^J$ defined in \eqref{eq:W_def},
\begin{align}
W^J_{(0,n)\to(a,b)} &= (b)_{b-n} (\cos\theta)^{b-n} (P(a,b-n) - Q(a-1,b-n-1)),  \\
W^J_{(1,n)\to(a,b)} &= (b)_{b-n} (\cos\theta)^{b-n} Q(a-1,b-n).
\end{align}
\end{prop}
To prove this Proposition \ref{prop:J_weights}, we first create a simpler recursion, $J^\ast$, which we solve first and then slightly modify the solution to get the solution for $J$.

\begin{lemma}[Weighted sums of paths for $J^\ast$]
\label{lem:J_star}
Let $J^\ast_{a,b}$ be defined to be the recursion:
\begin{align}
    {J^*_{a,b}} &:= \textcolor{red}{(a-1)}J^*_{a-2, b} + \textcolor{blue}{1}J^*_{a-1,b-1}\text{  for }2\leq a\leq b, \\ 
    J^*_{1,b} &:= \textcolor{red}{0} + \textcolor{blue}{1}J^*_{0,b-1}\text{ for }1\leq b.
    \label{newrecursion}
\end{align}
Thinking of this recursion as a graph as in Definition \ref{def:graph} (see Figure \ref{j_jstar} for an illustration), we have that the sum of weighted paths $W^{J^\ast}$ defined analogously to those in \eqref{eq:W_def}, are given by $P$ and $Q$ numbers, namely
\begin{align}
W^{J^\ast}_{(0,n)\to(a,b)} =  P(a,b-n),   \quad
W^{J^\ast}_{(1,n)\to(a,b)} =  Q(a-1,b-n). \label{eq:jstar_PQ_paths}
\end{align}
\end{lemma}

The connection between the $J^\ast$ and the $P$, $Q$ numbers is through the following recursion for the $P$, $Q$ numbers.


\begin{lemma}(Recursion for $P$ and $Q$ numbers)
\label{lem:PQ_rec}
The $P$ numbers, defined in Definition \ref{defn:PQ}, satisfy $P(0,0) = 1$, $P(n,n) = P(n-1,n-1)$ for $n \geq 1$, and the recursion
\begin{align*}
    P(a,b) = \textcolor{red}{(a-1)}\cdot P(a-2,b) + \textcolor{blue}{1}\cdot P(a-1,b-1), \;\; \textrm{for} \;\; a \geq 2,\; 0\leq b \leq a-2,
\end{align*}
under the convention that $P(a, -1)=0$. The $Q$ numbers satisfy the same recursion as the $P$ numbers, with a coefficient of $\textcolor{red}{a}$ rather than $\textcolor{red}{(a-1)}$.
\end{lemma}

The proof of Lemma \ref{lem:PQ_rec} is an easy consequence of known results from \cite{pq} and is deferred to \Cref{app:P_Q}.

\begin{proof}[Of Lemma \ref{lem:J_star}] Using the same idea of recursions expressed as graphs as in Definition \ref{def:graph}, the recursion from Lemma \ref{lem:PQ_rec} means that $P$ and $Q$ can be expressed as weighted directed graphs. These are displayed in Figure \ref{pqj}. Since the $P$ and $Q$ graphs have only one single unit valued source vertex at $(0,0)$,  \eqref{eq:graph_sources} shows that the $P$ and $Q$ numbers are actually themselves equal to sums over weighted paths in their respective graphs $$P(a,b) = W^P_{(0,0)\to(a,b)}, \quad Q(a,b) = W^Q_{(0,0)\to(a,b)}. $$
 \begin{figure}[h]
  \begin{subfigure}[b]{0.31\textwidth}
    \resizebox{5cm}{5cm}{
        \begin{tikzpicture}
    
    \foreach \i in {6,...,0}{
        \pgfmathsetmacro{\ii}{int(6-\i)}
        \draw (-0.5,\i+0.5) node{\ii};
    }
    
    \foreach \i in {0,...,6}{
        \draw (\i + 0.5, 7.5) node{\i};
    }
    
    \foreach \j in {0,...,3}{
        \pgfmathsetmacro{\z}{3-\j}
        
        \foreach \i in {0,...,\z}{
        
            \pgfmathsetmacro{\k}{2*\i}
            \pgfmathsetmacro{\l}{2*\j}
            \fill[yellow!5] (\k,\l) rectangle (\k+1,\l+1);
            
        }
    }
    
    \foreach \j in {0,...,2}{
        \pgfmathsetmacro{\z}{2-\j}
        
        \foreach \i in {0,...,\z}{
            \pgfmathsetmacro{\k}{2*\i+1}
            \pgfmathsetmacro{\l}{2*\j+1}
            \fill[yellow!5] (\k,\l) rectangle (\k+1,\l+1);
        }
    }
    
    \foreach \j in {0,...,2}{
        \pgfmathsetmacro{\z}{2-\j}
        
        \foreach \i in {0,...,\z}{
            
            \pgfmathsetmacro{\k}{2*\i}
            \pgfmathsetmacro{\l}{2*\j+1}
            
            \pgfmathsetmacro{\ll}{int(6-\l)}
            \pgfmathsetmacro{\kk}{int(\k)}
            \draw (\k + 0.5, \l + 0.5) node[black!10]{0};
            
        }
    }
    
    \foreach \j in {0,...,2}{
        \pgfmathsetmacro{\z}{2-\j}
        
        \foreach \i in {0,...,\z}{
            \pgfmathsetmacro{\k}{2*\i+1}
            \pgfmathsetmacro{\l}{2*\j}
            \pgfmathsetmacro{\ll}{int(6-\l)}
            \pgfmathsetmacro{\kk}{int(\k)}
            \draw (\k + 0.5, \l + 0.5) node[black!10]{0};
        }
    }
    
    \foreach \i in {1,...,6}{
        \foreach \j in {1,...,\i}{
            \pgfmathsetmacro{\jj}{7-\j}
            \draw (\i+0.5, \jj+0.5) node[black!10]{0};
        }
    }
    
    \draw[step = 1cm, black, thin] (0,0) grid (7,7);

    \foreach \j in {0,...,2}{
        \pgfmathsetmacro{\z}{2-\j}
        
        \foreach \i in {0,...,\z}{
            \pgfmathsetmacro{\k}{2*\i}
            \pgfmathsetmacro{\l}{2*\j}
            
            \pgfmathsetmacro{\n}{int(\l)}
            \pgfmathsetmacro{\m}{int(8-\k)}
            
            \draw[red,thick,->] (\k + 0.5, \l + 2) -- (\k + 0.5, \l + 1);
            
            \pgfmathsetmacro{\r}{int(5-\l)}
            \draw (\k + 0.65, \l + 1.5) node[red, scale=0.75]{\r};
        }
    }
    
    \foreach \j in {0,...,2}{
        \pgfmathsetmacro{\z}{2-\j}
        
        \foreach \i in {0,...,\z}{
            \pgfmathsetmacro{\k}{2*\i+1}
            \pgfmathsetmacro{\l}{2*\j-1}
            
            \pgfmathsetmacro{\n}{int(\l)}
            \pgfmathsetmacro{\m}{int(8-\k)}
            
            \draw[red,thick,->] (\k + 0.5, \l + 2) -- (\k + 0.5, \l + 1);
            
            \pgfmathsetmacro{\r}{int(5-\l)}
            \draw (\k + 0.65, \l + 1.5) node[red, scale=0.75]{\r};
        }
    }
    
    \foreach \j in {0,...,2}{
        \pgfmathsetmacro{\z}{2-\j}
        
        \foreach \i in {0,...,\z}{
            \pgfmathsetmacro{\k}{2*\i+2}
            \pgfmathsetmacro{\l}{2*\j}
            
            \pgfmathsetmacro{\n}{int(\l)}
            \pgfmathsetmacro{\m}{int(8-\k)}
            
            \draw[blue,thick,->] (\k - 0.25, \l + 1.25) -- (\k + 0.25, \l + 0.75);
            
            \draw (\k + 0.15, \l + 1.15) node[blue, scale=0.75]{1};
        }
    }
    
    \foreach \j in {0,...,2}{
        \pgfmathsetmacro{\z}{2-\j}
        
        \foreach \i in {0,...,\z}{
            \pgfmathsetmacro{\k}{2*\i+1}
            \pgfmathsetmacro{\l}{2*\j+1}
            
            \pgfmathsetmacro{\n}{int(\l)}
            \pgfmathsetmacro{\m}{int(8-\k)}
            
            \draw[blue,thick,->] (\k - 0.25, \l + 1.25) -- (\k + 0.25, \l + 0.75);
            
            \draw (\k + 0.15, \l + 1.15) node[blue, scale=0.75]{1};
        }
    }

    \draw (0.5, 6.5) node{1};
    \draw (0.5, 4.5) node{1};
    \draw (0.5, 2.5) node{3};
    \draw (0.5, 0.5) node{15};
    
    \draw (1.5, 5.5) node{1};
    \draw (1.5, 3.5) node{3};
    \draw (1.5, 1.5) node{15};
    
    \draw (2.5, 4.5) node{1};
    \draw (2.5, 2.5) node{6};
    \draw (2.5, 0.5) node{45};
    
    \draw (3.5, 3.5) node{1};
    \draw (3.5, 1.5) node{10};
    
    \draw (4.5, 2.5) node{1};
    \draw (4.5, 0.5) node{15};
    
    \draw (5.5, 1.5) node{1};
    
    \draw (6.5, 0.5) node{1};
    
    \end{tikzpicture}
    }
    \caption{$P$}
    \label{p}
  \end{subfigure}
  \hfill
  \begin{subfigure}[b]{0.31\textwidth}
    \resizebox{5cm}{5cm}{
        \begin{tikzpicture}
    
    \foreach \i in {6,...,0}{
        \pgfmathsetmacro{\ii}{int(6-\i)}
        \draw (-0.5,\i+0.5) node{\ii};
    }
    
    \foreach \i in {0,...,6}{
        \draw (\i + 0.5, 7.5) node{\i};
    }
    
    \foreach \j in {0,...,3}{
        \pgfmathsetmacro{\z}{3-\j}
        
        \foreach \i in {0,...,\z}{
        
            \pgfmathsetmacro{\k}{2*\i}
            \pgfmathsetmacro{\l}{2*\j}
            \fill[yellow!5] (\k,\l) rectangle (\k+1,\l+1);
            
        }
    }
    
    \foreach \j in {0,...,2}{
        \pgfmathsetmacro{\z}{2-\j}
        
        \foreach \i in {0,...,\z}{
            \pgfmathsetmacro{\k}{2*\i+1}
            \pgfmathsetmacro{\l}{2*\j+1}
            \fill[yellow!5] (\k,\l) rectangle (\k+1,\l+1);
        }
    }
    
    \foreach \j in {0,...,2}{
        \pgfmathsetmacro{\z}{2-\j}
        
        \foreach \i in {0,...,\z}{
            
            \pgfmathsetmacro{\k}{2*\i}
            \pgfmathsetmacro{\l}{2*\j+1}
            
            \pgfmathsetmacro{\ll}{int(6-\l)}
            \pgfmathsetmacro{\kk}{int(\k)}
            \draw (\k + 0.5, \l + 0.5) node[black!10]{0};
            
        }
    }
    
    \foreach \j in {0,...,2}{
        \pgfmathsetmacro{\z}{2-\j}
        
        \foreach \i in {0,...,\z}{
            \pgfmathsetmacro{\k}{2*\i+1}
            \pgfmathsetmacro{\l}{2*\j}
            \pgfmathsetmacro{\ll}{int(6-\l)}
            \pgfmathsetmacro{\kk}{int(\k)}
            \draw (\k + 0.5, \l + 0.5) node[black!10]{0};
        }
    }
    
    \foreach \i in {1,...,6}{
        \foreach \j in {1,...,\i}{
            \pgfmathsetmacro{\jj}{7-\j}
            \draw (\i+0.5, \jj+0.5) node[black!10]{0};
        }
    }
    
    \draw[step = 1cm, black, thin] (0,0) grid (7,7);

    \foreach \j in {0,...,2}{
        \pgfmathsetmacro{\z}{2-\j}
        
        \foreach \i in {0,...,\z}{
            \pgfmathsetmacro{\k}{2*\i}
            \pgfmathsetmacro{\l}{2*\j}
            
            \pgfmathsetmacro{\n}{int(\l)}
            \pgfmathsetmacro{\m}{int(8-\k)}
            
            \draw[red,thick,->] (\k + 0.5, \l + 2) -- (\k + 0.5, \l + 1);
            
            \pgfmathsetmacro{\r}{int(6-\l)}
            \draw (\k + 0.65, \l + 1.5) node[red, scale=0.75]{\r};
        }
    }

    \foreach \j in {0,...,2}{
        \pgfmathsetmacro{\z}{2-\j}
        
        \foreach \i in {0,...,\z}{
            \pgfmathsetmacro{\k}{2*\i+1}
            \pgfmathsetmacro{\l}{2*\j-1}
            
            \pgfmathsetmacro{\n}{int(\l)}
            \pgfmathsetmacro{\m}{int(8-\k)}
            
            \draw[red,thick,->] (\k + 0.5, \l + 2) -- (\k + 0.5, \l + 1);
            
            \pgfmathsetmacro{\r}{int(6-\l)}
            \draw (\k + 0.65, \l + 1.5) node[red, scale=0.75]{\r};
        }
    }
    
    \foreach \j in {0,...,2}{
        \pgfmathsetmacro{\z}{2-\j}
        
        \foreach \i in {0,...,\z}{
            \pgfmathsetmacro{\k}{2*\i+2}
            \pgfmathsetmacro{\l}{2*\j}
            
            \pgfmathsetmacro{\n}{int(\l)}
            \pgfmathsetmacro{\m}{int(8-\k)}
            
            \draw[blue,thick,->] (\k - 0.25, \l + 1.25) -- (\k + 0.25, \l + 0.75);
            
            \draw (\k + 0.15, \l + 1.15) node[blue, scale=0.75]{1};
        }
    }
    
    \foreach \j in {0,...,2}{
        \pgfmathsetmacro{\z}{2-\j}
        
        \foreach \i in {0,...,\z}{
            \pgfmathsetmacro{\k}{2*\i+1}
            \pgfmathsetmacro{\l}{2*\j+1}
            
            \pgfmathsetmacro{\n}{int(\l)}
            \pgfmathsetmacro{\m}{int(8-\k)}
            
            \draw[blue,thick,->] (\k - 0.25, \l + 1.25) -- (\k + 0.25, \l + 0.75);
            
            \draw (\k + 0.15, \l + 1.15) node[blue, scale=0.75]{1};
        }
    }

    \draw (0.5, 6.5) node{1};
    \draw (0.5, 4.5) node{2};
    \draw (0.5, 2.5) node{8};
    \draw (0.5, 0.5) node{$\vdots$}; 
    
    \draw (1.5, 5.5) node{1};
    \draw (1.5, 3.5) node{5};
    \draw (1.5, 1.5) node{33};
    
    \draw (2.5, 4.5) node{1};
    \draw (2.5, 2.5) node{9};
    \draw (2.5, 0.5) node{$\vdots$}; 
    
    \draw (3.5, 3.5) node{1};
    \draw (3.5, 1.5) node{14};
    
    \draw (4.5, 2.5) node{1};
    \draw (4.5, 0.5) node{$\vdots$}; 
    
    \draw (5.5, 1.5) node{1};
    
    \draw (6.5, 0.5) node{$\vdots$}; 
    
    \end{tikzpicture}
    }
    \caption{$Q$}
    \label{q}
  \end{subfigure}
  \hfill
  \begin{subfigure}[b]{0.31\textwidth}
    \resizebox{5cm}{5cm}{
        \begin{tikzpicture}
    
    \foreach \i in {0,1,...,6}{
        \draw (-0.5,\i+0.5) node{\i};
    }
    
    \foreach \i\j in {0/8,1/7,2/6,3/5,4/4,5/3,6/2}{
        \draw (\i + 0.5, 7.5) node{\j};
    }
    
    \foreach \j in {0,...,3}{
        \pgfmathsetmacro{\z}{3-\j}
        
        \foreach \i in {0,...,\z}{
        
            \pgfmathsetmacro{\k}{2*\i}
            \pgfmathsetmacro{\l}{2*\j}
            \fill[yellow!5] (\k,\l) rectangle (\k+1,\l+1);
            
            \pgfmathsetmacro{\n}{int(\l)}
            \pgfmathsetmacro{\m}{int(8-\k)}
            \draw (\k + 0.5, \l + 0.5) node{$J_{\n,\m}^*$};
            
        }
    }
    
    \foreach \j in {0,...,2}{
        \pgfmathsetmacro{\z}{2-\j}
        
        \foreach \i in {0,...,\z}{
            \pgfmathsetmacro{\k}{2*\i+1}
            \pgfmathsetmacro{\l}{2*\j+1}
            \fill[yellow!5] (\k,\l) rectangle (\k+1,\l+1);
            
            \pgfmathsetmacro{\n}{int(\l)}
            \pgfmathsetmacro{\m}{int(8-\k)}
            \draw (\k + 0.5, \l + 0.5) node{$J_{\n,\m}^*$};
        }
    }

    \draw[step = 1cm, black, thin] (0,0) grid (7,7);

    \foreach \j in {0,...,2}{
        \pgfmathsetmacro{\z}{2-\j}
        
        \foreach \i in {0,...,\z}{
            \pgfmathsetmacro{\k}{2*\i}
            \pgfmathsetmacro{\l}{2*\j}
            
            \pgfmathsetmacro{\n}{int(\l)}
            \pgfmathsetmacro{\m}{int(8-\k)}
            
            \draw[red,thick,->] (\k + 0.5, \l + 1) -- (\k + 0.5, \l + 2);
            
            \pgfmathsetmacro{\r}{int(\l+1)}
            \draw (\k + 0.65, \l + 1.5) node[red, scale=0.75]{\r};
        }
    }
    
    \foreach \j in {1,...,2}{
        \pgfmathsetmacro{\z}{2-\j}
        
        \foreach \i in {0,...,\z}{
            \pgfmathsetmacro{\k}{2*\i+1}
            \pgfmathsetmacro{\l}{2*\j-1}
            
            \pgfmathsetmacro{\n}{int(\l)}
            \pgfmathsetmacro{\m}{int(8-\k)}
            
            \draw[red,thick,->] (\k + 0.5, \l + 1) -- (\k + 0.5, \l + 2);
            
            \pgfmathsetmacro{\r}{int(\l+1)}
            \draw (\k + 0.65, \l + 1.5) node[red, scale=0.75]{\r};
        }
    }
    
    \foreach \j in {1,...,3}{
        \pgfmathsetmacro{\z}{3-\j}
        
        \foreach \i in {0,...,\z}{
            \pgfmathsetmacro{\k}{2*\i+2}
            \pgfmathsetmacro{\l}{2*\j-2}
            
            \pgfmathsetmacro{\n}{int(\l)}
            \pgfmathsetmacro{\m}{int(8-\k)}
            
            \draw[blue,thick,->] (\k + 0.25, \l + 0.75) -- (\k - 0.25, \l + 1.25);
            
            \draw (\k + 0.15, \l + 1.15) node[blue, scale=0.75]{1};
        }
    }
    
    \foreach \j in {0,...,2}{
        \pgfmathsetmacro{\z}{2-\j}
        
        \foreach \i in {0,...,\z}{
            \pgfmathsetmacro{\k}{2*\i+1}
            \pgfmathsetmacro{\l}{2*\j+1}
            
            \pgfmathsetmacro{\n}{int(\l)}
            \pgfmathsetmacro{\m}{int(8-\k)}
            
            \draw[blue,thick,->] (\k + 0.25, \l + 0.75) -- (\k - 0.25, \l + 1.25);
            
            \draw (\k + 0.15, \l + 1.15) node[blue, scale=0.75]{1};
        }
    }
    \end{tikzpicture}
    }
    \caption{$J^*_{6,8}$}
    \label{j}
  \end{subfigure}
  \caption{Graphs associated with the recursions for the $P$ numbers (left), $Q$ numbers (middle), and $J^*$ (right). The weighted edges indicate the coefficients in the recursions for $P,Q,J^\ast$ respectively. The diagrams are lined up so that the sum of weighted paths in from $J^*_{6,8}$ can be directly read from the $P$ and $Q$ entries in the same location. By reading from the bottom displayed row of $P$ we see that the weighted sum over paths $W^{J^\ast}_{(0,n)\to(6,8)}$ are $15$, $45$, $15$ and $1$ for $n=8,6,4$ and $2$ respectively. (Since these are the source vertices, this shows that $J^*_{6,8}=15J^*_{0,8}+45J^*_{0,6}+15J^*_{0,4}+1 J^*_{0,2}$.) From the bottom displayed row of $Q$ we see that the values for sums of weighted paths from vertices $W^{J^\ast}_{(1,n)\to(6,8)}$ are $33$, $14$ and $1$ for $n=7,5$ and $3$ respectively. }
  \label{pqj}
\end{figure}
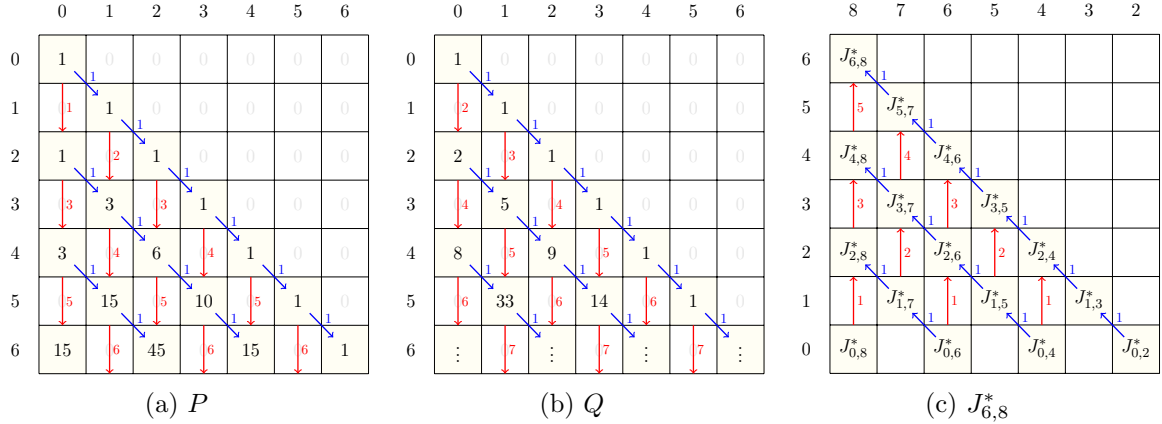
Therefore the statement of the lemma is that sum over weighted paths in the $J^\ast$ graph are the same as other sums over weighted paths in the $P$ graph/$Q$ graphs,
\begin{equation}
W^{J^\ast}_{(0,n)\to (a,b)} = W^P_{(0,0) \to (a,b-n)}, \quad W^{J^\ast}_{(1,n)\to(a,b)} = W^Q_{(0,0)\to (a-1,b-n)}. \label{eq:path_equality}
\end{equation}

The fact that these are equal is demonstrated by establishing a simple bijection between weighted paths in the $P$ graph/$Q$ graph, and weighted paths in the $J^\ast$ graph. For example, in Figure \ref{pqj}, there is a bijection between the weighted paths in the $P$ graph which connect $P(0,0)$ to $P(6,2)$, to the paths which connect $J^*_{6,8}$ to $J^*_{0,6}$ in the $J^\ast$ graph. The bijection is simply to \emph{flip} any path in the $P$-graph by rotating it by $180^\circ$ to get a valid path in the $J^\ast$-graph. Moreover, the edge weights for $J^\ast$ and $P$ are precisely set up so that under this bijection, the paths will have the same set of weighted edges in the same order. A full, more detailed, explanation of this bijection is given in \Cref{app:bijection}. This argument shows that $W^{J^\ast}_{(0,n) \to (a,b)}=P(a,b-n)$ as desired.


%


%
The $P$ numbers do not apply for paths between $J^*_{1,n}$ and $J^*_{a,b}$ because we are starting one row higher so the first vertical upward edge is weight $2$. In this case, there is a bijection to the $Q$-graph after flipping the path. For any path which runs from a node in row 1 to the top left corner of the $J^*$-graph, we can find the same ``flipped" path in the graph of the $Q$, running from the top left entry to the corresponding node in row $a-1$. (The bijection is explained in detail in \Cref{app:bijection}.) Hence $W^{J^\ast}_{(1,n) \to (a,b)}=Q(a-1,b-n)$ as desired.
\end{proof}



Having solved for $J^\ast$ in terms of $P$ and $Q$, it remains to translate these into the weights for $J$ to obtain Proposition \ref{prop:J_weights}.

\begin{proof}[Of Proposition \ref{prop:J_weights}]
    
\noindent The proof follows by relating the weighted sum of paths for $J$ in terms of $J^\ast$ and then applying the result of Lemma \ref{lem:J_star}. There are two differences between the formula for $J_{a,b}$ compared to $J^\ast_{a,b}$, which can both be seen in Figure \ref{j_jstar}. We handle both differences as follows:
 
\noindent \textbf{\emph{Difference \#1:}} $J$ has a weight of $b\cos\theta$ on the blue diagonal edges $(a,b) \to (a+1,b+1)$ vs $J^\ast$ has a weight of 1.

This difference is handled by the following observation: any path from $(a,b) \to (a^\prime,b^\prime)$ in the graph goes through each column between $b$ and $b^\prime$ exactly once. This means that the contribution of the edge weights from these edges do not depend on the details of which path was taken, only the starting and ending points. They always contribute the same factor, $(b)_{b^\prime - b} (\cos\theta)^{b^\prime - b}$. (here $(b)_k=b(b-1)\cdots(b-k+1)$ is the falling factorial with $k$ terms). This argument shows that the weighted sum of paths in $J$ and $J^\ast$ are related by

\begin{equation}
    \label{eq:WJ_vs_WJast}
    W^J_{(a,b)\to(a^\prime,b^\prime)} = (b)_{b^\prime - b} (\cos \theta)^{b^\prime - b} W^{J^\ast}_{(a,b)\to(a^\prime,b^\prime)}.
\end{equation}
By the result of Lemma \ref{lem:J_star}, this shows that $ W^J_{(1,n)\to(a,b)} = (b)_{b-n} (\cos\theta)^{b-n} Q(a-1,b-n)$ as desired. Equation \eqref{eq:WJ_vs_WJast} holds for all paths with starting point $a\geq 1$. When $a=0$, there is one additional difference between $J$ and $J^\ast$ which is accounted for below.

\noindent \textbf{\emph{Difference \#2:}} $J_{0,n}$ has \emph{no} diagonal edge vs $J^\ast_{0,n}$ has a diagonal blue edge of weight 1.

Because of this ``missing edge'', the only choice in $J$ for paths starting from $(0,n)$ is to first go vertically up by 2 units to $(2,n)$. Hence $W^{J}_{(0,n)\to(a,b)}=W^{J}_{(2,n)\to(a,b)}$. To evaluate this, we use the decomposition of paths in $J^\ast$ by what their first step is, either a diagonal blue step or a red vertical up step, to see that
\begin{align}
  W^{J^\ast}_{(0,n)\to(a,b)} &= W^{J^\ast}_{(1,n+1)\to(a,b)}+W^{J^\ast}_{(2,n)\to(a,b)},\\
  \implies W^{J^\ast}_{(2,n)\to(a,b)} &= W^{J^\ast}_{(0,n)\to(a,b)} - W^{J^\ast}_{(1,n+1)\to(a,b)} \\
  &= P(a,b-n) - Q(a-1,b-n-1),
\end{align}
by the result of Lemma \ref{lem:J_star}. By applying now \eqref{eq:WJ_vs_WJast} to relate $J$ and $J^\ast$, we obtain $ W^J_{(0,n)\to(a,b)} = W^{J}_{(2,n)\to(a,b)} =  (b)_{b-n} (\cos\theta)^{b-n} (P(a,b-n) - Q(a-1,b-n-1))$ as desired. 
\end{proof}

\begin{proof}[Of Theorem \ref{thm:j_ab}]
The formula is immediate from \eqref{eq:graph_sources}, which writes $J_{a,b}$ as a linear combination of $J_{0,n}$ and $J_{1,n}$, and Proposition \ref{prop:J_weights} which gives the the coefficients.
\end{proof}






\appendix
\section*{Appendix A.}
\fakesection{fakesection}
\subsection{Expected Value Approximation}
\label{app:expected_value}

\begin{lemma}
Both the random variables $X = R^{\ell+1}$ and $X= R^{\ell+1}\sin^2(\theta_\ell)$ satisfy 
\begin{align}
    \mathbf{E}[\ln(X)] = \ln(\mathbf{E}[X]) - \frac{\var[X]}{2\mathbf{E}[X]^2} + \mathcal{O}(n_\ell^{-2}). \label{eq:exp_value_approx_general}
\end{align}
\end{lemma}
\begin{proof} First note that by the properties of the logarithm, we have
\begin{equation}
    \ln(X) = \ln\left(\mathbf{E}[X]\left(  \frac{\mathbf{E}[X] +(X-\mathbf{E}[X])}{\mathbf{E}[X]}\right)\right) 
     = \ln(\mathbf{E}[X]) + \ln\left(1 + \frac{X-\mathbf{E}[X]}{\mathbf{E}[X]} \right). \label{eq:lnXequals}
\end{equation}

\noindent We can now apply the Taylor series $\ln(1+x) = x- \frac{x^2}{2} + \epsilon_2(x)$, where $\epsilon_2(x)$ is the Taylor series remainder and satisfies $\epsilon_2(x)=\mathcal{O}(x^3)$. Hence
\begin{align*}
    \ln(X) &= \ln(\mathbf{E}[X]) + \frac{X-\mathbf{E}[X]}{\mathbf{E}[X]} -\frac{(X-\mathbf{E}[X])^2}{2\mathbf{E}[X]^2} + \epsilon_2\left(\frac{X-\mathbf{E}[X]}{\mathbf{E}[X]} \right).
\end{align*}

\noindent Note that $\mathbf{E}[X-\mathbf{E}[X]]=0$, and $\mathbf{E}[(X-\mathbf{E}[X])^2]=\var[X]$. Thus, if we take the expected value of our above approximation, we get the following:
\begin{align*}
    \mathbf{E}[\ln(X)] = \ln(\mathbf{E}[X]) - \frac{\var[X]}{2\mathbf{E}[X]^2} + \mathbf{E}\left[\epsilon_2\left(\frac{X-\mathbf{E}[X]}{\mathbf{E}[X]} \right)\right].
\end{align*}
By using bounds on the Taylor series error term $\epsilon_2(x)=\mathcal{O}(x^3)$, one can obtain bounds for this last error term. By (\ref{sums_4}, \ref{eq:norm_ratio_identity}), both $X=\normratioo$ and $X=\normratioo \sin^2(\theta_{\ell+1})$ can be expressed as averages of the form
\begin{gather}
    X = \frac{1}{n_\ell^2} \sum_{i,j}^{n_\ell} f(G_i, \hat{G}_j). \label{eq:A_sum}
\end{gather}
From the bound on the 3rd moment in Lemma \ref{lem:third_moment}, it follows that $\mathbf{E}[\epsilon_2 (X- \mathbf{E}[X])] = \mathcal{O}(n_\ell^{-2})$, thus giving the desired result.\end{proof}

\subsection{Variance Approximation}
\label{app:variance}

\begin{lemma}
Both the random variables $X = R^{\ell+1}$ and $X= R^{\ell+1}\sin^2(\theta_\ell)$ satisfy  
\begin{equation*}
    \var[\ln(X)] = \frac{\var[X]}{\mathbf{E}[X]^2} + \mathcal{O}(n_\ell^{-2}).
\end{equation*}
\end{lemma}
\begin{proof} Starting with \eqref{eq:lnXequals}, and using the first term of the Taylor series approximation for $\ln(1+x)=x+\epsilon_1(x)$ now, we have that
\begin{align}
    \ln(X) = \ln(\mathbf{E}[X]) + \frac{X-\mathbf{E}[X]}{\mathbf{E}[X]} + \epsilon_1\left(\frac{X-\mathbf{E}[X]}{\mathbf{E}[X]} \right).
    \label{lnapprox}
\end{align}
where $\epsilon_1(x)$ is the Taylor error term and satisfies $\epsilon_1(x)=\mathcal{O}(x^2)$. Taking the variance of this, we arrive at an approximation of $\var[\ln(X)]$. 
\begin{gather*}
    \var[\ln(X)] = \var\left[\ln(\mathbf{E}[X]) + \frac{X-\mathbf{E}[X]}{\mathbf{E}[X]} + \epsilon_1\left(\frac{X-\mathbf{E}[X]}{\mathbf{E}[X]} \right)\right] \\
    = \var\left[\frac{X-\mathbf{E}[X]}{\mathbf{E}[X]} \right] + \var\left[ \epsilon_1\left(\frac{X-\mathbf{E}[X]}{\mathbf{E}[X]} \right)\right] + \cov\left(\frac{X-\mathbf{E}[X]}{\mathbf{E}[X]} ,\epsilon_1\left(\frac{X-\mathbf{E}[X]}{\mathbf{E}[X]} \right) \right).
\end{gather*}
As with the expected value approximation, this approximation for variance is used twice, once for $X = \normratioo$, and once for $X = \normratioo \sin^2(\theta_{\ell+1})$ (see \Cref{sec:variance}), both of which can can be expressed as a sum as in (\ref{eq:A_sum}). Since $\epsilon_1(x) = \mathcal{O}(x^2)$, we have that the terms with $\epsilon_1(x)$ are both $\mathcal{O}(n_\ell^{-2})$ from Lemma \ref{lem:third_moment}. Simplifying the first term, $\var\left[\frac{X-\mathbf{E}[X]}{\mathbf{E}[X]} \right]=\frac{\var[X]}{\mathbf{E}[X]^2}$ gives the result of the Lemma.
\end{proof}

\subsection{Covariance Approximation}
\label{app:covariance_approx}
\begin{lemma}
Both the random variables $X = R^{\ell+1}$ and $X= R^{\ell+1}\sin^2(\theta_\ell)$ satisfy  
\begin{equation*}
    \cov(\ln(X), \ln(Y)) = \frac{\cov(X,Y)}{\mathbf{E}[X]\mathbf{E}[Y]} + \mathcal{O}(n_\ell^{-2}).
\end{equation*}
\end{lemma}
\begin{proof}
Using the approximation in (\ref{lnapprox}) for $\ln(X)$ and $\ln(Y)$, we get the following expression for the covariance:
\begin{align*}
    &\cov(\ln(X),\ln(Y))\\
    &= \cov\left( \ln(\mathbf{E}[X]) + \frac{X-\mathbf{E}[X]}{\mathbf{E}[X]} + \epsilon_1\left(\frac{X-\mathbf{E}[X]}{\mathbf{E}[X]} \right),\ln(\mathbf{E}[Y]) + \frac{Y-\mathbf{E}[Y]}{\mathbf{E}[Y]} + \epsilon_1\left(\frac{Y-\mathbf{E}[Y]}{\mathbf{E}[Y]} \right) \right) \\
    &= \cov\left( \frac{X}{\mathbf{E}[X]} + \epsilon_1\left(\frac{X-\mathbf{E}[X]}{\mathbf{E}[X]} \right), \frac{Y}{\mathbf{E}[Y]} + \epsilon_1\left(\frac{Y-\mathbf{E}[Y]}{\mathbf{E}[Y]} \right) \right) \\
    &= \cov\left( \frac{X}{\mathbf{E}[X]}, \frac{Y}{\mathbf{E}[Y]} \right) + \cov \left(\frac{X}{\mathbf{E}[X]}, \epsilon_1\left(\frac{Y-\mathbf{E}[Y]}{\mathbf{E}[Y]} \right) \right) + \cov \left( \epsilon_1\left(\frac{X-\mathbf{E}[X]}{\mathbf{E}[X]} \right), \frac{Y}{\mathbf{E}[Y]} \right)\\
    &\;\;\;\;\;+ \cov \left(\epsilon_1\left(\frac{X-\mathbf{E}[X]}{\mathbf{E}[X]} \right),\epsilon_1\left(\frac{Y-\mathbf{E}[Y]}{\mathbf{E}[Y]} \right) \right).
\end{align*}
We get the desired result from the fact that that the error term $\epsilon_1(x)$ satisfies $\epsilon_1(x) = \mathcal{O}(x^2)$ and from our result in Lemma \ref{lem:bigo_covariance}.
\end{proof}


\subsection{Third and Fourth Moment Bound Lemma}
\label{app:third_moment_lemma}

\begin{lemma}
\label{lem:third_moment}
Let $G_i,\hat{G}_i$, $1\leq i \leq n$ be marginally $\mathcal{N}(0,1)$ random variables with correlation $\cos(\theta)$ and independent for different indices $i$. Let $A = \frac{1}{n^2} \sum_{i,j}^{n} f(G_i, \hat{G}_j)$ be the average over all $n^2$ pairs of some function $f:\mathbb{R}^2 \to \mathbb{R}$ which has finite fourth moment, $\mathbf{E}[f(G_i,\hat{G}_i)^4] < \infty$. Then, the third and fourth central moment of $A$ satisfy 
\begin{eqnarray}
\mathbf{E}[(A - \mathbf{E}[A])^3] = \mathcal{O}(n^{-2}), \quad  
\mathbf{E}[(A - \mathbf{E}[A])^4] = \mathcal{O}(n^{-2}).
\end{eqnarray}
\end{lemma}

\begin{proof}
We begin by showing the third moment bound. First, we can express $\mathbf{E}[(A - \mathbf{E}[A])^3]$ as a sum in the following way:
\begin{gather}
    A-\mathbf{E}[A] = \frac{1}{n^2} \sum_{i,j}^{n} \left( f(G_i, \hat{G}_j) - \mathbf{E}[f(G_i, \hat{G}_j)] \right) \nonumber \\
    \implies \mathbf{E}\left[(A-\mathbf{E}[A])^3 \right] = \frac{1}{n^6}  \sum_{\substack{i_1,i_2,i_3 \\j_1,j_2,j_3}}^{n} \mathbf{E}\left[ \prod_{k=1}^3 \left( f(G_{i_k}, \hat{G}_{j_k}) - \mathbf{E}[f(G_{i_k}, \hat{G}_{j_k})] \right) \right]. \label{eq:third_order_sum}
\end{gather}
Note that many of these terms are mean zero. For example, for any configuration of the indices where there is no overlap between the indices $(i_1,j_1)$ and the other two index pairs  ($\{i_1,j_1\}\cap\{i_2,j_2,i_3,j_3\}=\varnothing$), we may use independence to observe that

\begin{align*}
& \mathbf{E}\left[ \prod_{k=1}^3 \left( f(G_{i_k}, \hat{G}_{j_k}) - \mathbf{E}[f(G_{i_k}, \hat{G}_{j_k})] \right)\right] \\
=& \mathbf{E}\left[  f(G_{i_1}, \hat{G}_{j_1}) - \mathbf{E}[f(G_{i_1}, \hat{G}_{j_1})] \right] \mathbf{E}\left[ \prod_{k=2}^3 \left( f(G_{i_k}, \hat{G}_{j_k}) - \mathbf{E}[f(G_{i_k}, \hat{G}_{j_k})] \right) \right] = 0. 
\end{align*}
When this happens we say that $(i_1,j_1)$ is a ``reducible point''. Similarly, $(i_2,j_2)$ or $(i_3,j_3)$ can be reducible if they have no overlap with the other two index pairs. To control $\mathbf{E}\left[(A-\mathbf{E}[A])^3 \right]$, it will suffice to enumerate the number of indices $\{i_1,j_1,i_2,j_2,i_3,j_3\}$ so that all three points $(i_1,j_1),(i_2,j_2),(i_3,j_3)$ are \emph{not} reducible. We call these ``irreducible configurations''.

We now observe that at least one of the points $(i_1,j_1)$,$(i_2,j_2)$ or $(i_3,j_3)$ is reducible whenever the number of unique numbers is $\left|\bigcup_{k=1}^3 \{i_k, j_k\} \right| \geq 5$. This is because, by the pigeonhole principle, if there are no repeated or only one repeated number between 6 indices, then at least one of the 3 pairs $(i_1,j_1)$,$(i_2,j_2)$ or $(i_3,j_3)$ must consist of two unique numbers and therefore is a reducible point.

Since the irreducible configurations can only have at most 4 unique numbers, the number of irreducible configurations is $\mathcal{O}(n^4)$ as $n\to \infty$. In fact, a detailed enumeration of the number of configurations reveals that the number of irreducible configurations is precisely
\begin{gather}
    32(n)_4 + 68(n)_3 + 28(n)_2 + 1(n)_1.
\end{gather}
Here, $(n)_k=n\cdot(n-1)\cdot(n-2)\cdots(n-k+1)$ denotes the falling factorial with $k$ terms. The leading term is $32$ because there are $32$ possible ``patterns''  for how the indices can be arranged to be both irreducible and contain exactly 4 unique numbers $\left|\bigcup_{k=1}^3 \{i_k, j_k\} \right| = 4$; these patterns are listed in Table \ref{tab:third_moment}. Each pattern contributes $(n)_4=n(n-1)(n-2)(n-3)$ possible index configurations by filling in the 4 unique numbers in all the possible ways. Similarly, there are respectively 68, 28, and 1 pattern(s) for irreducible configurations with 3,2 and 1 unique number(s) in them which each contribute $(n)_3, (n)_2$ and $(n)_1$ configurations per pattern 

Since the number of irreducible configurations is $\mathcal{O}(n^4)$, the normalization by $n^6$ in \eqref{eq:third_order_sum} shows that $\mathbf{E}[(A - \mathbf{E}[A])^3]$ is $\mathcal{O}(n^{-2})$ as desired for the third moment.

The argument for the $4$th moment is similar. We write $\mathbf{E}[(A - \mathbf{E}[A])^4]$ as a sum over $i_1,j_1,i_2,j_2,i_3,j_3,i_4,j_4$ and again enumerate irreducible configurations. In this case, once again by the pigeonhole principle any configuration with 7 or more unique points $\left|\bigcup_{k=1}^3 \{i_k, j_k\} \right| \geq 7$ will be reducible. Since there are at most $6$ unique numbers, there will be $\mathcal{O}(n^6)$ irreducible configurations. A detailed enumeration of all the possible irreducible patterns and the number of unique elements in each yields that the number of irreducible configurations is precisely
\begin{equation*}
48(n)_6 + 544(n)_5 + 1268(n)_4 + 844(n)_3 + 123(n)_2 + 1(n)_1.
\end{equation*}
The normalization factor of $n^{-8}$ then shows that $\mathbf{E}[(A - \mathbf{E}[A])^4] = \mathcal{O}(n^{-2})$.
\end{proof}
\begin{remark}
A more detailed enumeration of the $4$th moment actually shows that the dominant terms in the $4$th moment correspond to the terms in the $2$nd moment written twice, and asymptotically $$\mathbf{E}[(A - \mathbf{E}[A])^4] = 3 \mathbf{E}[(A - \mathbf{E}[A])^2]^2 + \mathcal{O}(n^{-3}).$$
Here, $3$ arises as the number of pair partitions of 4 items, and is related to the fact that $3 = \mathbf{E}[G^4]$.
\end{remark}

\begin{lemma}
\label{lem:bigo_covariance}
Let $G_i,\hat{G}_i$, $1\leq i \leq n$ be marginally $\mathcal{N}(0,1)$ random variables with correlation $\cos(\theta)$ and independent for different indices $i$. Let $A_1 = \frac{1}{n^2} \sum_{i,j}^{n} f_1(G_i, \hat{G}_j)$, and let $A_2 = \frac{1}{n^2} \sum_{i,j}^{n} f_2(G_i, \hat{G}_j)$, where $f_1, f_2: \mathbb{R}^2 \to \mathbb{R}$ have finite fourth moments, $\mathbf{E}[f_1(G_i,\hat{G}_i)^4]$, $\mathbf{E}[f_2(G_i,\hat{G}_i)^4] < \infty$. Then,
$$\mathbf{E}[(A_1-\mathbf{E}[A_1])^2(A_2-\mathbf{E}[A_2])] = \mathcal{O}\left(n^{-2}\right).$$
    
\end{lemma}

\begin{proof}
We can express $\mathbf{E}[(A_1-\mathbf{E}[A_2])^2(A_2-\mathbf{E}[A_2])]$ using sums as follows:
\begin{align*}
    &\mathbf{E}[(A_1-\mathbf{E}[A_1])^2(A_2-\mathbf{E}[A_2])] \nonumber \\
    &= \frac{1}{n^6}  \sum_{\substack{i_1,i_2,i_3 \\j_1,j_2,j_3}}^{n} \mathbf{E}\left[ \prod_{k=1}^2 \left( f_1(G_{i_k}, \hat{G}_{j_k}) - \mathbf{E}[f_1(G_{i_k}, \hat{G}_{j_k})] \right)\left( f_2(G_{i_3}, \hat{G}_{j_3}) - \mathbf{E}[f_2(G_{i_3}, \hat{G}_{j_3})] \right) \right].
\end{align*}
By the same argument as in Lemma \ref{lem:third_moment}, we can show that the number of nonzero terms in the above summation is $\mathcal{O}(n^4)$ as $n \to \infty$. Thus, we have that $\mathbf{E}[(A_1-\mathbf{E}[A_1])^2(A_2-\mathbf{E}[A_2])] = \mathcal{O}(n^{-2})$. We can also show that $\mathbf{E}[(A_1-\mathbf{E}[A_1])^2(A_2-\mathbf{E}[A_2])^2] = \mathcal{O}(n^{-2})$ by the same argument.  
\end{proof}

%

\begin{table}[h!]
    \centering
    \begin{tabular}{c c c c c c}
         $(i_1,j_1)$ &  & $(i_2,j_2)$ &  &$(i_3,j_3)$ &  \\ \cline{1-5}
         $\{(a,b), (b,a)\}$ & $\times$ & $\{(a,c), (c,a)\}$ & $\times$ & $\{(a,d), (d,a)\}$ & 8 patterns \\
         $\{(a,b), (b,a)\}$ & $\times$ & $\{(a,c), (c,a)\}$ & $\times$ & $\{(c,d), (d,c)\}$ & 8 patterns \\
         $\{(a,b), (b,a)\}$ & $\times$ & $\{(c,d), (d,c)\}$ & $\times$ & $\{(a,c), (c,a)\}$ & 8 patterns \\
         $\{(a,c), (c,a)\}$ & $\times$ & $\{(a,b), (b,a)\}$ & $\times$ & $\{(c,d), (c,b)\}$ & 8 patterns \\
    \end{tabular}
    \caption{All 32 irreducible patterns using exactly 4 unique index values $a,b,c,d$. For example the pattern $(i_1,j_1),(i_2,j_2),(i_3,j_3)=(a,b),(a,c),(a,d)$ represents all configurations where $i_1=i_2=i_3$ and the $j$'s are all unique and different from $i$. For each pattern, there are $(n)_4 = n(n-1)(n-2)(n-3)$ configurations by filling in $a,b,c,d$ with unique numbers in $[n]$. These are the dominant terms in  (\ref{eq:third_order_sum}).}
    \label{tab:third_moment}
\end{table}

\subsection{Derivation of Useful Identities - Equations (\ref{sums_1}, \ref{sums_2})}
\label{app:identities}

\noindent Let $G \in \mathbb{R}^n$ be a Gaussian vector with iid entries $G_i \sim \mathcal{N}(0,1)$. Then, by standard properties of Gaussians,  the function $f:\mathbb{R}^n \to \mathbb{R}$ given by $f(x) = \langle G,x\rangle$ is a Gaussian random variable. Further, $f(x) \sim \mathcal{N}(0, \|x\|^2)$ for all $x \in \mathbb{R}^n$, and for any two vectors $x_\alpha, x_\beta \in \mathbb{R}^n$, the joint distribution of $f(x_\alpha), f(x_\beta)$ is jointly Gaussian with $$\begin{bmatrix}f(x_\alpha) \\ f(x_\beta)
    \end{bmatrix} \sim \mathcal{N}\left(0, \Sigma(x_\alpha,x_\beta) \right), \qquad \Sigma(x_\alpha,x_\beta) := \begin{bmatrix}\|x_\alpha\|^2 & \langle x_\alpha, x_\beta \rangle \\ \langle x_\alpha, x_\beta \rangle & \|x_\beta\|^2 \end{bmatrix}, $$ 
where $\Sigma(x_\alpha,x_\beta)$ is sometimes called the $2 \times 2$ Gram matrix of the vectors $x_\alpha,x_\beta$.
%
%
In the setting of our fully connected neural network, any index $i \in [n_{\ell+1}]$ in the vector of  $z^{\ell+1}$ is actually the inner product with the $i$-th row
\begin{gather*}
    z^{\ell+1}_i(x) = \sqrt{\frac{2}{n_\ell}}\langle W^{\ell+1}_{i,\cdot}, \ph(z^\ell(x)) \rangle.
\end{gather*}

\noindent Note that each row $W^{\ell+1}_{i,\cdot}$ is a Gaussian vector, so the previous fact about Gaussians applies and we see that the entries of $z^{\ell+1}$ are conditionally Gaussian given the value of the previous layer.  By the previous Gaussian fact, we have that $z_i^{\ell+1}(x_\alpha)$, $z_i^{\ell+1}(x_\beta)$ are jointly Gaussian with 
\begin{align*}
    \begin{bmatrix}z_i^{\ell+1}(x_\alpha) \\ z_i^{\ell+1}(x_\beta)
    \end{bmatrix} \sim \mathcal{N}\left(0,\frac{2}{n_\ell} \begin{bmatrix}\|\pha\|^2 & \langle \pha, \phb \rangle \\ \langle \pha, \phb \rangle & \|\phb\|^2 \end{bmatrix} \right) =: \mathcal{N}\left(0,K^\ell \right),
\end{align*}
where we use $K^\ell$ to denote the $2\times 2$ covariance matrix. $K^\ell$ is precisely the $2 \times 2$ Gram matrix of the previous layer $\pha$, $\phb$ scaled by $2/n_\ell$ and its entries $K^\ell_{\gamma \delta}$, for $ \gamma \in \{\alpha,\beta\}, \delta \in \{\alpha,\beta\}$ are actually \emph{averages} of entries in the previous layer
\begin{equation*}
    K^\ell_{\gamma, \delta} := \frac{2}{n_\ell}\langle \ph_\gamma^\ell, \ph_\delta^\ell\rangle = \frac{1}{n_\ell}\sum_{k=1}^{n_\ell} 2 \ph(z^\ell_k(x_i))\ph(z^\ell_k(x_j)).
\end{equation*}

 Moreover, in the weight matrix $W^{\ell+1}$, the $i^{th}$ and $j^{th}$ rows ($W^{\ell+1}_{i,\cdot}$ and $W^{\ell+1}_{j,\cdot}$, respectively) are independent. Therefore, all entries of $z^{\ell+1}$ are identically distributed and conditionally independent given $\ph(z^\ell)$. From this fact, we can equivalently write the entries explicitly as

\begin{equation}
    z_i^{\ell+1}(x_\alpha) = \sqrt{\frac{2}{n_\ell}}\|\pha\|G_i , \quad
    z_i^{\ell+1}(x_\beta) = \sqrt{\frac{2}{n_\ell}} \|\phb\| \hat{G}_i, \label{eq:z_x_alpha}
\end{equation}

\noindent where $G_i, \hat{G}_i$ are marginally $\mathcal{N}(0,1)$ variables with covariance $\cov(G_i,\hat{G}_i)=\cos(\theta_\ell)$ and independent for different indices. This formulation precisely ensures that the covariance structure for the entries is exactly what is specified by the covariance kernel $K^\ell$.

With this representation of $z_i^{\ell+1}(x_\alpha)$ and $z_i^{\ell+1}(x_\beta)$, we can apply $\varphi(\cdot)$ to each entry. By using the property of ReLU $\varphi(\lambda x)=\lambda \varphi(x)$ for $\lambda > 0$ to factor out the norms, we obtain

\begin{equation}
    \varphi(z_i^{\ell+1}(x_\alpha)) = \sqrt{\frac{2}{n_\ell}}\|\pha\| \varphi(G_i), \quad
    \varphi(z_i^{\ell+1}(x_\beta)) = \sqrt{\frac{2}{n_\ell}} \|\phb\| \varphi(\hat{G}_i).  \label{eq:phi_x_beta}
\end{equation}
Taking the norm/inner product of the vector now yields (\ref{sums_1}-\ref{sums_4}) as desired.

%

\subsection{Cauchy-Binet and Determinant of the Gram Matrix - Equation (\ref{sums_4})}
\label{app:cauchy_binet}

\noindent To prove this identity, we begin with the fact that \begin{gather*}
    \|\phaa\|^2 \|\phbb\|^2 \sin^2(\theta_{\ell+1}) = \det \begin{bmatrix}
    \|\phaa\|^2 & \langle \phaa, \phbb \rangle \\
    \langle \phaa, \phbb \rangle & \|\phbb\|^2
    \end{bmatrix}.
\end{gather*}

\noindent By the Cauchy-Binet identity, we can express the determinant as 
\begin{gather}
    \det \begin{bmatrix}
    \|\phaa\|^2 & \langle \phaa, \phbb \rangle \\
    \langle \phaa, \phbb \rangle & \|\phbb\|^2
    \end{bmatrix}
    = \sum_{1\leq i<j\leq n_\ell}\left(\ph_{i;\alpha}^{\ell+1}\ph_{j;\beta}^{\ell+1} - \ph_{j;\alpha}^{\ell+1}\ph_{i;\beta}^{\ell+1} \right)^2.
    \label{c-b}
\end{gather}

\noindent Due to the fact that the summand is equal to 0 when $i=j$, we can equivalently take the sum over all indices $i,j \in [n_\ell]$ and halve the result. We can also express layer $\ell+1$ using the following conditioning on the previous layer
\begin{equation*}
    \ph_{i;\alpha}^{\ell+1} = \sqrt{\frac{2}{n_\ell}} \|\pha\| \cdot \ph(G_i), \quad
    \ph_{i;\beta}^{\ell+1} = \sqrt{\frac{2}{n_\ell}} \|\phb\| \cdot \ph(\hat{G}_i).
\end{equation*}

\noindent Applying these facts to our expression in (\ref{c-b}), and dividing both sides by $\|\pha\|^2\|\phb\|^2$, we get our desired result.



\subsection{Expected Value Calculations}
\label{app:exp_value_calculation}
\noindent In this section, we derive the formulas for $\mathbf{E}\left[\normratioo\right]$, $\mathbf{E}\left[ \normratioo \sin^2(\theta_{\ell+1})\right]$. We use $J_{a,b}$ to represent $J_{a,b}(\theta_\ell)$. Note that $\mathbf{E}[\ph^2(G)] = \frac{1}{2}$, $\mathbf{E}[\ph^4(G)] = \frac{3}{2}$.

\subsubsection*{Calculation of $\mathbf{E}\left[\normratioo\right]:$}
First, we apply the identity as in \eqref{eq:norm_ratio_identity}:
\begin{align*}
    \mathbf{E}\left[\normratioo\right] = \left(\frac{2}{n_\ell}\right)^2 \mathbf{E}\left[ \sum_{i,j=1}^{n_\ell} \ph^2(G_i) \ph^2(\hat{G}_j)\right] .
\end{align*}
Whenever $i=j$, taking the expected value will give us a $J_{2,2}$ term. When $i \neq j$, the expected value of this term will be $\mathbf{E}[\ph^2(G)]^2 = \frac{1}{4}$. Since $i=j$ happens $n_\ell$ times, and therefore $i\neq j$ happens $n_\ell^2 - n_\ell$ times, we arrive at the following expression:
\begin{align*}
    \mathbf{E}\left[\normratioo\right] 
    = \left(\frac{2}{n_\ell}\right)^2 \left(n_\ell J_{2,2} + (n_\ell^2-n_\ell)\left( \frac{1}{4}\right) \right) = \frac{4J_{2,2}-1}{n_\ell}+1.
\end{align*}

\subsubsection*{Calculation of $\mathbf{E}\left[\normratioo \sin^2(\theta_{\ell+1})\right]:$}
Applying the identity \eqref{sums_4}, we get
\begin{align*}
    \mathbf{E}\left[\normratioo \sin^2(\theta_{\ell+1})\right] 
    &= \frac{2}{n_\ell^2}\mathbf{E}\left[ \sum_{i,j}^{n_\ell} \left(\ph(G_i)\ph(\hat{G}_j) - \ph(G_j)\ph(\hat{G}_i) \right)^2\right] \\
     &= \frac{2}{n_\ell^2}\mathbf{E}\left[ \sum_{i,j}^{n_\ell} \left(\ph^2(G_i)\ph^2(\hat{G}_j) - 2\ph(G_i)\ph(\hat{G}_i)\ph(G_j)\ph(\hat{G}_j)+\ph^2(G_j)\ph^2(\hat{G}_i) \right)\right].
\end{align*}

\noindent In the case where $i=j$, the expected value is equal to 0. Thus, we only need to consider the case where $i\neq j$, which happens $n_\ell^2-n_\ell$ times. When $i\neq j$, the expectation of $\ph(G_i)\ph(\hat{G}_i)\ph(G_j)\ph(\hat{G}_j)$ is $J_{1,1}^2$, and the expectation of $\ph^2(G_i)\ph^2(\hat{G}_j)$ is $\frac{1}{4}$. All together, we have
\begin{equation*}
    \mathbf{E}\left[\normratioo \sin^2(\theta_{\ell+1})\right] = \left(\frac{2}{n_\ell^2}\right) (n_\ell^2-n_\ell)\left( \frac{1}{4}-2J_{1,1}^2+\frac{1}{4}\right) = \frac{(n_\ell-1)(1-4J_{1,1}^2)}{n_\ell}.
\end{equation*}

\subsection{Variance and Covariance Calculations}
\label{app:formula_calculation}

\noindent In this section, $\var \left[\normratioo\right]$,  $\var\left[ \normratioo \sin^2(\theta_{\ell+1})\right]$, and $\cov\left(\normratioo \sin^2(\theta_{\ell+1}), \normratioo \right)$ are evaluated. We use $J_{a,b}$ to represent $J_{a,b}(\theta_\ell)$. Note that $\mathbf{E}[\ph^2(G)] = \frac{1}{2}$, $\mathbf{E}[\ph^4(G)] = \frac{3}{2}$. We will see that there are simple functions $f_1, f_2:\mathbb{R}^2 \to \mathbb{R}$ so that all of the variance and covariance calculations can be expressed as sums over $i_1,j_1,i_2,j_2$ of the form
\begin{align}
\label{eq:var_general_form}
    \frac{1}{n_\ell^4} \sum_{\substack{i_1,j_1 \\ i_2,j_2}}\left(  \mathbf{E}\left[f_1(G_{i_1}, \hat{G}_{j_1})f_2(G_{i_2}, \hat{G}_{j_2}) \right] - \mathbf{E}\left[ f_1(G_{i_1}, \hat{G}_{j_1})\right]\mathbf{E}\left[ f_2(G_{i_2}, \hat{G}_{j_2})\right] \right),
\end{align}
where the sum goes over index configurations $(i_1,j_1),(i_2,j_2)\in [n_\ell]^4$. We will use this form to organize our calculations of the variance and covariance formulas. The strategy is to evaluate each term in the sum \eqref{eq:var_general_form} individually.

Since the random variables $\{G_i,\hat{G}_i\}_{i=1}^n$ are exchangeable, the only thing that matters is the ``pattern'' of which of the indices $i_1,j_1,i_2,j_2$ are repeated versus which are distinct. For example, there will be $n$ index configurations where $i_1=j_1=i_2=j_2$ are all equal. All $n$ of these give same contribution. There are  $(n)_4=n(n-1)(n-2)(n-3)$ configurations where $i_1,j_1,i_2,j_2$ are all distinct. Knowing which indices are repeated/distinct allows us to evaluate the corresponding term in \eqref{eq:var_general_form}. We use the following formal notion of a pattern to organize this idea of repeated versus distinct indices.
\begin{definition}
\label{def:patterns}
   
    A \textbf{pattern} for $(i_1,j_1),(i_2,j_2)$ is a subset of all possible index configurations $(i_1,j_1),(i_2,j_2)\in[n]^4$ represented by an assignment of each index to the letters $a,b,c,d$. Each letter $a,b,c,d$ represents a choice of \emph{unique} indices from $[n]$.
\end{definition}

        For example, the pattern $(i_1,j_1),(i_2,j_2)=(a,a),(a,a)$ represents the set of all index configurations where all indices are equal and the pattern $(i_1,j_1),(i_2,j_2)=(a,b),(c,d)$ represents the set with all indices unique. The pattern $(i_1,j_1),(i_2,j_2)=(a,b),(a,c)$ represents all configurations where $i_1=i_2$ and $j_1,j_2$ are unique and different from $i_1=i_2$. For this pattern, there are $(n)_3=n(n-1)(n-2)$ configurations by filling in $a,b,c$ with unique numbers in $[n]$. More generally, for a pattern with $k$ letters, there are $(n)_k$ configurations that fall into that pattern.

        Fortunately, when enumerating \eqref{eq:var_general_form}, many patterns have \emph{no} contribution and can be ignored. We formalize this in the following definition.
\begin{definition}
\label{def:reducible}
    We say that the configuration of indices $(i_1,j_1),(i_2,j_2)$ is \textbf{reducible} if $\{i_1,j_1\}\cap\{i_2,j_2\}=\varnothing$. Otherwise, the index configuration is called \textbf{irreducible}. A pattern is called reducible if all index configuration in that pattern are reducible. 
\end{definition}

By the independence of the random variables $f_1(G_{i_1},G_{j_1})$ and $f_2(G_{i_2},G_{j_2})$, whenever $(i_1,j_1),(i_2,j_2)$ is reducible, we see that the corresponding term in \eqref{eq:var_general_form} completely vanishes! Therefore, to evaluate \eqref{eq:var_general_form}, we have only to understand the contribution of \emph{irreducible configurations}. The irreducible configurations can be organized into \emph{irreducible patterns}. For example, the pattern $(a,b),(c,c)$ is reducible (since formally $\{a,b\}\cap\{c\}=\varnothing$) and so any configuration from this pattern has \emph{no} contribution in the expectation. 

There are 11 irreducible patterns. (All these patterns are listed as part of Table \ref{tab:var_R}.) The expected value of the terms for each pattern will give a contribution that is expressed in terms of the $J_{a,b}$ depending on the details of exactly which indices are repeated. Then by enumerating the number of configurations in each pattern, we can evaluate \eqref{eq:var_general_form}.  This strategy is precisely how we evaluate each variance/covariance in this section.

\subsubsection*{Calculation of $\var\left[\normratioo\right]:$}
First, applying the identity in \eqref{eq:norm_ratio_identity}, we get
\begin{align*}
    \var\left[\normratioo\right] &= \left(\frac{2}{n_\ell}\right)^4 \var\left[ \sum_{i,j=1}^{n_\ell} \ph^2(G_i) \ph^2(\hat{G}_j)\right]\\
    &= \frac{16}{n_\ell^4} \left( \mathbf{E}\left[\sum_{\substack{i_1,j_1 \\i_2,j_2}} \ph^2(G_{i_1}) \ph^2(\hat{G}_{j_1}) \ph^2(G_{i_2}) \ph^2(\hat{G}_{j_2}) \right]
    - \mathbf{E}\left[ \sum_{i,j=1}^{n_\ell} \ph^2(G_i) \ph^2(\hat{G}_j)\right]^2 \right).
\end{align*}
 $\var[\normratioo]$ follows the form of \eqref{eq:var_general_form}, with $f_1(G_{i}, \hat{G}_{i}) =  f_2(G_{i}, \hat{G}_{i}) = \ph^2(G_{i}) \ph^2(\hat{G}_{j})$. We then evaluate the contribution from each irreducible pattern in  Table \ref{tab:var_R}.
 Combining all these cases and simplifying based on powers of $\frac{1}{n_\ell}$, we arrive at the following expression for $\var\left[\normratioo\right]$:
 \begin{align*}
     \frac{4}{n_\ell}(J_{2,2} +1)+  \frac{16}{n_\ell^2}\left(2J_{4,2} - \frac{5}{2}J_{2,2} + J_{2,2}^2 + \frac{5}{8} \right)+ \frac{16}{n_\ell^3}\left(J_{4,4} - 2J_{4,2} -2J_{2,2}^2 + 2J_{2,2} -\frac{9}{8} \right).
\end{align*}

\begin{table}[ht]
    \centering
    \begin{tabular}{|c|c|c|c|c|c|}
        \multicolumn{6}{c}{\large $\var[\normratioo]$ Calculation} \\ 
        
        \hline \# & $(i_1,j_1)$ & $(i_2, j_2)$ & $\mathbf{E}[f_1(G_{i_1}, \hat{G}_{j_1})]$ & $\mathbf{E}[f_2(G_{i_2}, \hat{G}_{j_2})]$ & $\mathbf{E}[f_1(G_{i_1}, \hat{G}_{j_1})f_2(G_{i_2}, \hat{G}_{j_2})]$  \\ \hline\hline
         
         $(n)_1$ & $(a,a)$ & $(a,a)$ & $J_{2,2}$ & $J_{2,2}$ & $J_{4,4}$ \\ \hline
         
         \multirow{6}{*}{$(n)_2$} & $(a,b)$ & $(a,b)$ & \multirow{2}{*}{$\left( \frac{1}{2}\right)^2$} & \multirow{2}{*}{$\left( \frac{1}{2}\right)^2$} & $\left( \frac{3}{2}\right)^2$ \\ \cline{2-3} \cline{6-6}
         
         & $(a,b)$ & $(b,a)$ &  &  & $J_{2,2}^2$ \\ \cline{2-6}
         
         & $(a,a)$ & $(a,b)$ & \multirow{2}{*}{$J_{2,2}$} & \multirow{2}{*}{$\left( \frac{1}{2}\right)^2$} & \multirow{4}{*}{$\frac{1}{2}J_{4,2}$} \\ \cline{2-3}
         
         & $(a,a)$ & $(b,a)$ &  &  & \\ \cline{2-5}
         
         & $(a,b)$ & $(a,a)$ & \multirow{2}{*}{$\left( \frac{1}{2}\right)^2$} & \multirow{2}{*}{$J_{2,2}$} &  \\ \cline{2-3}
         
         & $(b,a)$ & $(a,a)$ &  &  &  \\ \hline
         
        \multirow{4}{*}{$(n)_3$} & $(a,b)$ & $(a,c)$ & \multirow{4}{*}{$\left( \frac{1}{2}\right)^2$} & \multirow{4}{*}{$\left( \frac{1}{2}\right)^2$} & \multirow{2}{*}{$\frac{3}{2} \left(\frac{1}{2}\right)^2$} \\ \cline{2-3}
        
         & $(a,b)$ & $(c,b)$ &  &  & \\ \cline{2-3} \cline{6-6}
         
         & $(a,b)$ & $(c,a)$ &  &  & \multirow{2}{*}{$\left(\frac{1}{2}\right)^2J_{2,2}$} \\ \cline{2-3}
         
         & $(a,b)$ & $(b,c)$ &  &  &  \\ \hline
    \end{tabular}
    \caption{$\var[\normratioo]$ calculated in the form of \eqref{eq:var_general_form} with $f_1(G_{i}, \hat{G}_{j}) = f_2(G_{i}, \hat{G}_{j}) = \ph^2(G_{i}) \ph^2(\hat{G}_{j})$. The contribution from all 11 possible \emph{irreducible} patterns of the indices are shown. }
    \label{tab:var_R}
\end{table}


\subsubsection*{Calculation of $\var\left[\normratioo \sin^2(\theta_{\ell+1})\right]:$}
Applying identity \eqref{sums_4}, we can express $\var[\normratioo \sin^2(\theta_{\ell+1})]$ as
\begin{align*}
    &\var\left[\normratioo \sin^2(\theta_{\ell+1})\right] = \frac{1}{4}\left(\frac{2}{n_\ell} \right)^4\var\left[\sum_{i,j=1}^{n_\ell} \left(\ph(G_i)\ph(\hat{G}_j) - \ph(G_j)\ph(\hat{G}_i) \right)^2 \right]. 
\end{align*}
Note that we can express $\var[\normratioo \sin^2(\theta_{\ell+1})]$ as in \eqref{eq:var_general_form} by letting $f_1(G_{i}, \hat{G}_{j}) = f_2(G_{i}, \hat{G}_{j}) = ( \ph(G_{i})\ph(\hat{G}_{j})-\ph(G_{j})\ph(\hat{G}_{i}))^2$.We then evaluate the contribution from each irreducible pattern in  Table \ref{tab:var_R_sin}.
 Combining all these cases and simplifying based on powers of $\frac{1}{n_\ell}$, we arrive at the following expression:
%
%
\begin{align*}
    &\var\left[\normratioo \sin^2(\theta_{\ell+1})\right]= \frac{8}{n_\ell}\left( -8J_{1,1}^4 + 8J_{1,1}^2J_{2,2} + 4J_{1,1}^2 - 8J_{1,1}J_{3,1} + J_{2,2} + 1\right)\\
    &+ \frac{2}{n_\ell^2}\left(80J_{1,1}^4 - 96J_{1,1}^2J_{2,2} - 40J_{1,1}^2 + 96J_{1,1}J_{3,1} + 24J_{2,2}^2 - 12J_{2,2} - 32J_{3,1}^2 + 5 \right) \\
    &+ \frac{2}{n_\ell^3} \left( -48J_{1,1}^4 + 64J_{1,1}^2J_{2,2} + 24J_{1,1}^2 - 64J_{1,1}J_{3,1} - 24J_{2,2}^2 + 8J_{2,2} + 32J_{3,1}^2 -9\right).
\end{align*}

\begin{table}[ht]
    \centering
    \begin{tabular}{|c|c|c|c|c|c|}
        \multicolumn{6}{c}{\large $\var[\normratioo \sin^2(\theta_{\ell+1})]$ Calculation} \\  
    
        \hline \# & $(i_1,j_1)$ & $(i_2, j_2)$ & $\mathbf{E}[f(G_{i_1}, \hat{G}_{j_1})]$ & $\mathbf{E}[f(G_{i_2}, \hat{G}_{j_2})]$ & $\mathbf{E}[f(G_{i_1}, \hat{G}_{j_1})f(G_{i_2}, \hat{G}_{j_2})]$  \\ \hline\hline
         
         \multirow{2}{*}{$(n)_2$} & $(a,b)$ & $(a,b)$ & \multirow{6}{*}{$\frac{1}{2}-2J_{1,1}^2$} & \multirow{6}{*}{$\frac{1}{2}-2J_{1,1}^2$} & \multirow{2}{*}{$6J_{2,2}^2 - 8J_{3,1}^2 + \frac{9}{2}$} \\ \cline{2-3} 
         
         & $(a,b)$ & $(b,a)$ &  &  &  \\ \cline{1-3} \cline{6-6}
         
         \multirow{4}{*}{$(n)_3$} & $(a,b)$ & $(a,c)$ &  &  & \multirow{4}{*}{$4 J_{2,2} J_{1,1}^2 - 4 J_{3,1}J_{1,1} + \frac{1}{2}J_{2,2} + \frac{3}{4}$} \\ \cline{2-3}
         
         & $(a,b)$ & $(c,a)$ &  &  &  \\ \cline{2-3}
         
         & $(a,b)$ & $(c,b)$ &  &  &  \\ \cline{2-3}
         
         & $(a,b)$ & $(b,c)$ &  &  &  \\ \hline
    \end{tabular}
    \caption{$\var[\normratioo\sin^2(\theta_{\ell+1})]$ calculated in the form of \eqref{eq:var_general_form} with $f_1(G_{i}, \hat{G}_{j}) = f_2(G_{i}, \hat{G}_{j})$ $= ( \ph(G_{i})\ph(\hat{G}_{j})-\ph(G_{j})\ph(\hat{G}_{i}))^2$. The \emph{non-zero} contribution  \emph{irreducible} patterns of the indices are shown. Note  that because $f_1(G_{i_1},G_{j_1})=0$ when $i_1=j_1$ and $f_2(G_{i_2},G_{j_2})=0$ when $i_2=j_2$, there are 5  irreducible patterns (of the possible 11) that have zero contribution and are not displayed in this table.}
    \label{tab:var_R_sin}
\end{table}

\subsubsection*{Calculation of $\cov\left(\normratioo \sin^2(\theta_{\ell+1}),\normratioo \right)$:}

\begin{table}[ht]
    \centering
    \begin{tabular}{|c|c|c|c|c|c|}
        \multicolumn{6}{c}{\large $\cov(\normratioo, \normratioo \sin^2(\theta_{\ell+1}))$ Calculation} \\
        
        \hline \# & $(i_1,j_1)$ & $(i_2, j_2)$ & $\mathbf{E}[f_1(G_{i_1}, \hat{G}_{j_1})]$ & $\mathbf{E}[f_2(G_{i_2}, \hat{G}_{j_2})]$ & $\mathbf{E}[f_1(G_{i_1}, \hat{G}_{j_1})f_2(G_{i_2}, \hat{G}_{j_2})]$  \\ \hline\hline
         
         \multirow{4}{*}{$(n)_2$} & $(a,b)$ & $(b,b)$ & \multirow{8}{*}{$\frac{1}{2} - 2J_{1,1}^2$} & \multirow{2}{*}{$J_{2,2}$} & \multirow{2}{*}{$J_{4,2} - 2J_{1,1}J_{3,3}$}  \\ \cline{2-3}
         
          & $(a,b)$ & $(a,a)$ &  &  &  \\ \cline{2-3} \cline{5-6}
         
          & $(a,b)$ & $(a,b)$ &  & \multirow{6}{*}{$\left(\frac{1}{2}\right)^2$} & \multirow{6}{*}{$J_{2,2}^2 - 2J_{3,1}^2 + \left( \frac{3}{2} \right)^2$}  \\ \cline{2-3}
         
          & $(a,b)$ & $(b,a)$ &  &  &   \\ \cline{1-3}
         
         \multirow{4}{*}{$(n)_3$} & $(a,b)$ & $(a,c)$ &  &  &   \\ \cline{2-3}
         
          & $(a,b)$ & $(c,a)$ &  &  &   \\ \cline{2-3}
         
          & $(a,b)$ & $(b,c)$ &  &  &   \\ \cline{2-3}
         
          & $(a,b)$ & $(c,b)$ &  &  &   \\ \hline
    \end{tabular}
    \caption{$\cov\left(\normratioo \sin^2(\theta_{\ell+1}),\normratioo \right)$ calculated in the form of \eqref{eq:var_general_form} with $f_1(G_i, \hat{G}_j) = (\ph(G_i)\ph(\hat{G}_j) - \ph(G_j)\ph(\hat{G}_i))^2$, and $f_2(G_i, \hat{G}_j) = \ph^2(G_i) \ph^2(\hat{G}_j)$. The \emph{non-zero} contribution from  \emph{irreducible} patterns of the indices are shown. Note  that because $f_1(G_{i_1},G_{j_1})=0$ when $i_1=j_1$, there are 3  irreducible patterns (of the possible 11) that have zero contribution which are \emph{not} displayed in this table.}
    \label{tab:cov_R_sin}
\end{table}


\begin{gather*}
    \cov\left( \normratioo \sin^2(\theta_{\ell+1}),\normratioo \right) = \mathbf{E}\left[ \left(\normratioo \right)^2 \sin^2(\theta_{\ell+1}) \right] - \mathbf{E}\left[\normratioo \sin^2(\theta_{\ell+1}) \right] \mathbf{E}\left[ \normratioo \right].
\end{gather*}

\noindent Applying known identities (\ref{sums_4}, \ref{eq:norm_ratio_identity}) derived in \Cref{app:identities} and \Cref{app:cauchy_binet}, we can express this in the form of \eqref{eq:var_general_form}, where $f_1(G_i, \hat{G}_j) = (\ph(G_i)\ph(\hat{G}_j) - \ph(G_j)\ph(\hat{G}_i))^2$, and $f_2(G_i, \hat{G}_j) = \ph^2(G_i) \ph^2(\hat{G}_j)$. 
%
%
%
Table \ref{tab:cov_R_sin} shows the calculation of all the irreducible patterns. Collecting all cases and simplifying based on powers of $\frac{1}{n_\ell}$ gives:
\begin{gather*}
    \cov\left( \normratioo \sin^2(\theta_{\ell+1}),\normratioo \right) = \frac{1}{n_\ell}\left(16J_{1,1}^2 - 32J_{1,1}J_{3,1} + 8J_{2,2} +8 \right)\\
    + \frac{1}{n_\ell^2}\left(32J_{1,1}^2J_{2,2} -40J_{1,1}^2 + 96J_{1,1}J_{3,1} - 32J_{1,1}J_{3,3} + 16J_{2,2}^2 - 32J_{2,2} - 32J_{3,1}^2 + 16J_{4,2} + 10 \right) \\
    + \frac{1}{n_\ell^3}\left(24J_{1,1}^2 - 32J_{1,1}^2J_{2,2} - 64J_{1,1}J_{3,1} +32J_{1,1}J_{3,3} - 16J_{2,2}^2 + 24J_{2,2} + 32J_{3,1}^2 - 16J_{4,2} - 18 \right).
\end{gather*}


\subsection{Infinite Width Update Rule}
\label{app:infinite_width}
\begin{lemma}
Let $f(x)$ be a feed forward neural network as defined in \ref{tbl:notation}. Conditional on the value of $\theta_\ell$ in layer $\ell$, the angle $\theta_\ell$ between inputs at layer $\ell$ of $f$ follows the following deterministic update rule in the limit $n_\ell \to \infty$.
$$ \cos (\theta_{\ell+1}) = 2 J_{1,1}(\theta_\ell).$$
\end{lemma}
\begin{remark}
    Note that a more general proof of this result  appears in prior work \cite{boris_correlation_functions} which allows one to take the layer sizes $n_1,n_2,\ldots,n_\ell \to \infty$ in any order, rather than one layer at a time as we prove here. 
\end{remark}
\begin{proof}
We begin with the identity (\ref{sums_2}), and use the inner product to introduce $\cos(\theta_{\ell+1})$, 
\begin{gather*}
    \frac{\|\pha\| \|\phb\| }{n_\ell} \sum_{i=1}^{n_\ell} 2\ph(G_i) \ph(\hat{G}_i) = \langle \phaa, \phbb \rangle = \|\phaa\| \|\phbb\| \cos(\theta_{\ell+1}).
\end{gather*}
Applying the identities in \eqref{sums_1} to $\|\phaa\|$ and $\|\phbb\|$, we get
\begin{align*}
    \frac{\|\pha\| \|\phb\| }{n_\ell} \sum_{i=1}^{n_\ell} 2\ph(G_i) \ph(\hat{G}_i) &= \sqrt{\frac{\|\pha\|^2}{n_\ell} \sum_{i=1}^{n_\ell}2\ph^2(G_i)} \sqrt{\frac{\|\phb\|^2}{n_\ell} \sum_{i=1}^{n_\ell}2\ph^2(\hat{G}_i)} \cos(\theta_{\ell+1}), \\
    \implies \frac{1}{n_\ell} \sum_{i=1}^{n_\ell} \ph(G_i) \ph(\hat{G}_i) &= \sqrt{\frac{1}{n_\ell} \sum_{i=1}^{n_\ell}\ph^2(G_i)} \sqrt{\frac{1}{n_\ell} \sum_{i=1}^{n_\ell}\ph^2(\hat{G}_i)} \cos(\theta_{\ell+1}).
\end{align*}
Now, in the limit $n_\ell \to \infty$ we have by application of the Law of Large Numbers,
\begin{align*}
    \lim_{n_\ell \to \infty} \left( \frac{1}{n_\ell} \sum_{i=1}^{n_\ell} \ph(G_i) \ph(\hat{G}_i) \right) &= \lim_{n_\ell \to \infty} \left( \sqrt{\frac{1}{n_\ell} \sum_{i=1}^{n_\ell}\ph^2(G_i)} \sqrt{\frac{1}{n_\ell} \sum_{i=1}^{n_\ell}\ph^2(\hat{G}_i)} \cos(\theta_{\ell+1}) \right) \\
    \implies \mathbf{E}\left[\ph(G_i) \ph(\hat{G}_i) \right] &= \sqrt{\mathbf{E}\left[ \ph^2(G_i)\right]} \sqrt{\mathbf{E}[ \ph^2(\hat{G}_i)]} \cos(\theta_{\ell+1}) \\
    \implies J_{1,1}(\theta_\ell) &= \frac{1}{2} \cos(\theta_{\ell+1}),
\end{align*}
where we have used the definition of $J_{1,1}(\theta)$ and the fact that $\mathbf{E}[\ph^2(G)]=\frac{1}{2}$. 
\end{proof}


\section*{Appendix B.}

\fakesection{fakesection2}
\subsection{Derivation of Lower-Order $\mathbf{J}$ Functions - Proof of Proposition \ref{prop:j_small}}
\label{app:low_order}



\begin{proof} [Of formula for $J_{0,0}$]
We find a differential equation that $J_{0,0}$ satisfies and solve it to obtain the formula. First note the initial condition $J_{0,0}(0) = \mathbf{E}[1\{G>0\}] = \frac{1}{2}$. To find $J_{0,0}^\prime(\theta)$, we take the derivative inside the expectation and have by the chain rule that 
\begin{align*}
    J_{0,0}'(\theta) &= \mathbf{E}[1\{G>0\} 1'\{G \cos \theta + W \sin \theta >0 \}G](-\sin \theta) \\
    &+ \mathbf{E}[1\{G>0\} 1'\{G \cos \theta + W \sin \theta >0 \}W]\cos \theta.
\end{align*}

\noindent Applying the change of variables as in \eqref{cov}, we have 
\begin{align*}
     J_{0,0}'(\theta) &= \mathbf{E}[(Z\cos \theta  + Y \sin \theta ) 1\{Z \cos \theta  + Y\sin \theta >0\} 1'\{Z >0 \}](-\sin \theta) \\
    &+ \mathbf{E}[(Z \sin \theta  - Y \cos \theta ) 1\{Z \cos \theta  + Y \sin \theta >0\} 1'\{Z >0 \}]\cos \theta \\
    &= \mathbf{E}[(Y \sin \theta ) 1\{ Y \sin \theta >0\}]\frac{-\sin \theta}{\sqrt{2\pi}} + \mathbf{E}[(-Y \cos \theta ) 1\{ Y \sin \theta >0\}]\frac{\cos \theta}{\sqrt{2\pi}} \\
    &= (-\sin^2 \theta - \cos^2 \theta)\mathbf{E}[Y 1\{Y>0\}]\frac{1}{\sqrt{2\pi}} = -\frac{1}{2\pi},
\end{align*}
where we have used \eqref{eq:E_deriv_one} to evaluate the integrals involving $1^\prime\{Z>0\}$ and $\mathbf{E}[Y1\{Y>0\}] = (\sqrt{2\pi})^{-1}$ from Lemma \ref{lem:powers_of_phi}. We now have $J'_{0,0}(\theta) = -\frac{1}{2\pi}$ with initial condition given by $J_{0,0}(0) = 0$. Solving this differential equation gives the desired result. 
\end{proof}

\begin{proof}[$J_{1,0}$ and $J_{1,1}$] Here we use the Gaussian integration-by-parts strategy (\ref{eq:Gaussian_IBP} -\ref{eq:E_deriv_one}).\\

\noindent \textbf{Formula for $J_{1,0}$:}
\begin{align*}
    J_{1,0}(\theta) &= \mathbf{E}[G 1\{G>0\}\; 1\{G \cos \theta + W \sin \theta >0 \}] \\
    &= \mathbf{E}\left[\frac{d}{dg} \left( \; 1\{G>0\}\; 1\{G \cos \theta + W \sin \theta >0 \}\right)\right] \\
    &= \mathbf{E}\left[  1'\{G>0\}\; 1\{G \cos \theta + W \sin \theta >0 \}\right] + \mathbf{E}\left[  1\{G>0\}\; 1'\{G \cos \theta + W \sin \theta >0 \}\right]\cos \theta.
\end{align*}
By using the change of variables as in (\ref{cov}) on the second term, we arrive at
\begin{align*}
     J_{1,0}(\theta) &= \mathbf{E}[1\{W \sin \theta >0 \}]\frac{1}{\sqrt{2 \pi}} + \cos \theta \mathbf{E}[1 \{Z \cos \theta  + Y \sin \theta  > 0\} 1'\{Z>0\}] \\ 
     &= \frac{1}{2} \frac{1}{\sqrt{2 \pi}} + \cos \theta \mathbf{E}[1\{ Y \sin \theta  >0 \}]\frac{1}{\sqrt{2 \pi}} = \frac{1}{2} \frac{1}{\sqrt{2 \pi}} + \frac{\cos \theta}{2} \frac{1}{\sqrt{2 \pi}} = \frac{1+\cos \theta}{2 \sqrt{2 \pi}}.
\end{align*}
\noindent \textbf{Formula for $J_{1,1}$:}
\begin{align*}
    J_{1,1}(\theta) &= \mathbf{E}[G (G \cos \theta + W \sin \theta) 1\{G>0\} 1\{G \cos \theta + W \sin \theta > 0 \}] \\ 
    &= \mathbf{E}[\cos\theta \;1\{G>0\} 1\{G \cos \theta + W \sin \theta\}] \\
    &+ \mathbf{E}[(G \cos \theta + W \sin \theta) 1'\{G>0\} 1\{G \cos \theta + W \sin \theta >0\}] \\
    &+ \mathbf{E}[(G \cos \theta + W \sin \theta) 1\{G>0\} 1'\{G \cos \theta + W \sin \theta >0\}]\cos \theta \\
    &= \cos \theta J_{0,0} + \mathbf{E}[W \sin \theta \;1\{W \sin \theta > 0\}]\frac{1}{\sqrt{2 \pi}} + \mathbf{E}[Z 1\{Z \cos \theta + Y \sin \theta \} 1'\{z>0\}] \\
    &= \cos \theta J_{0,0} + \sin \theta \mathbf{E}[\varphi(W)]\frac{1}{\sqrt{2 \pi}} + 0  = \frac{\sin \theta + (\pi - \theta)\cos \theta }{2 \pi} .
\end{align*} 
\end{proof}


\subsection{Proof of Explicit Formulas for $J_{n,0}$ and $J_{n,1}$}
\label{app:j_proof}
Once the recursion is established, the formula for both $J_{n,0}$ and $J_{n,1}$ is a simple proof by induction. We provide a detailed proof for $J_{n,0}$ here; $J_{n,1}$ is similar.
\begin{lemma} Let $J_{n,0}^{rec}$ be the recursively defined formula, and let $J_{n,0}^{exp}$ be the explicitly defined formula for $J_{n,0}$, namely
\begin{gather*}
    J_{n,0}^{rec} := (n-1) J^{rec}_{n-2,0} + \frac{\sin^{n-1}\theta \cos \theta}{c_{n \bmod 2}}(n-2)!!, \quad J_{1,0}^{rec}:=J_{1,0}, \quad J_{0,0}^{rec}:=J_{0,0}, \\
    J_{n,0}^{exp} := (n-1)!!\left( J_{n \bmod 2,0} + \frac{\cos \theta }{c_{n \bmod 2}} \sum_{\substack{i \not\equiv n(\bmod 2) \\ 0<i<n}} \frac{(i-1)!!}{i!!} \sin^i \theta \right).
\end{gather*}
Then $J_{n,0}^{rec} = J_{n,0}^{exp}$ for all $n \geq 0$.
\end{lemma}

\begin{proof} Let $S_n$, $n \in \mathbb{N},\; n \geq 2$ be the statement $J_{n,0}^{rec} = J_{n,0}^{exp}\text{ and }J_{n-1,0}^{rec} = J_{n-1,0}^{exp}$. We prove $S_n$ is true by induction. The base case $S_2$ is true because,
\begin{align*}
    J_{2,0}^{rec} &= (2-1)J_{0,0} + \frac{\sin \theta \cos \theta}{c_{2 \bmod 2}}(2-2)!! = J_{0,0} + \frac{\cos\theta \sin\theta }{2\pi},\\
    J_{2,0}^{exp} &= (2-1)!!J_{0,0} + \cos \theta \sum_{i=1}^{1} \frac{(2-1)!!}{(2i-1)!!}(2i-2)!! \frac{\sin^{2i-1} \theta}{2\pi}
    = J_{0,0} + \frac{\cos\theta \sin\theta }{2\pi},
\end{align*}
and the fact that $J_{1,0}^{rec}=J_{1,0}^{exp}$ is trivial. Induction step:  Assume $S_n$ is true. To prove $S_{n+1}$, we have only to show that $J_{n+1,0}^{exp}=J_{n+1,0}^{rec}$. To do this, we separate the last term of the sum to get
\begin{align*} 
    J^{exp}_{n+1,0} &=  n!! \left( J_{(n+1)\bmod 2,0} + \frac{\cos \theta}{c_{(n+1)\bmod 2}} \sum_{\substack{i \not\equiv (n+1)(\bmod 2) \\ 0<i<n-1}} \frac{(i-1)!!}{i!!} \sin^i \theta \right)\\
    &+ n!!\frac{\cos \theta}{c_{(n+1)\bmod 2}}\frac{(n-1)!!}{n!!}\sin^n \theta. 
\end{align*}
Because the parity of $n+1$ and $n-1$ are the same, and using $n!!=n(n-2)!!$ we recognize the first term as $nJ^{exp}_{n-1,0}$. So after simplifying the last term, we remain with

    
\begin{equation*}
    J^{exp}_{n+1,0} = n J_{n-1,0}^{exp}  + \frac{\sin^n\theta \cos \theta}{c_{(n+1) \bmod 2}}(n-1)!! = J_{n+1,0}^{rec},
\end{equation*}
by the induction hypothesis. This completes the induction.
\end{proof}

\subsection{Bijection between Paths in Graphs of $\mathbf{J}$ Functions and the Bessel Number graphs $P$,$Q$}
\label{app:bijection}
Let $G_{J^*} = (V_{J*}, E_{J*})$ be the graph of $J_{a,b}^*$ as in Figure \ref{j}. Similarly, let $G_P = (V_P, E_P)$ and $G_Q = (V_Q, E_Q)$ be the graph of the $P$ and $Q$ matrices up to row $a$, respectively, as in Figures \ref{p}, \ref{q}. We define a map  $\lambda: \mathbb{Z}^2 \times \mathbb{Z}^2   \to \mathbb{Z}^2$ as follows: Let $((i,j),(m,n)) \in \mathbb{Z}^2 \times \mathbb{Z}^2, \; 0 \leq i \leq a,\; b-a+m \leq j\leq b$. Then define $\lambda$ by
\begin{equation*}
    \lambda((i,j),(m,n)) := (i-m, j-n), \quad \lambda^{-1}((i,j),(m, n)) = (i+m, j+n).
\end{equation*}
The function $\lambda$ can be used as a map between vertices of graph $G_{J^*}$ to vertices of graph $G_{P}$ or $G_Q$. Let $\pi = (v_1, v_2,...,v_{k-1},v_k)$ be a path in  $G_{J^*}$ from vertex $v_1 = (m,n)$ to vertex $v_k = (a,b)$, where $v_i \in \mathbb{Z}^2,\;1\leq i \leq k$ is a vertex on the graph. $\lambda$ extends to a map on paths, $\Lambda$, defined by
\begin{align*}
    \Lambda((v_1, v_2,...,v_{k-1},v_k)) :=& (\lambda(v_1, v_1), \lambda(v_2, v_1),...,\lambda(v_{k-1}, v_1),\lambda(v_k, v_1)), \\ \Lambda^{-1}((v_1, v_2,...,v_{k-1},v_k)) =& (\lambda^{-1}(v_1, v_1), \lambda^{-1}(v_2, v_1),...,\lambda^{-1}(v_{k-1}, v_1),\lambda^{-1}(v_k, v_1)).
\end{align*}
Now, let $\Gamma_{J^\ast}(a,b,m,n)$ be the set of all paths in the graph of $J^\ast$ from $J^\ast_{m,n}$ to $J^\ast_{a,b}$, and let $\Gamma_P(a,b,m,n)$ be the set of all paths in the graph of $P$ from $P(0,0)$ to $P(a-m, b-n) = P(\lambda((a,b),(m,n)))$. For example, $\Gamma_{J^\ast}(6,8,0,4)$ is the set of all paths which run from $J_{6,8}^*$ to $J_{0,4}^*$, and  $\Gamma_P(6,8,0,4)$ is the set of all paths which run from $P(0,0)$ to $P(6,4)$.

With these definitions, $\Lambda: \Gamma_{J^\ast}(a,b,0,n) \to \Gamma_P(a,b,0,n)$ is a bijection. An illustration of all paths $\pi \in  \Pi(6,8,0,6)$ and the corresponding paths $\Lambda(\pi) \in \Gamma_P(6,8,0,6)$ is given in Figure \ref{jp}. Similarly, if we let  $\Gamma_Q(a,b,m,n)$ be the set of all paths from $Q(0,0)$ to $Q(a-m, b-n) = Q(\lambda((a,b),(m,n)))$ then $\Lambda: \Gamma_{J^\ast}(a,b,1,n) \to \Gamma_Q(a,b,1,n)$ is a bijection. This bijection establishes the equality of the weighted paths claim in \eqref{eq:path_equality}.

\begin{figure}[h!]
    \centering
     \resizebox{16cm}{12cm}{
    \begin{tikzpicture}
    
    \foreach \i in {0,1,...,6}{
        \draw (-0.5,\i+0.5) node{\i};
    }
    
    \foreach \i\j in {0/8,1/7,2/6}{
        \draw (\i + 0.5, 7.5) node{\j};
    }
    
    \foreach \i in {0,...,3}{
        \pgfmathsetmacro{\ii}{2*\i}
        \fill[yellow!5] (0,\ii) rectangle (1,\ii+1);
        
        \pgfmathsetmacro{\ini}{int(\ii)}
        \draw (0.5,\ii+0.5) node{$J_{\ini,8}$};
    }
    
    \foreach \i in {0,...,2}{
        \pgfmathsetmacro{\ii}{2*\i+1}
        \fill[yellow!5] (1,\ii) rectangle (2,\ii+1);
        
        \pgfmathsetmacro{\ini}{int(\ii)}
        \draw (1.5,\ii+0.5) node{$J_{\ini,7}$};
    }
    
    \foreach \i in {0,...,2}{
        \pgfmathsetmacro{\ii}{2*\i}
        \fill[yellow!5] (2,\ii) rectangle (3,\ii+1);
        
        \pgfmathsetmacro{\ini}{int(\ii)}
        \draw (2.5,\ii+0.5) node{$J_{\ini,6}$};
    }
    
    \draw[step = 1, black, thin] (0,0) grid (3,7);
    
    \draw[red,thick,->] (0.5, 5) -- (0.5, 6);
    \draw (0.65, 5.5) node[red, scale = 0.75]{5};
    
    \draw[red,thick,->] (0.5, 3) -- (0.5, 4);
    \draw (0.65, 3.5) node[red, scale = 0.75]{3};
    
    \draw[blue,thick,->] (1.25, 1.75) --(0.75, 2.25);
    \draw (1.15, 2.15) node[blue, scale = 0.75]{1};
    
    \draw[blue,thick,->] (2.25, 0.75) --(1.75, 1.25);
    \draw (2.15, 1.15) node[blue, scale = 0.75]{1};

    \begin{scope}[shift={(3.5,0)}]
    
    \foreach \i\j in {0/8,1/7,2/6}{
        \draw (\i + 0.5, 7.5) node{\j};
    }
    
    \foreach \i in {0,...,3}{
        \pgfmathsetmacro{\ii}{2*\i}
        \fill[yellow!5] (0,\ii) rectangle (1,\ii+1);
        
        \pgfmathsetmacro{\ini}{int(\ii)}
        \draw (0.5,\ii+0.5) node{$J_{\ini,8}$};
    }
    
    \foreach \i in {0,...,2}{
        \pgfmathsetmacro{\ii}{2*\i+1}
        \fill[yellow!5] (1,\ii) rectangle (2,\ii+1);
        
        \pgfmathsetmacro{\ini}{int(\ii)}
        \draw (1.5,\ii+0.5) node{$J_{\ini,7}$};
    }
    
    \foreach \i in {0,...,2}{
        \pgfmathsetmacro{\ii}{2*\i}
        \fill[yellow!5] (2,\ii) rectangle (3,\ii+1);
        
        \pgfmathsetmacro{\ini}{int(\ii)}
        \draw (2.5,\ii+0.5) node{$J_{\ini,6}$};
    }
    
    \draw[step = 1, black, thin] (0,0) grid (3,7);
    
    \draw[red,thick,->] (0.5, 5) -- (0.5, 6);
    \draw (0.65, 5.5) node[red, scale = 0.75]{5};
    
    \draw[red,thick,->] (1.5, 2) -- (1.5, 3);
    \draw (1.65, 2.5) node[red, scale = 0.75]{2};
    
    \draw[blue,thick,->] (1.25, 3.75) --(0.75, 4.25);
    \draw (1.15, 4.15) node[blue, scale = 0.75]{1};
    
    \draw[blue,thick,->] (2.25, 0.75) --(1.75, 1.25);
    \draw (2.15, 1.15) node[blue, scale = 0.75]{1};
    
    \end{scope}
    
    \begin{scope}[shift={(7,0)}]
    
    \foreach \i\j in {0/8,1/7,2/6}{
        \draw (\i + 0.5, 7.5) node{\j};
    }
    
    \foreach \i in {0,...,3}{
        \pgfmathsetmacro{\ii}{2*\i}
        \fill[yellow!5] (0,\ii) rectangle (1,\ii+1);
        
        \pgfmathsetmacro{\ini}{int(\ii)}
        \draw (0.5,\ii+0.5) node{$J_{\ini,8}$};
    }
    
    \foreach \i in {0,...,2}{
        \pgfmathsetmacro{\ii}{2*\i+1}
        \fill[yellow!5] (1,\ii) rectangle (2,\ii+1);
        
        \pgfmathsetmacro{\ini}{int(\ii)}
        \draw (1.5,\ii+0.5) node{$J_{\ini,7}$};
    }
    
    \foreach \i in {0,...,2}{
        \pgfmathsetmacro{\ii}{2*\i}
        \fill[yellow!5] (2,\ii) rectangle (3,\ii+1);
        
        \pgfmathsetmacro{\ini}{int(\ii)}
        \draw (2.5,\ii+0.5) node{$J_{\ini,6}$};
    }
    
    \draw[step = 1, black, thin] (0,0) grid (3,7);
    
    \draw[red,thick,->] (0.5, 5) -- (0.5, 6);
    \draw (0.65, 5.5) node[red, scale = 0.75]{5};
    
    \draw[red,thick,->] (2.5, 1) -- (2.5, 2);
    \draw (2.65, 1.5) node[red, scale = 0.75]{1};
    
    \draw[blue,thick,->] (1.25, 3.75) --(0.75, 4.25);
    \draw (1.15, 4.15) node[blue, scale = 0.75]{1};
    
    \draw[blue,thick,->] (2.25, 2.75) --(1.75, 3.25);
    \draw (2.15, 3.15) node[blue, scale = 0.75]{1};
    
    \end{scope}
    
    \begin{scope}[shift={(10.5,0)}]
    
    \foreach \i\j in {0/8,1/7,2/6}{
        \draw (\i + 0.5, 7.5) node{\j};
    }
    
    \foreach \i in {0,...,3}{
        \pgfmathsetmacro{\ii}{2*\i}
        \fill[yellow!5] (0,\ii) rectangle (1,\ii+1);
        
        \pgfmathsetmacro{\ini}{int(\ii)}
        \draw (0.5,\ii+0.5) node{$J_{\ini,8}$};
    }
    
    \foreach \i in {0,...,2}{
        \pgfmathsetmacro{\ii}{2*\i+1}
        \fill[yellow!5] (1,\ii) rectangle (2,\ii+1);
        
        \pgfmathsetmacro{\ini}{int(\ii)}
        \draw (1.5,\ii+0.5) node{$J_{\ini,7}$};
    }
    
    \foreach \i in {0,...,2}{
        \pgfmathsetmacro{\ii}{2*\i}
        \fill[yellow!5] (2,\ii) rectangle (3,\ii+1);
        
        \pgfmathsetmacro{\ini}{int(\ii)}
        \draw (2.5,\ii+0.5) node{$J_{\ini,6}$};
    }
    
    \draw[step = 1, black, thin] (0,0) grid (3,7);
    
    \draw[red,thick,->] (1.5, 4) -- (1.5, 5);
    \draw (1.65, 4.5) node[red, scale = 0.75]{4};
    
    \draw[red,thick,->] (1.5, 2) -- (1.5, 3);
    \draw (1.65, 2.5) node[red, scale = 0.75]{2};
    
    \draw[blue,thick,->] (1.25, 5.75) --(0.75, 6.25);
    \draw (1.15, 6.15) node[blue, scale = 0.75]{1};
    
    \draw[blue,thick,->] (2.25, 0.75) --(1.75, 1.25);
    \draw (2.15, 1.15) node[blue, scale = 0.75]{1};
    
    \end{scope}
    
    \begin{scope}[shift={(14,0)}]
    
    \foreach \i\j in {0/8,1/7,2/6}{
        \draw (\i + 0.5, 7.5) node{\j};
    }
    
    \foreach \i in {0,...,3}{
        \pgfmathsetmacro{\ii}{2*\i}
        \fill[yellow!5] (0,\ii) rectangle (1,\ii+1);
        
        \pgfmathsetmacro{\ini}{int(\ii)}
        \draw (0.5,\ii+0.5) node{$J_{\ini,8}$};
    }
    
    \foreach \i in {0,...,2}{
        \pgfmathsetmacro{\ii}{2*\i+1}
        \fill[yellow!5] (1,\ii) rectangle (2,\ii+1);
        
        \pgfmathsetmacro{\ini}{int(\ii)}
        \draw (1.5,\ii+0.5) node{$J_{\ini,7}$};
    }
    
    \foreach \i in {0,...,2}{
        \pgfmathsetmacro{\ii}{2*\i}
        \fill[yellow!5] (2,\ii) rectangle (3,\ii+1);
        
        \pgfmathsetmacro{\ini}{int(\ii)}
        \draw (2.5,\ii+0.5) node{$J_{\ini,6}$};
    }
    
    \draw[step = 1, black, thin] (0,0) grid (3,7);
    
    \draw[red,thick,->] (1.5, 4) -- (1.5, 5);
    \draw (1.65, 4.5) node[red, scale = 0.75]{4};
    
    \draw[red,thick,->] (2.5, 1) -- (2.5, 2);
    \draw (2.65, 1.5) node[red, scale = 0.75]{1};
    
    \draw[blue,thick,->] (1.25, 5.75) --(0.75, 6.25);
    \draw (1.15, 6.15) node[blue, scale = 0.75]{1};
    
    \draw[blue,thick,->] (2.25, 2.75) --(1.75, 3.25);
    \draw (2.15, 3.15) node[blue, scale = 0.75]{1};
    \end{scope}
    
    \begin{scope}[shift={(17.5,0)}]
    
    \foreach \i\j in {0/8,1/7,2/6}{
        \draw (\i + 0.5, 7.5) node{\j};
    }
    
    \foreach \i in {0,...,3}{
        \pgfmathsetmacro{\ii}{2*\i}
        \fill[yellow!5] (0,\ii) rectangle (1,\ii+1);
        
        \pgfmathsetmacro{\ini}{int(\ii)}
        \draw (0.5,\ii+0.5) node{$J_{\ini,8}$};
    }
    
    \foreach \i in {0,...,2}{
        \pgfmathsetmacro{\ii}{2*\i+1}
        \fill[yellow!5] (1,\ii) rectangle (2,\ii+1);
        
        \pgfmathsetmacro{\ini}{int(\ii)}
        \draw (1.5,\ii+0.5) node{$J_{\ini,7}$};
    }
    
    \foreach \i in {0,...,2}{
        \pgfmathsetmacro{\ii}{2*\i}
        \fill[yellow!5] (2,\ii) rectangle (3,\ii+1);
        
        \pgfmathsetmacro{\ini}{int(\ii)}
        \draw (2.5,\ii+0.5) node{$J_{\ini,6}$};
    }
    
    \draw[step = 1, black, thin] (0,0) grid (3,7);
    
      \draw[red,thick,->] (2.5, 3) -- (2.5, 4);
    \draw (2.65, 3.5) node[red, scale = 0.75]{3};
    
    \draw[red,thick,->] (2.5, 1) -- (2.5, 2);
    \draw (2.65, 1.5) node[red, scale = 0.75]{1};
    
    \draw[blue,thick,->] (1.25, 5.75) --(0.75, 6.25);
    \draw (1.15, 6.15) node[blue, scale = 0.75]{1};
    
    \draw[blue,thick,->] (2.25, 4.75) --(1.75, 5.25);
    \draw (2.15, 5.15) node[blue, scale = 0.75]{1};
    
    \end{scope}

    \pgfmathsetmacro{\min}{-9}
    
    \foreach \i in {0,1,...,6}{
        \pgfmathsetmacro{\ii}{int(6-\i)}
        \draw (-0.5,\i+0.5 + \min) node{\ii};
    }
    
    \foreach \i in {0,1,2}{
        \draw (\i + 0.5, 7.5 + \min) node{\i};
    }
    
    \foreach \i in {0,...,3}{
        \pgfmathsetmacro{\ii}{2*\i}
        \fill[yellow!5] (0,\ii + \min) rectangle (1,\ii+1 + \min);
        
        \pgfmathsetmacro{\ini}{int(6-\ii)}
        \draw (0.5,\ii+0.5 + \min) node{\scriptsize $P(\ini,0)$};
    }
    
    \foreach \i in {0,...,2}{
        \pgfmathsetmacro{\ii}{2*\i+1}
        \fill[yellow!5] (1,\ii + \min) rectangle (2,\ii+1 + \min);
        
        \pgfmathsetmacro{\ini}{int(6-\ii)}
        \draw (1.5,\ii+0.5 + \min) node{\scriptsize $P(\ini,1)$};
    }
    
    \foreach \i in {0,...,2}{
        \pgfmathsetmacro{\ii}{2*\i}
        \fill[yellow!5] (2,\ii + \min) rectangle (3,\ii+1 + \min);
        
        \pgfmathsetmacro{\ini}{int(6-\ii)}
        \draw (2.5,\ii+0.5 + \min) node{\scriptsize $P(\ini,2)$};
    }

    \draw[step = 1cm, black, thin] (0,-9) grid (3,-2);
    
    \draw[red,thick,->] (2.5, 4 + \min) -- (2.5, 3 + \min);
    \draw (2.65, 3.5 + \min) node[red, scale = 0.75]{3};
    
    \draw[red,thick,->] (2.5, 2 + \min) -- (2.5, 1 + \min);
    \draw (2.65, 1.5 + \min) node[red, scale = 0.75]{5};
    
    \draw[blue,thick,->] (0.75, 6.25 + \min) -- (1.25, 5.75 + \min);
    \draw (1.15, 6.15 + \min) node[blue, scale = 0.75]{1};
    
    \draw[blue,thick,->] (1.75, 5.25 + \min) -- (2.25, 4.75 + \min);
    \draw (2.15, 5.15 + \min) node[blue, scale = 0.75]{1};
    
    \begin{scope}[shift={(3.5,0)}]
    \foreach \i in {0,1,2}{
        \draw (\i + 0.5, 7.5 + \min) node{\i};
    }
    
    \foreach \i in {0,...,3}{
        \pgfmathsetmacro{\ii}{2*\i}
        \fill[yellow!5] (0,\ii + \min) rectangle (1,\ii+1 + \min);
        
        \pgfmathsetmacro{\ini}{int(6-\ii)}
        \draw (0.5,\ii+0.5 + \min) node{\scriptsize $P(\ini,0)$};
    }
    
    \foreach \i in {0,...,2}{
        \pgfmathsetmacro{\ii}{2*\i+1}
        \fill[yellow!5] (1,\ii + \min) rectangle (2,\ii+1 + \min);
        
        \pgfmathsetmacro{\ini}{int(6-\ii)}
        \draw (1.5,\ii+0.5 + \min) node{\scriptsize $P(\ini,1)$};
    }
    
    \foreach \i in {0,...,2}{
        \pgfmathsetmacro{\ii}{2*\i}
        \fill[yellow!5] (2,\ii + \min) rectangle (3,\ii+1 + \min);
        
        \pgfmathsetmacro{\ini}{int(6-\ii)}
        \draw (2.5,\ii+0.5 + \min) node{\scriptsize $P(\ini,2)$};
    }

    \draw[step = 1cm, black, thin] (0,-9) grid (3,-2);
    
     \draw[red,thick,->] (1.5, 5 + \min) -- (1.5, 4 + \min);
    \draw (1.65, 4.5 + \min) node[red, scale = 0.75]{2};
    
    \draw[red,thick,->] (2.5, 2 + \min) -- (2.5, 1 + \min);
    \draw (2.65, 1.5 + \min) node[red, scale = 0.75]{5};
    
    \draw[blue,thick,->] (0.75, 6.25 + \min) -- (1.25, 5.75 + \min);
    \draw (1.15, 6.15 + \min) node[blue, scale = 0.75]{1};
    
    \draw[blue,thick,->] (1.75, 3.25 + \min) -- (2.25, 2.75 + \min);
    \draw (2.15, 3.15 + \min) node[blue, scale = 0.75]{1};
    
    \end{scope}
    
    \begin{scope}[shift={(7,0)}]
    \foreach \i in {0,1,2}{
        \draw (\i + 0.5, 7.5 + \min) node{\i};
    }
    
    \foreach \i in {0,...,3}{
        \pgfmathsetmacro{\ii}{2*\i}
        \fill[yellow!5] (0,\ii + \min) rectangle (1,\ii+1 + \min);
        
        \pgfmathsetmacro{\ini}{int(6-\ii)}
        \draw (0.5,\ii+0.5 + \min) node{\scriptsize $P(\ini,0)$};
    }
    
    \foreach \i in {0,...,2}{
        \pgfmathsetmacro{\ii}{2*\i+1}
        \fill[yellow!5] (1,\ii + \min) rectangle (2,\ii+1 + \min);
        
        \pgfmathsetmacro{\ini}{int(6-\ii)}
        \draw (1.5,\ii+0.5 + \min) node{\scriptsize $P(\ini,1)$};
    }
    
    \foreach \i in {0,...,2}{
        \pgfmathsetmacro{\ii}{2*\i}
        \fill[yellow!5] (2,\ii + \min) rectangle (3,\ii+1 + \min);
        
        \pgfmathsetmacro{\ini}{int(6-\ii)}
        \draw (2.5,\ii+0.5 + \min) node{\scriptsize $P(\ini,2)$};
    }

    \draw[step = 1cm, black, thin] (0,-9) grid (3,-2);
    
    \draw[red,thick,->] (0.5, 6 + \min) -- (0.5, 5 + \min);
    \draw (0.65, 5.5 + \min) node[red, scale = 0.75]{1};
    
    \draw[red,thick,->] (2.5, 2 + \min) -- (2.5, 1 + \min);
    \draw (2.65, 1.5 + \min) node[red, scale = 0.75]{5};
    
    \draw[blue,thick,->] (0.75, 4.25 + \min) -- (1.25, 3.75 + \min);
    \draw (1.15, 4.15 + \min) node[blue, scale = 0.75]{1};
    
    \draw[blue,thick,->] (1.75, 3.25 + \min) -- (2.25, 2.75 + \min);
    \draw (2.15, 3.15 + \min) node[blue, scale = 0.75]{1};
    
    \end{scope}
    
    \begin{scope}[shift={(10.5,0)}]
    \foreach \i in {0,1,2}{
        \draw (\i + 0.5, 7.5 + \min) node{\i};
    }
    
    \foreach \i in {0,...,3}{
        \pgfmathsetmacro{\ii}{2*\i}
        \fill[yellow!5] (0,\ii + \min) rectangle (1,\ii+1 + \min);
        
        \pgfmathsetmacro{\ini}{int(6-\ii)}
        \draw (0.5,\ii+0.5 + \min) node{\scriptsize $P(\ini,0)$};
    }
    
    \foreach \i in {0,...,2}{
        \pgfmathsetmacro{\ii}{2*\i+1}
        \fill[yellow!5] (1,\ii + \min) rectangle (2,\ii+1 + \min);
        
        \pgfmathsetmacro{\ini}{int(6-\ii)}
        \draw (1.5,\ii+0.5 + \min) node{\scriptsize $P(\ini,1)$};
    }
    
    \foreach \i in {0,...,2}{
        \pgfmathsetmacro{\ii}{2*\i}
        \fill[yellow!5] (2,\ii + \min) rectangle (3,\ii+1 + \min);
        
        \pgfmathsetmacro{\ini}{int(6-\ii)}
        \draw (2.5,\ii+0.5 + \min) node{\scriptsize $P(\ini,2)$};
    }

    \draw[step = 1cm, black, thin] (0,-9) grid (3,-2);
    
     \draw[red,thick,->] (1.5, 5 + \min) -- (1.5, 4 + \min);
    \draw (1.65, 4.5 + \min) node[red, scale = 0.75]{2};
    
    \draw[red,thick,->] (1.5, 3 + \min) -- (1.5, 2 + \min);
    \draw (1.65, 2.5 + \min) node[red, scale = 0.75]{4};
    
    \draw[blue,thick,->] (0.75, 6.25 + \min) -- (1.25, 5.75 + \min);
    \draw (1.15, 6.15 + \min) node[blue, scale = 0.75]{1};
    
    \draw[blue,thick,->] (1.75, 1.25 + \min) -- (2.25, 0.75 + \min);
    \draw (2.15, 1.15 + \min) node[blue, scale = 0.75]{1};
    
    \end{scope}
    
    \begin{scope}[shift={(14,0)}]
    \foreach \i in {0,1,2}{
        \draw (\i + 0.5, 7.5 + \min) node{\i};
    }
    
    \foreach \i in {0,...,3}{
        \pgfmathsetmacro{\ii}{2*\i}
        \fill[yellow!5] (0,\ii + \min) rectangle (1,\ii+1 + \min);
        
        \pgfmathsetmacro{\ini}{int(6-\ii)}
        \draw (0.5,\ii+0.5 + \min) node{\scriptsize $P(\ini,0)$};
    }
    
    \foreach \i in {0,...,2}{
        \pgfmathsetmacro{\ii}{2*\i+1}
        \fill[yellow!5] (1,\ii + \min) rectangle (2,\ii+1 + \min);
        
        \pgfmathsetmacro{\ini}{int(6-\ii)}
        \draw (1.5,\ii+0.5 + \min) node{\scriptsize $P(\ini,1)$};
    }
    
    \foreach \i in {0,...,2}{
        \pgfmathsetmacro{\ii}{2*\i}
        \fill[yellow!5] (2,\ii + \min) rectangle (3,\ii+1 + \min);
        
        \pgfmathsetmacro{\ini}{int(6-\ii)}
        \draw (2.5,\ii+0.5 + \min) node{\scriptsize $P(\ini,2)$};
    }

    \draw[step = 1cm, black, thin] (0,-9) grid (3,-2);
    
    \draw[red,thick,->] (0.5, 6 + \min) -- (0.5, 5 + \min);
    \draw (0.65, 5.5 + \min) node[red, scale = 0.75]{1};
    
    \draw[red,thick,->] (1.5, 3 + \min) -- (1.5, 2 + \min);
    \draw (1.65, 2.5 + \min) node[red, scale = 0.75]{4};
    
    \draw[blue,thick,->] (0.75, 4.25 + \min) -- (1.25, 3.75 + \min);
    \draw (1.15, 4.15 + \min) node[blue, scale = 0.75]{1};
    
    \draw[blue,thick,->] (1.75, 1.25 + \min) -- (2.25, 0.75 + \min);
    \draw (2.15, 1.15 + \min) node[blue, scale = 0.75]{1};
    
    \end{scope}
    
    \begin{scope}[shift={(17.5,0)}]
    \foreach \i in {0,1,2}{
        \draw (\i + 0.5, 7.5 + \min) node{\i};
    }
    
    \foreach \i in {0,...,3}{
        \pgfmathsetmacro{\ii}{2*\i}
        \fill[yellow!5] (0,\ii + \min) rectangle (1,\ii+1 + \min);
        
        \pgfmathsetmacro{\ini}{int(6-\ii)}
        \draw (0.5,\ii+0.5 + \min) node{\scriptsize $P(\ini,0)$};
    }
    
    \foreach \i in {0,...,2}{
        \pgfmathsetmacro{\ii}{2*\i+1}
        \fill[yellow!5] (1,\ii + \min) rectangle (2,\ii+1 + \min);
        
        \pgfmathsetmacro{\ini}{int(6-\ii)}
        \draw (1.5,\ii+0.5 + \min) node{\scriptsize $P(\ini,1)$};
    }
    
    \foreach \i in {0,...,2}{
        \pgfmathsetmacro{\ii}{2*\i}
        \fill[yellow!5] (2,\ii + \min) rectangle (3,\ii+1 + \min);
        
        \pgfmathsetmacro{\ini}{int(6-\ii)}
        \draw (2.5,\ii+0.5 + \min) node{\scriptsize $P(\ini,2)$};
    }

    \draw[step = 1cm, black, thin] (0,-9) grid (3,-2);
    
    
    \draw[red,thick,->] (0.5, 6 + \min) -- (0.5, 5 + \min);
    \draw (0.65, 5.5 + \min) node[red, scale = 0.75]{1};
    
    \draw[red,thick,->] (0.5, 4 + \min) -- (0.5, 3 + \min);
    \draw (0.65, 3.5 + \min) node[red, scale = 0.75]{3};
    
    \draw[blue,thick,->] (0.75, 2.25 + \min) -- (1.25, 1.75 + \min);
    \draw (1.15, 2.15 + \min) node[blue, scale = 0.75]{1};
    
    \draw[blue,thick,->] (1.75, 1.25 + \min) -- (2.25, 0.75 + \min);
    \draw (2.15, 1.15 + \min) node[blue, scale = 0.75]{1};
    
    \end{scope}

    \end{tikzpicture}
    }
    \caption{Top: All paths  $\pi \in  \Gamma_{J^\ast}(6,8,0,6)$. Bottom: All paths $\Lambda(\pi) \in \Gamma_P(6,8,0,6)$. The paths are lined up so that for each path $\pi$ in the top row, $\Lambda(\pi)$ appears in the bottom row. }
    \label{jp}
\end{figure}
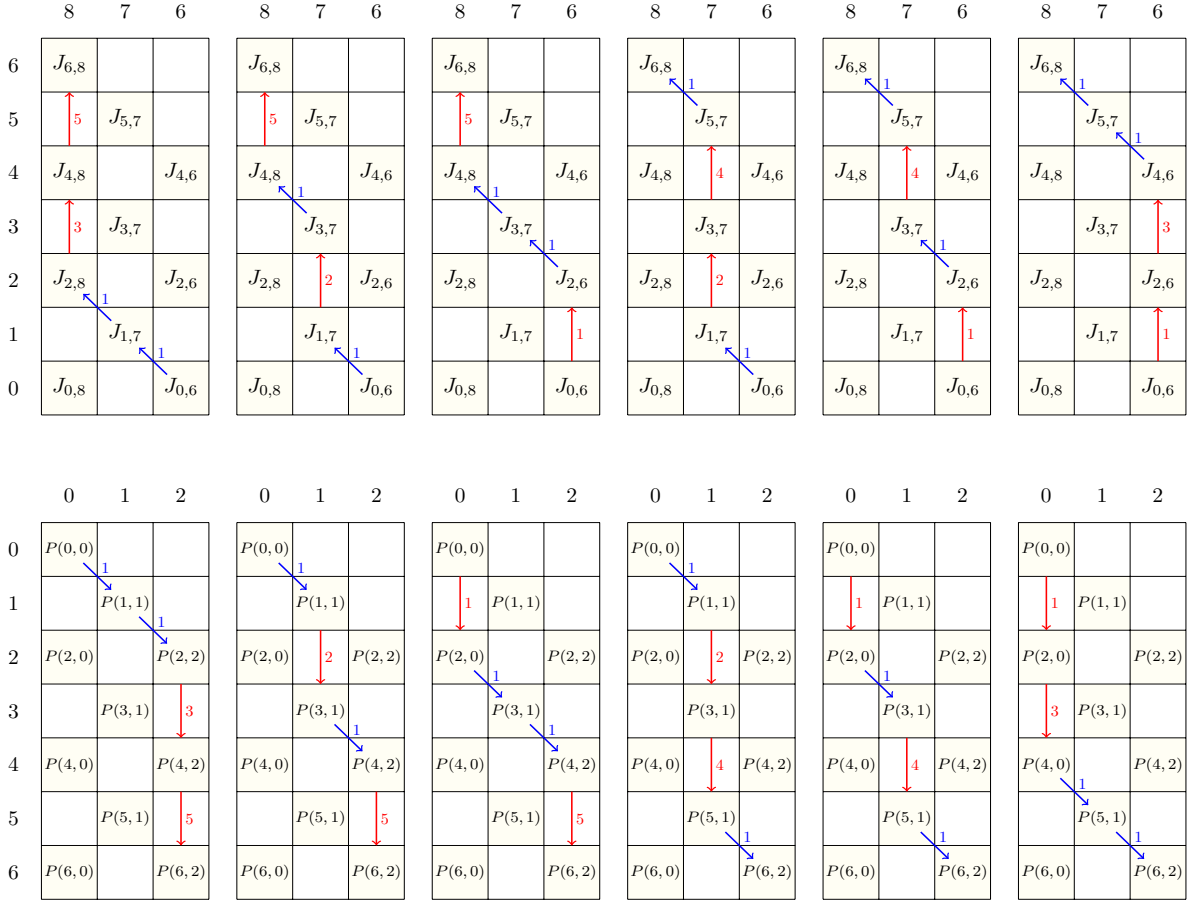



\section{Recursions for the $P$ and $Q$ numbers - Proof of Lemma \ref{lem:PQ_rec}}
\label{app:P_Q}

\noindent Earlier work established the following properties of the $P$ and $Q$ numbers.
\begin{theorem}[\citet{pq}]
The elements of the matrices $P$ and $Q$ satisfy
\begin{align}
    b \cdot P(a,b) &= a \cdot P(a-1,b-1), & &a \geq 1, 1 \leq b \leq a, \label{pq1}\\
    P(a+1,b) &= P(a,b-1) + (b+1) \cdot P(a,b+1), & &a \geq 0, 1 \leq b \leq a,\label{pq2}\\
    Q(a,b) &= P(a,b) + (b+1) \cdot Q(a-1,b+1), & &a\geq 1, 1\leq b\leq a, \label{pq3}\\ 
    Q(a,b) &=  a \cdot Q(a-2,b)+ Q(a-1,b-1),  & &a\geq 2, 1\leq b\leq a. \label{pq4}
\end{align}
\end{theorem}

\begin{proof}[Of Lemma \ref{lem:PQ_rec}]  Equation (\ref{pq2}) tells us that $P(a,b) = P(a-1,b-1) + (b+1) \cdot P(a-1,b+1)$ for $a \geq 1, 1 \leq b \leq a-1$, while equation (\ref{pq1}) tells us that $P(a-1,b+1) = \frac{(a-1)}{(b+1)} \cdot P(a-2,b)$ for $a \geq 2, 0 \leq b \leq a-2$. Putting these together, we get the following recurrence equation for $P(a,b)$:
\begin{align*}
    P(a,b) &= P(a-1,b-1) + (b+1) \left( \frac{(a-1)}{(b+1)} \cdot P(a-2,b) \right) \\
    &= (a-1) \cdot  P(a-2,b) + P(a-1,b-1),
\end{align*}
which holds for $a \geq 3, 1 \leq b \leq a-2$. Further, looking at equation (\ref{pq4}), we see that the recursion for the $Q$ numbers is very similar to that of the $P$ numbers, but with a coefficient of $a$ rather than $(a-1)$. This establishes Lemma \ref{lem:PQ_rec}.
\end{proof}

\section{Details of Network Architectures}
\label{app:network_architectures}
This section details the architectures of the 45 different network architectures used to produce \Cref{fig:simulations}.

\renewcommand{\arraystretch}{1}

\begin{table}[h]
    \centering
    \begin{tabular}{|c|c|c|c|c|c|c|c|}
    \hline
        \multirow{2}{*}{\#} & \multirow{2}{*}{Depth} & Avg. & \multicolumn{2}{c|}{\# Parameters}  &  \multicolumn{3}{c|}{Avg. Test Accuracy $\pm$ Standard Deviation}  \\ \cline{4-8}
         &  & Width & (F)MNIST & CIFAR & MNIST & FMNIST & CIFAR \\ \hline \hline
        1 & 2 & 50 & 58880 & 165790 & 0.924 $\pm$ 0.007 & 0.79 $\pm$ 0.02 & 0.211 $\pm$ 0.029 \\ \hline
        2 & 2 & 85 & 57350 & 135510 & 0.837 $\pm$ 0.051 & 0.709 $\pm$ 0.028 & 0.276 $\pm$ 0.011 \\ \hline
        3 & 2 & 200 & 19930 & 54250 & 0.878 $\pm$ 0.009 & 0.721 $\pm$ 0.098 & 0.163 $\pm$ 0.048 \\ \hline
        4 & 2 & 25 & 138300 & 201600 & 0.94 $\pm$ 0.004 & 0.812 $\pm$ 0.009 & 0.229 $\pm$ 0.025 \\ \hline
        5 & 2 & 125 & 31725 & 88925 & 0.89 $\pm$ 0.005 & 0.768 $\pm$ 0.013 & 0.199 $\pm$ 0.027 \\ \hline
        6 & 3 & 25 & 43990 & 114550 & 0.928 $\pm$ 0.008 & 0.812 $\pm$ 0.013 & 0.167 $\pm$ 0.022 \\ \hline
        7 & 3 & 50 & 62830 & 173280 & 0.916 $\pm$ 0.002 & 0.79 $\pm$ 0.012 & 0.224 $\pm$ 0.019 \\ \hline
        8 & 3 & 100 & 59700 & 96756 & 0.952 $\pm$ 0.004 & 0.839 $\pm$ 0.003 & 0.27 $\pm$ 0.016 \\ \hline
        9 & 3 & 67.67 & 87200 & 309900 & 0.924 $\pm$ 0.006 & 0.799 $\pm$ 0.011 & 0.281 $\pm$ 0.011 \\ \hline
        10 & 3 & 50 & 17310 & 189100 & 0.553 $\pm$ 0.181 & 0.599 $\pm$ 0.119 & 0.263 $\pm$ 0.022 \\ \hline
        11 & 4 & 30 & 369400 & 366150 & 0.877 $\pm$ 0.052 & 0.757 $\pm$ 0.026 & 0.192 $\pm$ 0.029 \\ \hline
        12 & 4 & 75 & 99400 & 105060 & 0.957 $\pm$ 0.003 & 0.842 $\pm$ 0.006 & 0.23 $\pm$ 0.025 \\ \hline
        13 & 5 & 21 & 74700 & 51630 & 0.931 $\pm$ 0.005 & 0.811 $\pm$ 0.009 & 0.146 $\pm$ 0.029 \\ \hline
        14 & 6 & 55 & 8840 & 976400 & 0.715 $\pm$ 0.088 & 0.569 $\pm$ 0.146 & 0.337 $\pm$ 0.008 \\ \hline
        15 & 6 & 87.5 & 169400 & 398200 & 0.949 $\pm$ 0.008 & 0.833 $\pm$ 0.007 & 0.332 $\pm$ 0.018 \\ \hline
        16 & 10 & 10 & 79020 & 180010 & 0.951 $\pm$ 0.003 & 0.832 $\pm$ 0.01 & 0.278 $\pm$ 0.018 \\ \hline
        17 & 10 & 100 & 64850 & 122050 & 0.939 $\pm$ 0.004 & 0.824 $\pm$ 0.008 & 0.262 $\pm$ 0.059 \\ \hline
        18 & 10 & 200 & 54170 & 262060 & 0.933 $\pm$ 0.005 & 0.81 $\pm$ 0.014 & 0.335 $\pm$ 0.016 \\ \hline
        19 & 10 & 17.5 & 49920 & 1002300 & 0.794 $\pm$ 0.052 & 0.648 $\pm$ 0.106 & 0.184 $\pm$ 0.026 \\ \hline
        20 & 11 & 34.55 & 518800 & 31720 & 0.955 $\pm$ 0.006 & 0.835 $\pm$ 0.011 & 0.14 $\pm$ 0.037 \\ \hline
    \end{tabular}
    \caption[Network Architecture Summary: Networks 1-20]{Summary of the architectures of the first 20 neural networks used in \Cref{fig:simulations}, as well as their performance on the test datasets. Note that the number of parameters differs between the (F)MNIST and CIFAR-10 datasets due to the fact that CIFAR-10 images are in colour requiring 3 colour channels, while the MNIST and FMNIST images are in grayscale. This table is continued in \Cref{tab:summary_2}. }
    \label{tab:summary_1}
\end{table}

\begin{table}[h]
    \centering
    \begin{tabular}{|c|c|c|c|c|c|c|c|}
    \hline
        \multirow{2}{*}{\#} & \multirow{2}{*}{Depth} & Avg. & \multicolumn{2}{c|}{\# Parameters}  &  \multicolumn{3}{c|}{Average Score $\pm$ Standard Deviation}  \\ \cline{4-8}
         &  & Width & (F)MNIST & CIFAR & MNIST & FMNIST & CIFAR \\ \hline \hline
        21 & 11 & 35 & 21100 & 269195 & 0.93 $\pm$ 0.005 & 0.823 $\pm$ 0.007 & 0.363 $\pm$ 0.016 \\ \hline
        22 & 13 & 42 & 36420 & 328200 & 0.91 $\pm$ 0.008 & 0.789 $\pm$ 0.01 & 0.364 $\pm$ 0.016 \\ \hline
        23 & 15 & 30 & 41844 & 174100 & 0.92 $\pm$ 0.004 & 0.805 $\pm$ 0.011 & 0.349 $\pm$ 0.015 \\ \hline
        24 & 15 & 50 & 13860 & 235650 & 0.909 $\pm$ 0.005 & 0.8 $\pm$ 0.012 & 0.328 $\pm$ 0.02 \\ \hline
        25 & 15 & 75 & 16580 & 206848 & 0.927 $\pm$ 0.003 & 0.823 $\pm$ 0.007 & 0.359 $\pm$ 0.009 \\ \hline
        26 & 16 & 35 & 42200 & 159100 & 0.943 $\pm$ 0.004 & 0.838 $\pm$ 0.004 & 0.343 $\pm$ 0.021 \\ \hline
        27 & 16 & 22.5 & 198800 & 656400 & 0.963 $\pm$ 0.003 & 0.845 $\pm$ 0.01 & 0.37 $\pm$ 0.016 \\ \hline
        28 & 20 & 25 & 94900 & 323700 & 0.955 $\pm$ 0.002 & 0.843 $\pm$ 0.006 & 0.367 $\pm$ 0.006 \\ \hline
        29 & 20 & 50 & 60416 & 62340 & 0.951 $\pm$ 0.003 & 0.837 $\pm$ 0.005 & 0.163 $\pm$ 0.058 \\ \hline
        30 & 20 & 37.5 & 44700 & 156600 & 0.948 $\pm$ 0.003 & 0.834 $\pm$ 0.008 & 0.346 $\pm$ 0.028 \\ \hline
        31 & 23 & 31.30 & 194550 & 598200 & 0.927 $\pm$ 0.005 & 0.788 $\pm$ 0.008 & 0.17 $\pm$ 0.004 \\ \hline
        32 & 25 & 15 & 64050 & 48180 & 0.951 $\pm$ 0.002 & 0.84 $\pm$ 0.004 & 0.186 $\pm$ 0.071 \\ \hline
        33 & 25 & 75 & 55160 & 125880 & 0.899 $\pm$ 0.014 & 0.748 $\pm$ 0.033 & 0.274 $\pm$ 0.048 \\ \hline
        34 & 25 & 150 & 53760 & 64390 & 0.782 $\pm$ 0.077 & 0.676 $\pm$ 0.064 & 0.206 $\pm$ 0.041 \\ \hline
        35 & 28 & 35.71 & 74715 & 78300 & 0.953 $\pm$ 0.001 & 0.844 $\pm$ 0.001 & 0.244 $\pm$ 0.075 \\ \hline
        36 & 30 & 15 & 60860 & 152380 & 0.819 $\pm$ 0.08 & 0.719 $\pm$ 0.033 & 0.17 $\pm$ 0.02 \\ \hline
        37 & 30 & 30 & 18630 & 145280 & 0.862 $\pm$ 0.08 & 0.772 $\pm$ 0.017 & 0.168 $\pm$ 0.02 \\ \hline
        38 & 30 & 100 & 34360 & 146680 & 0.941 $\pm$ 0.003 & 0.826 $\pm$ 0.009 & 0.165 $\pm$ 0.022 \\ \hline
        39 & 30 & 26.67 & 659100 & 118560 & 0.932 $\pm$ 0.014 & 0.785 $\pm$ 0.011 & 0.175 $\pm$ 0.007 \\ \hline
        40 & 30 & 31.67 & 18435 & 52755 & 0.313 $\pm$ 0.131 & 0.349 $\pm$ 0.109 & 0.158 $\pm$ 0.026 \\ \hline
        41 & 35 & 40 & 86160 & 276600 & 0.753 $\pm$ 0.074 & 0.586 $\pm$ 0.11 & 0.148 $\pm$ 0.029 \\ \hline
        42 & 35 & 75 & 250800 & 450525 & 0.725 $\pm$ 0.163 & 0.608 $\pm$ 0.077 & 0.165 $\pm$ 0.007 \\ \hline
        43 & 40 & 50 & 137200 & 251600 & 0.522 $\pm$ 0.141 & 0.513 $\pm$ 0.089 & 0.167 $\pm$ 0.007 \\ \hline
        44 & 40 & 75 & 278925 & 422400 & 0.467 $\pm$ 0.123 & 0.466 $\pm$ 0.09 & 0.161 $\pm$ 0.022 \\ \hline
        45 & 50 & 50 & 162200 & 177680 & 0.242 $\pm$ 0.064 & 0.22 $\pm$ 0.042 & 0.161 $\pm$ 0.019 \\ \hline
    \end{tabular}
    \caption[Network Architecture Summary: Networks 21-45]{Continuation of \Cref{tab:summary_1} for networks 21 through 45.}
    \label{tab:summary_2}
\end{table}

\begin{table}[h]
    \centering
    \begin{tabular}{|c|l|}
    \hline
        \# & Hidden Layer Widths \\ \hline
        1 & 50, 50 \\ \hline
        2 & 85, 85 \\ \hline
        3 & 200, 200 \\ \hline
        4 & 20, 30 \\ \hline
        5 & 100, 150 \\ \hline
        6 & 25, 25, 25 \\ \hline
        7 & 50, 50, 50 \\ \hline
        8 & 100, 100, 100 \\ \hline
        9 & 64, 75, 64 \\ \hline
        10 & 75, 50, 25 \\ \hline
        11 & 40, 40, 20, 20 \\ \hline
        12 & 50, 100, 100, 50 \\ \hline
        13 & 15, 15, 15, 30, 30 \\ \hline
        14 & 80, 70, 60, 50, 40, 30 \\ \hline
        15 & 25, 50, 75, 100, 125, 150 \\ \hline
        16 & 10, 10, 10, 10, 10, 10, 10, 10, 10, 10 \\ \hline
        17 & 100, 100, 100, 100, 100, 100, 100, 100, 100, 100 \\ \hline
        18 & 200, 200, 200, 200, 200, 200, 200, 200, 200, 200 \\ \hline
        19 & 20, 20, 20, 20, 20, 15, 15, 15, 15, 15 \\ \hline
        20 & 55, 30, 30, 30, 30, 30, 30, 30, 30, 30, 55 \\ \hline
        21 & 40, 39, 38, 37, 36, 35, 34, 33, 32, 31, 30 \\ \hline
        22 & 24, 27, 30, 33, 36, 39, 42, 45, 48, 51, 54, 57, 60 \\ \hline
        23 & 30, 30, 30, 30, 30, 30, 30, 30, 30, 30, 30, 30, 30, 30, 30 \\ \hline
        24 & 50, 50, 50, 50, 50, 50, 50, 50, 50, 50, 50, 50, 50, 50, 50 \\ \hline
        25 & 75, 75, 75, 75, 75, 75, 75, 75, 75, 75, 75, 75, 75, 75, 75 \\ \hline
    \end{tabular}
    \caption[List of Hidden Layer Widths: Networks 1-25]{Ordered list of hidden layer widths for the first 25 networks used in \Cref{fig:simulations}. This table is continued in \Cref{tab:widths_2}.}
    \label{tab:widths_1}
\end{table}

\begin{table}[h]
    \centering
    \begin{tabular}{|c|p{15cm}|}
    \hline
        \# & Hidden Layer Widths \\ \hline
        26 & 50, 48, 46, 44, 42, 40, 38, 36, 34, 32, 30, 28, 26, 24, 22, 20 \\ \hline
        27 & 15, 16, 17, 18, 19, 20, 21, 22, 23, 24, 25, 26, 27, 28, 29, 30 \\ \hline
        28 & 25, 25, 25, 25, 25, 25, 25, 25, 25, 25, 25, 25, 25, 25, 25, 25, 25, 25, 25, 25 \\ \hline
        29 & 50, 50, 50, 50, 50, 50, 50, 50, 50, 50, 50, 50, 50, 50, 50, 50, 50, 50, 50, 50 \\ \hline
        30 & 45, 45, 45, 45, 45, 40, 40, 40, 40, 40, 35, 35, 35, 35, 35, 30, 30, 30, 30, 30 \\ \hline
        31 & 40, 40, 40, 40, 40, 40, 40, 40, 40, 40, 40, 40, 40, 20, 20, 20, 20, 20, 20, 20, 20, 20, 20 \\ \hline
        32 & 15, 15, 15, 15, 15, 15, 15, 15, 15, 15, 15, 15, 15, 15, 15, 15, 15, 15, 15, 15, 15, 15, 15, 15, 15 \\ \hline
        33 & 75, 75, 75, 75, 75, 75, 75, 75, 75, 75, 75, 75, 75, 75, 75, 75, 75, 75, 75, 75, 75, 75, 75, 75, 75 \\ \hline
        34 & 150, 150, 150, 150, 150, 150, 150, 150, 150, 150, 150, 150, 150, 150, 150, 150, 150, 150, 150, 150, 150, 150, 150, 150, 150 \\ \hline
        35 & 25, 25, 25, 25, 50, 50, 50, 50, 25, 25, 25, 25, 50, 50, 50, 50, 25, 25, 25, 25, 50, 50, 50, 50, 25, 25, 25, 25 \\ \hline
        36 & 15, 15, 15, 15, 15, 15, 15, 15, 15, 15, 15, 15, 15, 15, 15, 15, 15, 15, 15, 15, 15, 15, 15, 15, 15, 15, 15, 15, 15, 15 \\ \hline
        37 & 30, 30, 30, 30, 30, 30, 30, 30, 30, 30, 30, 30, 30, 30, 30, 30, 30, 30, 30, 30, 30, 30, 30, 30, 30, 30, 30, 30, 30, 30 \\ \hline
        38 & 100, 100, 100, 100, 100, 100, 100, 100, 100, 100, 100, 100, 100, 100, 100, 100, 100, 100, 100, 100, 100, 100, 100, 100, 100, 100, 100, 100, 100, 100 \\ \hline
        39 & 40, 40, 40, 40, 40, 20, 20, 20, 20, 20, 20, 20, 20, 20, 20, 20, 20, 20, 20, 20, 20, 20, 20, 20, 20, 40, 40, 40, 40, 40 \\ \hline
        40 & 40, 40, 40, 40, 40, 30, 30, 30, 30, 30, 30, 30, 30, 30, 30, 30, 30, 30, 30, 30, 30, 30, 30, 30, 30, 30, 30, 30, 30, 30 \\ \hline
        41 & 40, 40, 40, 40, 40, 40, 40, 40, 40, 40, 40, 40, 40, 40, 40, 40, 40, 40, 40, 40, 40, 40, 40, 40, 40, 40, 40, 40, 40, 40, 40, 40, 40, 40, 40 \\ \hline
        42 & 75, 75, 75, 75, 75, 75, 75, 75, 75, 75, 75, 75, 75, 75, 75, 75, 75, 75, 75, 75, 75, 75, 75, 75, 75, 75, 75, 75, 75, 75, 75, 75, 75, 75, 75 \\ \hline
        43 & 50, 50, 50, 50, 50, 50, 50, 50, 50, 50, 50, 50, 50, 50, 50, 50, 50, 50, 50, 50, 50, 50, 50, 50, 50, 50, 50, 50, 50, 50, 50, 50, 50, 50, 50, 50, 50, 50, 50, 50 \\ \hline
        44 & 75, 75, 75, 75, 75, 75, 75, 75, 75, 75, 75, 75, 75, 75, 75, 75, 75, 75, 75, 75, 75, 75, 75, 75, 75, 75, 75, 75, 75, 75, 75, 75, 75, 75, 75, 75, 75, 75, 75, 75 \\ \hline
        45 & 50, 50, 50, 50, 50, 50, 50, 50, 50, 50, 50, 50, 50, 50, 50, 50, 50, 50, 50, 50, 50, 50, 50, 50, 50, 50, 50, 50, 50, 50, 50, 50, 50, 50, 50, 50, 50, 50, 50, 50, 50, 50, 50, 50, 50, 50, 50, 50, 50, 50 \\ \hline
    \end{tabular}
    \caption[List of Hidden Layer Widths: Networks 26-45]{Continuation of \Cref{tab:widths_1} for networks 26 through 45.}
    \label{tab:widths_2}
\end{table}

\FloatBarrier
\bibliography{sources}

\begin{thebibliography}{23}
\providecommand{\natexlab}[1]{#1}
\providecommand{\url}[1]{\texttt{#1}}
\expandafter\ifx\csname urlstyle\endcsname\relax
  \providecommand{\doi}[1]{doi: #1}\else
  \providecommand{\doi}{doi: \begingroup \urlstyle{rm}\Url}\fi

\bibitem[Abadi et~al.(2015)Abadi, Agarwal, Barham, Brevdo, Chen, Citro, Corrado, Davis, Dean, Devin, Ghemawat, Goodfellow, Harp, Irving, Isard, Jia, Jozefowicz, Kaiser, Kudlur, Levenberg, Man\'{e}, Monga, Moore, Murray, Olah, Schuster, Shlens, Steiner, Sutskever, Talwar, Tucker, Vanhoucke, Vasudevan, Vi\'{e}gas, Vinyals, Warden, Wattenberg, Wicke, Yu, and Zheng]{tensorflow}
Mart\'{i}n Abadi, Ashish Agarwal, Paul Barham, Eugene Brevdo, Zhifeng Chen, Craig Citro, Greg~S. Corrado, Andy Davis, Jeffrey Dean, Matthieu Devin, Sanjay Ghemawat, Ian Goodfellow, Andrew Harp, Geoffrey Irving, Michael Isard, Yangqing Jia, Rafal Jozefowicz, Lukasz Kaiser, Manjunath Kudlur, Josh Levenberg, Dandelion Man\'{e}, Rajat Monga, Sherry Moore, Derek Murray, Chris Olah, Mike Schuster, Jonathon Shlens, Benoit Steiner, Ilya Sutskever, Kunal Talwar, Paul Tucker, Vincent Vanhoucke, Vijay Vasudevan, Fernanda Vi\'{e}gas, Oriol Vinyals, Pete Warden, Martin Wattenberg, Martin Wicke, Yuan Yu, and Xiaoqiang Zheng.
\newblock {TensorFlow}: Large-scale machine learning on heterogeneous systems, 2015.
\newblock URL \url{https://www.tensorflow.org/}.
\newblock Software available from tensorflow.org.

\bibitem[Avelin and Karlsson(2022)]{JMLR:cut_off}
Benny Avelin and Anders Karlsson.
\newblock Deep limits and a cut-off phenomenon for neural networks.
\newblock \emph{Journal of Machine Learning Research}, 23\penalty0 (191):\penalty0 1--29, 2022.
\newblock URL \url{http://jmlr.org/papers/v23/21-0431.html}.

\bibitem[Buchanan et~al.(2021)Buchanan, Gilboa, and Wright]{buchanan_deep}
Sam Buchanan, Dar Gilboa, and John Wright.
\newblock Deep networks and the multiple manifold problem.
\newblock In \emph{International Conference on Learning Representations}, 2021.
\newblock URL \url{https://openreview.net/forum?id=O-6Pm_d_Q-}.

\bibitem[Cheon et~al.(2013)Cheon, Jung, and Shapiro]{bessel}
Gi-Sang Cheon, Ji-Hwan Jung, and Louis~W. Shapiro.
\newblock Generalized {B}essel numbers and some combinatorial settings.
\newblock \emph{Discrete Mathematics}, 313\penalty0 (20):\penalty0 2127--2138, 2013.
\newblock ISSN 0012-365X.

\bibitem[Cho and Saul(2009)]{cho-saul}
Youngmin Cho and Lawrence Saul.
\newblock Kernel methods for deep learning.
\newblock In Y.~Bengio, D.~Schuurmans, J.~Lafferty, C.~Williams, and A.~Culotta, editors, \emph{Advances in Neural Information Processing Systems}, volume~22. Curran Associates, Inc., 2009.
\newblock URL \url{https://proceedings.neurips.cc/paper/2009/file/5751ec3e9a4feab575962e78e006250d-Paper.pdf}.

\bibitem[Deng(2012)]{mnist}
Li~Deng.
\newblock The {MNIST} database of handwritten digit images for machine learning research.
\newblock \emph{IEEE Signal Processing Magazine}, 29\penalty0 (6):\penalty0 141--142, 2012.

\bibitem[Dherin et~al.(2022)Dherin, Munn, Rosca, and Barrett]{dherin_simple_solutions}
Benoit Dherin, Michael Munn, Mihaela Rosca, and David~GT Barrett.
\newblock Why neural networks find simple solutions: The many regularizers of geometric complexity.
\newblock In Alice~H. Oh, Alekh Agarwal, Danielle Belgrave, and Kyunghyun Cho, editors, \emph{Advances in Neural Information Processing Systems}, 2022.
\newblock URL \url{https://openreview.net/forum?id=-ZPeUAJlkEu}.

\bibitem[Eldan and Shamir(2015)]{eldan_expressivity}
Ronen Eldan and Ohad Shamir.
\newblock The power of depth for feedforward neural networks.
\newblock In \emph{Annual Conference Computational Learning Theory}, 2015.

\bibitem[Elsken et~al.(2019)Elsken, Metzen, and Hutter]{NAS}
Thomas Elsken, Jan~Hendrik Metzen, and Frank Hutter.
\newblock Neural architecture search: a survey.
\newblock \emph{J. Mach. Learn. Res.}, 20\penalty0 (1):\penalty0 1997–2017, jan 2019.
\newblock ISSN 1532-4435.

\bibitem[Hanin(2018)]{boris_gradients}
Boris Hanin.
\newblock Which neural net architectures give rise to exploding and vanishing gradients?
\newblock In S.~Bengio, H.~Wallach, H.~Larochelle, K.~Grauman, N.~Cesa-Bianchi, and R.~Garnett, editors, \emph{Advances in Neural Information Processing Systems}, volume~31. Curran Associates, Inc., 2018.
\newblock URL \url{https://proceedings.neurips.cc/paper/2018/file/13f9896df61279c928f19721878fac41-Paper.pdf}.

\bibitem[Hanin(2023)]{boris_correlation_functions}
Boris Hanin.
\newblock Random fully connected neural networks as perturbatively solvable hierarchies, 2023.
\newblock URL \url{https://arxiv.org/abs/2204.01058}.

\bibitem[Hayou et~al.(2019)Hayou, Doucet, and Rousseau]{hayou}
Soufiane Hayou, Arnaud Doucet, and Judith Rousseau.
\newblock On the impact of the activation function on deep neural networks training.
\newblock In \emph{International Conference on Machine Learning}, 2019.

\bibitem[He et~al.(2015)He, Zhang, Ren, and Sun]{he_initialization}
Kaiming He, Xiangyu Zhang, Shaoqing Ren, and Jian Sun.
\newblock Delving deep into rectifiers: Surpassing human-level performance on imagenet classification.
\newblock In \emph{2015 IEEE International Conference on Computer Vision (ICCV)}, pages 1026--1034, 2015.
\newblock \doi{10.1109/ICCV.2015.123}.

\bibitem[Kreinin(2016)]{pq}
Alexander Kreinin.
\newblock Integer sequences connected to the {L}aplace continued fraction and {R}amanujan’s identity.
\newblock \emph{Journal of Integer Sequences}, 19:\penalty0 1--12, 06 2016.

\bibitem[Krizhevsky(2009)]{cifar}
Alex Krizhevsky.
\newblock Learning multiple layers of features from tiny images.
\newblock pages 32--33, 2009.
\newblock URL \url{https://www.cs.toronto.edu/~kriz/learning-features-2009-TR.pdf}.

\bibitem[Li et~al.(2022)Li, Nica, and Roy]{li_nica_roy}
Mufan~Bill Li, Mihai Nica, and Daniel~M. Roy.
\newblock The neural covariance {SDE}: Shaped infinite depth-and-width networks at initialization.
\newblock In Alice~H. Oh, Alekh Agarwal, Danielle Belgrave, and Kyunghyun Cho, editors, \emph{Advances in Neural Information Processing Systems}, 2022.
\newblock URL \url{https://openreview.net/forum?id=WG3vmsteqR_}.

\bibitem[Martens et~al.(2021)Martens, Ballard, Desjardins, Swirszcz, Dalibard, Sohl{-}Dickstein, and Schoenholz]{martens_shaping}
James Martens, Andy Ballard, Guillaume Desjardins, Grzegorz Swirszcz, Valentin Dalibard, Jascha Sohl{-}Dickstein, and Samuel~S. Schoenholz.
\newblock Rapid training of deep neural networks without skip connections or normalization layers using deep kernel shaping.
\newblock \emph{CoRR}, 2021.
\newblock URL \url{https://arxiv.org/abs/2110.01765}.

\bibitem[Nachum et~al.(2022)Nachum, Hazla, Gastpar, and Khina]{cnn_nachum}
Ido Nachum, Jan Hazla, Michael Gastpar, and Anatoly Khina.
\newblock A {J}ohnson-{L}indenstrauss framework for randomly initialized {CNN}s.
\newblock In \emph{International Conference on Learning Representations}, 2022.
\newblock URL \url{https://openreview.net/forum?id=YX0lrvdPQc}.

\bibitem[Poole et~al.(2016)Poole, Lahiri, Raghu, Sohl-Dickstein, and Ganguli]{poole_expressivity}
Ben Poole, Subhaneil Lahiri, Maithra Raghu, Jascha Sohl-Dickstein, and Surya Ganguli.
\newblock Exponential expressivity in deep neural networks through transient chaos.
\newblock In D.~Lee, M.~Sugiyama, U.~Luxburg, I.~Guyon, and R.~Garnett, editors, \emph{Advances in Neural Information Processing Systems}, volume~29. Curran Associates, Inc., 2016.
\newblock URL \url{https://proceedings.neurips.cc/paper/2016/file/148510031349642de5ca0c544f31b2ef-Paper.pdf}.

\bibitem[Roberts et~al.(2022)Roberts, Yaida, and Hanin]{principles_deep_learning}
Daniel~A. Roberts, Sho Yaida, and Boris Hanin.
\newblock \emph{The Principles of Deep Learning Theory: An Effective Theory Approach to Understanding Neural Networks}.
\newblock Cambridge University Press, 2022.
\newblock \doi{10.1017/9781009023405}.

\bibitem[Schoenholz et~al.(2017)Schoenholz, Gilmer, Ganguli, and Sohl-Dickstein]{schoenholz}
Samuel~S. Schoenholz, Justin Gilmer, Surya Ganguli, and Jascha Sohl-Dickstein.
\newblock Deep information propagation.
\newblock In \emph{International Conference on Learning Representations}, 2017.
\newblock URL \url{https://openreview.net/forum?id=H1W1UN9gg}.

\bibitem[Vershynin(2018)]{vershynin}
Roman Vershynin.
\newblock \emph{High-Dimensional Probability: An Introduction with Applications in Data Science}.
\newblock Number~47 in Cambridge Series in Statistical and Probabilistic Mathematics. {Cambridge University Press}, 2018.
\newblock ISBN 978-1-108-41519-4.

\bibitem[Xiao et~al.(2017)Xiao, Rasul, and Vollgraf]{fmnist}
Han Xiao, Kashif Rasul, and Roland Vollgraf.
\newblock Fashion-{MNIST}: a novel image dataset for benchmarking machine learning algorithms, 2017.
\newblock URL \url{http://arxiv.org/abs/1708.07747}.

\end{thebibliography}

\end{document}